\documentclass{article}

\PassOptionsToPackage{dvipsnames,table}{xcolor}

\usepackage{arxiv}
\usepackage{microtype}
\usepackage[numbers,sort&compress]{natbib}

\usepackage{amsmath,amsfonts,bm}

\def\eqref#1{equation~\ref{#1}}

\def\1{\bm{1}}

\DeclareMathAlphabet{\mathsfit}{\encodingdefault}{\sfdefault}{m}{sl}
\SetMathAlphabet{\mathsfit}{bold}{\encodingdefault}{\sfdefault}{bx}{n}

\newcommand{\E}{\mathbb{E}}

\newcommand{\R}{\mathbb{R}}

\DeclareMathOperator*{\argmax}{arg\,max}
\DeclareMathOperator*{\argmin}{arg\,min}

\usepackage{hyperref}
\usepackage{url}
\usepackage{amsmath}
\usepackage{amssymb}
\usepackage{amsthm}
\usepackage{booktabs}
\usepackage{tabularx}
\usepackage{graphicx}
\usepackage{xspace}
\usepackage{wrapfig}
\usepackage{diagbox}
\usepackage{multirow}
\usepackage{array}
\usepackage{placeins}
\usepackage{algorithm}
\usepackage{algpseudocode}

\usepackage[most]{tcolorbox}

\usepackage{tikz}

\theoremstyle{plain}
\newtheorem{theorem}{Theorem}
\newtheorem{proposition}{Proposition}
\newtheorem{lemma}{Lemma}
\newtheorem{corollary}{Corollary}
\theoremstyle{definition}
\newtheorem{definition}{Definition}
\newtheorem{assumption}{Assumption}

\newtheoremstyle{myremark}
  {\topsep}      %
  {\topsep}      %
  {\itshape}     %
  {}             %
  {\bfseries}    %
  {.}            %
  {0.5em}        %
  {}             %

\theoremstyle{myremark}
\newtheorem{remark}{Remark}

\newtcolorbox{contributionbox}{
  enhanced,
  breakable,
  colback=blue!3,
  colframe=blue!45!black,
  boxrule=0.6pt,
  arc=2pt,
  left=6pt,
  right=6pt,
  top=6pt,
  bottom=6pt,
  before skip=8pt,
  after skip=8pt
}

\definecolor{rqblue}{HTML}{5677C8}
\definecolor{rqbluebg}{HTML}{EAF0FF}
\definecolor{rqcyan}{HTML}{4F91A8}
\definecolor{rqcyanbg}{HTML}{E8F5F8}
\definecolor{rqpurple}{HTML}{8A4FA3}
\definecolor{rqpurplebg}{HTML}{F3EAF7}
\newcommand{\rqbadge}[3]{%
  \tcbox[on line, boxsep=0pt,
    left=2.5pt, right=2.5pt, top=0.5pt, bottom=0.5pt,
    arc=2pt, boxrule=0.45pt, colframe=#1, colback=#2]{%
    \textcolor{#1}{\strut\scriptsize\bfseries #3}%
  }%
}
\newcommand{\rqone}{\rqbadge{rqblue}{rqbluebg}{RQ1}\enspace}
\newcommand{\rqtwo}{\rqbadge{rqcyan}{rqcyanbg}{RQ2}\enspace}

\newcommand*\circledblue[1]{%
\tikz[baseline=(char.base)]{
  \node[shape=circle, draw=NavyBlue!60, fill=NavyBlue!10, thick, inner sep=1pt] (char) {\scriptsize\textsf{#1}};
}}

\renewcommand{\E}{\mathbb{E}}

\let\argmax\relax
\let\argmin\relax
\DeclareMathOperator*{\argmax}{arg\,max}
\DeclareMathOperator*{\argmin}{arg\,min}

\newcommand{\M}{\mathcal{M}}
\newcommand{\X}{\mathcal{X}}

\newcommand{\DeltaK}{\Delta_K}

\newcommand{\ADB}{A^{\mathrm{DB}}}
\newcommand{\AhatDB}{\widehat{A}^{\mathrm{DB}}}

\newcommand{\Exploit}{\operatorname{Expl}}

\renewcommand{\1}{\mathbf{1}}

\newcommand{\Om}{\Omega}
\newcommand{\Qcal}{\mathcal{Q}}
\newcommand{\Pcal}{\mathcal{P}}

\newcommand{\elb}{\underline{e}}
\newcommand{\mucol}{\bm{\mu}}
\newcommand{\ecol}{\bm{e}}

\newcommand{\Lcal}{\mathcal{L}}

\newcommand{\method}{\textsc{NashEval}\xspace}

\definecolor{lightgreen}{RGB}{144,238,144}
\definecolor{lightred}{RGB}{255,182,193}

\title{Context-dependent agent evaluation with \mbox{orthogonal} equilibrium learning}

\author{%
Haorui Ma\thanks{Corresponding authors: Haorui Ma (\texttt{H.Ma@lmu.de}) and Stefan Feuerriegel (\texttt{feuerriegel@lmu.de}).}
\& Stefan Feuerriegel\footnotemark[1] \\
AI in Management, LMU Munich \\
Munich, Germany \\
Munich Center for Machine Learning, Germany \\
\texttt{\{H.Ma,feuerriegel\}@lmu.de}
\AND
Jiangmeng Li, Yi Li \& Fanjing Xu \\
Institute of Software \\
Chinese Academy of Sciences \\
\texttt{jiangmeng2019@iscas.ac.cn}
\And
Zehua Zang \\
Peking University
}

\begin{document}
\maketitle

\begin{abstract}
Many applications require to evaluate agents under contextual information (e.g., a prompt, task, or user group). We study how to perform such context-dependent agent evaluation from offline feedback. Existing score-based models for this purpose (e.g., Bradley--Terry) impose a transitive preference ordering, which fails to reflect collective preferences when human judgements are heterogeneous. Inspired by social choice theory, we frame evaluation as a contextual game between two players, each selecting a distribution over agents as the strategy to receive greater collective preference than the other. Then, the support of the Nash equilibrium defines a context-specific set of winners. However, learning context-specific equilibria from offline logs is difficult because each context reveals human feedback on only a subset of agents, and, hence, a na\"ive plug-in estimator can therefore be biased. To address these challenges, we propose \method, a general framework for robust contextual equilibrium learning. \method first constructs debiased estimates of the contextual payoff matrix that characterizes the game. \method then learns the context-to-equilibrium mapping with a tailored orthogonal loss, which avoids the need to solve a separate game for each context. We show theoretically that errors in estimating the nuisance functions underlying the payoff matrix affect the risk of the learned equilibrium (i.e., exploitability) only through higher-order terms. Across various experiments, \method improves robustness of equilibrium learning and consistently identifies the set of top-performing agents across contexts. 

\end{abstract}

\section{Introduction}

Many applications require agents to be evaluated together with contextual information, such as the prompt, task, language, domain, or user group~\citep[e.g.,][]{Xu.2025context, Frick.2025}.. For example, in model routing, the goal is not to identify a single globally best language model, but to select an appropriate model for a particular prompt~\citep{Ong.2025}. More generally, agents may have different strengths across tasks and domains~\citep{Frick.2025}, and human preferences can also vary across users and settings~\citep{Haghtalab.2026}. Here, we study \textbf{context-dependent agent evaluation} from selectively observed relative feedback (e.g., pairwise preferences or partial rankings) collected from a population of voters with heterogeneous preferences to identify, for each context, a set of agents that best reflects collective preference of the population (see Figure~\ref{fig:concept}).

\begin{figure}
\centering
\includegraphics[width=\linewidth]{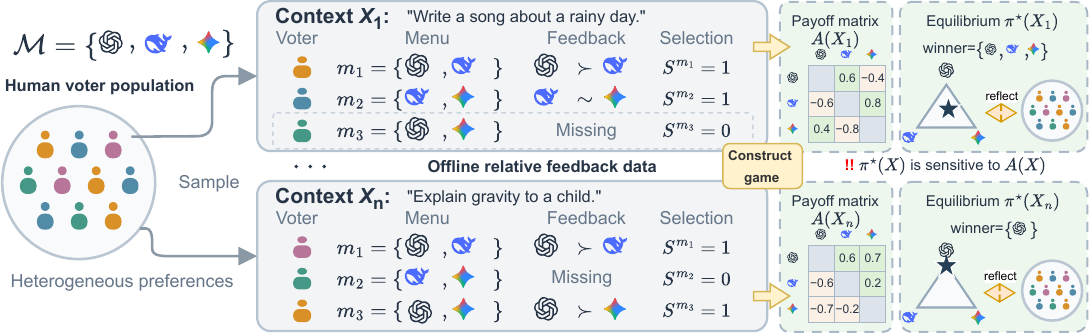}
\caption{\textbf{Agent evaluation is subject to two challenges:} (1) preferences sampled from a large population of voters are heterogeneous and, thus, often intransitive; and (2) the capabilities of agents vary across contexts. $\Rightarrow$ We propose a novel framework called \method for this task by framing evaluation as solving a contextual, two-player zero-sum game, constructed from the preferences.}
\label{fig:concept}
\end{figure}

In this paper, we propose \method, a game-theoretic framework for contextual agent evaluation. Our framework is flexible and accommodates any form of relative feedback that can be represented as ordered partitions (e.g.. pairwise preferences, partial rankings, and winner-of-subset feedback). For a given context, we formulate the evaluation as a regularized two-player zero-sum game, in which each player selects a distribution over agents and seeks to outperform the other according to the collective preference. The support of the corresponding Nash equilibrium then defines a context-specific set of winners.

To learn this equilibrium from selectively observed feedback, our framework proceeds in two stages. First, we construct a novel \textit{debiased payoff matrix (DPM)} that is robust to nuisance misspecification. Second, we use the estimated payoff to introduce a tailored orthogonal loss to directly learn the context-to-equilibrium mapping. Thereby, we amortize equilibrium computation rather than solving a separate game for every context. We show theoretically that that the robustness guarantees of the DPM transfer to the learned equilibrium; i.e., errors in estimating the nuisance functions underlying the payoff matrix affect the risk of the learned equilibrium (i.e., exploitability) only through higher-order terms and, under suitable conditions, can therefore achieve quasi-oracle convergence rate. As a by-product, we derive a general equilibrium learning framework that is robust to nuisance misspecifications.

Our main \textbf{contributions}
are: \textbf{(i)}~We propose \method, a novel game-theoretic framework for learning context-dependent agent evaluators from logged data with relative feedback. \textbf{(ii)}~We derive strong theoretical guarantees for our equilibrium learner in the form of Neyman orthogonality, bounds on the average exploitability, and a quasi-oracle convergence rate. \textbf{(iii)}~We demonstrate the effectiveness of our \method framework on both simulated and real-world datasets.

\section{Related work\protect\footnote{We provide an extended related work in Appendix~\ref{sec:related-work}.}}

\textbf{Social choice theory:} This literature studies collective choice under potentially intransitive preferences~\citep{Fishburn.1984a,Brandt.2017}. Concepts like Condorcet \citep{Condorcet.1785} and Copeland winners \citep{Copeland.1951} may fail to identify a unique or robust winner~\citep{Condorcet.1785,Laslier.1997,Tideman.1987}. We build on this literature by developing a game-theoretic solution concept for context-dependent agent evaluation.

\textbf{Agent evaluations:} There are several approaches for related but \emph{different} settings (see Table~\ref{tab:related}). Traditionally, evaluation methods aggregate feedback across contexts to produce a single \textit{global} leaderboard~\citep[e.g.,][]{Chiang.2024, Frauen.2026}, and, more recently, there are extensions to incorporate contextual information~\citep[e.g.,][]{Ong.2025,Frick.2025,Chiang.2025dueling}. However, these approaches rely on score-based models \citep[e.g., Bradley--Terry;][]{Bradley.1952}, and therefore assume \textit{transitive} preferences, meaning that if $A \succ B$ and $B \succ C$, then $A \succ C$. Importantly, this assumption can be restrictive when feedback is collected from a heterogeneous population, since aggregating individual preferences can result in intransitive collective preferences~\citep{Condorcet.1785,Kramer.1973}, particularly for subjective, open-ended tasks~\citep{Haghtalab.2026}. As a result, models that assume such transitivity can be misspecified and produce biased (or even conflicting) evaluations that do not reflect the underlying heterogeneous. To the best of our knowledge, approaches for context-dependent agent evaluation with intransitive collective preferences are missing.

\textbf{Game-theoretic approaches:} A common approach is to frame agent evaluation as a two-player zero-sum game between two dueling evaluation policies that compete head-to-head for voter preference~\citep{Yue.2012}. The game is characterized by an antisymmetric \emph{payoff matrix (PM)} constructed from the relative feedback, and its Nash equilibrium is a distribution over agents for which support is interpreted as the set of winners. Because no evaluation policy strictly beats this equilibrium, it is the best evaluation achievable under intransitive feedback and has typically favorable properties (e.g., Condorcet-consistent, meaning that when a Condorcet winner exists, the equilibrium places all probability on that agent). In the dueling bandit literature, such equilibrium concepts are referred to as the \emph{von Neumann winner} (vNw)~\citep{Dudik.2015} and, social choice theory, as the \emph{maximal lottery}~\citep{Kreweras.1965,Fishburn.1984a}. Recently, \citet{Lanctot.2023,Khalaf.2026} study maximal lotteries for \textit{global} agent evaluation. In contrast, approaches that learns the context-specific equilibrium for agent evaluation are still missing.

\begin{wraptable}{r}{0.5\textwidth}
\centering
\caption{Key works on agent evaluation from relative feedback.}
\label{tab:related}
\scriptsize
\setlength{\tabcolsep}{3pt}
\begin{tabularx}{\linewidth}{@{}l>{\raggedright\arraybackslash}X|>{\raggedright\arraybackslash}X@{}}
\toprule
\diagbox[width=9em]{\textbf{Preference}}{\textbf{Evaluator}} & \textbf{Global} & \textbf{Contextual} \\
\midrule
\textbf{Transitive} & \cellcolor{lightred!40}\citep{Chiang.2024,Frauen.2026} & \cellcolor{lightred!40}\citep{Ong.2025,Frick.2025} \\
\cmidrule(lr){1-3}
\textbf{Intransitive} & \cellcolor{lightred!40}\citep{Lanctot.2023,Khalaf.2026} & \multicolumn{1}{|c}{\cellcolor{lightgreen!40}\textbf{\method (ours)}} \\
\bottomrule
\end{tabularx}
\end{wraptable}

\emph{\textbf{Why context-dependent agent evaluation is challenging?}} A major difficulty of context-dependent evaluation is that, for a given context, a voter is shown only a small subset of agents, and thus the feedback is often \textit{partially observed}. This means that the contextual PM must be \emph{estimated} using nuisance functions\footnote{The term ``nuisance functions'' refers to components that facilitate estimation of statistical target by modeling observable distributions from data but are not themselves of direct interest.}. When these nuisance functions are misspecified, the estimated PM can be biased, which leads to a suboptimal approximation of the equilibrium (see Proposition~\ref{prop:stability}). Hence, a key challenge is to learn the contextual equilibrium in a way that is \textit{robust to errors in the nuisance function}.

\section{Preliminaries}

When preferences are aggregated across a population, they may become intransitive even if each individual's preferences are transitive (see Appendix~\ref{example:arising-cycles} for an example). In such cases, a single deterministic winner may not exist, which motivates game-theoretic solution concepts (i.e., lotteries over winners) that can represent collective preferences~\citep{Fishburn.1984a,Brandt.2017}.

\textbf{Preference game:} Let $A\in[-1,1]^{K\times K}$ denote the collective preference matrix with entries$A_{jk}=P(j\succ k)-P(k\succ j)$, which satisfies $A=-A^\top$. We interpret $A$ as the payoff matrix of a symmetric two-player zero-sum game, where the two players play strategies $\pi,q\in\Delta_K$. Player $\pi$ receives payoff $\Pcal(\pi,q)=\pi^\top A q$, and player $q$ receives $-\Pcal(\pi,q)$. The corresponding \textit{Nash equilibrium} is given by the following minimax formulation~\citep{Neumann.1928}:
\begin{align}
    \label{def:vnw}
    (\pi^\star, q^\star) = \left(\argmax_{\pi} \min_{q} \Pcal(\pi, q),\ \argmin_{q} \max_{\pi} \Pcal(\pi, q)\right).
\end{align}
By symmetry, both players share the same equilibrium strategy, and the solution is called \emph{von Neumann winner}~\citep{Dudik.2015} or \emph{maximal lottery}~\citep{Kreweras.1965,Fishburn.1984a}.

The equilibrium is well suited to intransitive preferences because it depends only on pairwise comparisons and always exists by the minimax theorem. Additional properties, including Condorcet consistency and related consistency axioms, are discussed in Appendix~\ref{sec:benefits-vnw}.

\section{Context-dependent agent evaluation from selective feedback}
\label{sec:problem}

We now formulate context-dependent agent evaluation from offline, selectively observed relative feedback as a contextual game and define its equilibrium.

\textbf{Setting:} $\bullet$~\emph{Agents}: Let $\M=\{1,\dots,K\}$ denote the candidate agents, which perform tasks specified by the context $X\in\X$ (e.g., a prompt, task, language, domain, or user group). A \emph{menu} $m \subset \M$ is a subset of agents being compared, and let $\Qcal$ denote the collection of all possible menus. For each context, $S^m\in\{0,1\}$ indicates whether menu $m$ was selected to generate the feedback, $Y_m$. 

$\bullet$~\emph{Relative feedback}: The feedback $Y_m$ is \emph{relative}, represented by an \emph{ordered partition} of the selected menu, i.e., 
\begin{align}
Y_m=(B_1\succ\cdots\succ B_r),
\qquad
B_i\ne\emptyset,\quad
B_i\cap B_j=\emptyset\ \text{ for }i\ne j,
\quad
\bigcup_{i=1}^r B_i=m,
\end{align}
where $r$ is the number of blocks, and where agents within the same block are tied and earlier blocks are preferred. The above various forms of relative feedback, including pairwise preferences, partial rankings, and winner-of-subset feedback (see Table~\ref{tab:modalities}). For example, pairwise feedback corresponds to $|m|=2$.

$\bullet$~\emph{Logged data}: We write the logged data as $O=\big(X,\ \{S^m\ :\ m\in\Qcal\},\ \{S^m Y_m\ :\ m\in\Qcal\}\big)$ and observe i.i.d. samples $\mathcal{D}_n=\{O_i\}_{i=1}^n \sim\mathbb{P}$. 

\textbf{Pairwise outcomes:} For each menu $m$, the feedback $Y_m$ naturally induces a pairwise outcome for every $\{j, k\} \subset m$:
\begin{align}
    \label{eq:verdict}
    Z_{jk}^m = \1\{j\succ_{Y_m}k\}-\1\{k\succ_{Y_m}j\}\in\{-1,0,1\},
\end{align}
which satisfies $Z^m_{jk}=-Z^m_{kj}$ and where $Z^m_{jk}=0$ denotes a tie. We are then interested in the \emph{context-dependent mean pairwise outcome}, defined as $g^m_{jk}(x)=\E[Z^m_{jk}\mid X=x]$, which quantifies the average collective preference of agent $j$ over agent $k$ under context $x$ and menu $m$. 

\textbf{Contextual payoff:} We then aggregate context-dependent mean pairwise outcomes across menus into a contextual payoff matrix (PM).
\begin{definition}[Contextual payoff matrix]
\label{def:payoff}
\em Let $\Qcal_{jk}=\{m\in\Qcal:\{j,k\}\subseteq m\}$ be the set of menus containing the pair $\{j,k\}$. Let $\mathcal{F}=\{f_{jk}:\ [-1,1]^{\Qcal_{jk}}\to[-1,1]\mid j, k \in \M, j\neq k \}$ denote set of smooth, odd functions that serve as menu aggregators, with $f_{kj}=f_{jk}$ for every $j\ne k$. Then, the contextual payoff matrix is defined by aggregating the menu-specific mean outcomes $g^m_{jk}(x)$ via the menu aggregators $f_{jk} \in \mathcal{F}$:
\begin{align}
    \label{eq:payoff}
    A(x) = \big(A_{jk}(x)\big)_{j,k\in\M},\quad \text{where }
    A_{jk}(x)=f_{jk}\Big(\big\{g^m_{jk}(x)\ :\ m\in\Qcal_{jk}\big\}\Big)\quad(j\ne k),
\end{align}
with $A_{jj}(x)=0$ for every $j\in\M$.
\end{definition}

\begin{remark}
The contextual PM characterizes the expected pairwise preferences under a given context, which reflects the collective preference of the population. Since $g^m_{kj}(x)=-g^m_{jk}(x)$, we have $A_{kj}(x)=-A_{jk}(x)$. Thus, $A(x)$ is skew-symmetric.
\end{remark} 
\begin{remark}
A simple instantiation of the aggregators is the uniformly weighted combination $f_{jk}(\textbf{g}_{jk})=\frac{1}{|\Qcal_{jk}|}\sum_{m\in\Qcal_{jk}}g_{jk}^m$. In practice, users can customize the aggregator to their needs, for example, to assign more or less weight to feedback from certain menus. See Appendix~\ref{app:instantiations} for more instantiations.
\end{remark}

\textbf{Contextual equilibrium:} Based on the contextual PM, we formulate a regularized contextual game and define its equilibrium. 
\begin{definition}[Contextual game and its equilibrium]
\label{def:contextual-game}
\em
Let $\Om$ be continuously differentiable on a neighborhood of $\Delta_K$ and $\kappa$-strongly convex on $\Delta_K$, where $\kappa>0$. We consider a \emph{regularized two-player zero-sum game}, where the expected payoff between two players with strategies $\pi, q \in \Delta_K$ is given by $F_\Om(\pi,q;x)=\pi^\top A(x)q-[\Om(\pi)-\Om(q)]$. The regularized contextual equilibrium is the minimax solution of the optimization problem:
\begin{align}
    \label{eq:contextual-equilibrium}
    \pi^\star(x) = \argmax_{\pi\in\DeltaK}\ \min_{q\in\DeltaK}\ F_\Om(\pi,q;x) := \pi^\top A(x)q-[\Om(\pi)-\Om(q)].
\end{align}
\end{definition}
Since $F_\Om(\pi,q;x) = -F_\Om(q,\pi;x)$, the game is symmetric and zero-sum. Strong convexity implies that the associated self-play operator is strongly monotone, which ensures uniqueness of the equilibrium and which also ensures Lipschitz continuity to perturbations of the PM~\citep{Dontchev.2009, Sokota.2022} (see Appendix~\ref{app:game-theory} for details). 

Further, the regularization term can be used to control the equilibrium according to constraints in practice. For example, one can incorporate deployment considerations such as agent cost or encourage properties such as sparsity. In practice, a common choice is to combine a linear term for the deployment cost with an $L_2$ regularizer; i.e., $\Om(\pi) = \frac{\lambda}{2}\|\pi\|_2^2 + \sum_{j=1}^K c_j \pi_j,$ where $c_j$ is the cost of deploying agent $j$.

\textbf{Target.} Our aim to learn the \textit{context-to-equilibrium mapping} $x\mapsto\pi^\star(x)$ defined in Eq.~(\ref{eq:contextual-equilibrium}). Here, uniqueness of the equilibrium ensures that this mapping is well-defined. We then interpret the support of the contextual equilibrium (i.e. $\text{supp}(\pi^\star(x)) = \{k: \pi^\star_k(x) > 0\}$) as the set of winning agents for the given context $x$.

We make standard identifiability assumptions (see Assumption~\ref{ass:identifiability}), such as positivity and missing at random~\citep{Rubin.1974}, so that we can identify the contextual PM from observed data through the context-dependent mean outcomes $\mu^m_{0,jk}(x)$:
\begin{align}
    \label{eq:mu-to-g}
    g^m_{jk}(x) = \E[Z^m_{jk}\mid X=x] = \E[Z^m_{jk} \mid X=x, S^m=1] := \mu^m_{0,jk}(x) \quad\text{for }m\in\Qcal_{jk}.
\end{align}
Hence, identification of the contextual PM, and thus its equilibrium, follows from Eqs.~(\ref{eq:payoff}) and~(\ref{eq:contextual-equilibrium}) (see Appendix~\ref{app:proof-identification} for details).

\textbf{Nuisances:} We refer to $\mu^m_{0,jk}(x)$ and $e^m_0(x)$ as \emph{nuisance functions} and write $\eta_0=(\mucol_0,\ecol_0)$, where (i)~$\mucol_0=(\mu_0^m)_{m\in\Qcal}$, $\mu_0^m=(\mu^m_{0,jk})_{j,k\in m,\,j\ne k}$, is the collection of the outcome regressions and (ii)~$\ecol_0=(e_0^m)_{m\in\Qcal}$  are the selection propensities across menus.

\textbf{Stability:} We adapt the stability result for equilibrium mappings~\citep{Dontchev.2009} to our context-dependent setting (see Appendix~\ref{app:payoff-stability} for the derivation and definition of $\|\cdot\big\|_\star$.).
\begin{proposition}[Contextual payoff stability]
\label{prop:stability}
For any two contexts $x,x'\in\X$,
\begin{align}
    \label{eq:stability}
    \big\|\pi^\star(x)-\pi^\star(x')\big\|
    \le \kappa^{-1}\big\|(A(x)-A(x'))\pi^\star(x')\big\|_\star 
    \le \kappa^{-1}\big\|A(x)-A(x')\big\|_\star.
\end{align}
\end{proposition}

Proposition~\ref{prop:stability} shows that the equilibrium is Lipschitz continuous with respect to changes in the contextual PM, which implies that errors in the contextual PM propagate directly to the equilibrium. Hence, in order to accurately learn the contextual equilibrium, we need to accurately estimate the contextual PM $A(x)$. However, this is challenging: some pairwise entries of the PM are missing, and, because $A(x)$ depends on estimated nuisance functions, a na\"ive plug-in estimator can inherit first-order bias from the nuisance estimation error. To remove this first-order bias, we propose a novel Neyman-orthogonal equilibrium learning method in the next section.

\section{Orthogonal equilibrium learning}
\label{sec:method}

\textbf{Overview:} Our \method framework consists of {two stages} (see Figure~\ref{fig:framework}): in Stage~\circledblue{1}, we estimate the contextual payoff matrix $\hat A(x)$ from the logged feedback, and in Stage~\circledblue{2}, we use estimated contextual payoff matrix to learn the context-to-equilibrium mapping. 

Importantly, we show in Section~\ref{sec:naive-equilibrium} that a na\"ive, plug-in equilibrium learner is subject to the so-called \emph{plug-in bias}, which can lead to a suboptimal equilibrium. As a remedy, we develop a Neyman-orthogonal equilibrium learner, and, to do so, we first propose a \emph{debiased payoff matrix (DPM)} that is robust to nuisance misspecification (Section~\ref{subsec:dpme}), which we use in Stage~\circledblue{1}. In Section~\ref{sec:learning-equilibrium}, we present a tailored orthogonal loss based on DPM for learning the contextual equilibrium in Stage~\circledblue{2}. Finally, we show that our \method framework has favorable theoretical guarantees, such as Neyman-orthogonality and a quasi-oracle convergence rate (Section~\ref{subsec:theoretical-benefits}).

\subsection{Why a na\"ive plug-in equilibrium learner is problematic}
\label{sec:naive-equilibrium}

\textbf{Plug-in equilibrium learner.} A natural approach is to first estimate the outcome regressions $\widehat\mu^m_{jk}$ using an arbitrary machine learning method and then ``plug'' them into the menu aggregator. This yields 
\begin{equation}
\label{eq:plugin-payoff-estimator}
\widehat A^{\mathrm{PI}}_{jk}(x)
= \frac{1}{|\Qcal_{jk}|}\sum_{m\in\Qcal_{jk}}\widehat\mu^m_{jk}(x), \quad \text{where } \widehat\mu^m_{jk}(x)=\widehat \E[Z^m_{jk}\mid X=x,S^m=1].
\end{equation}
where we consider the mean aggregator for simplicity. We can then learn the equilibrium $\widehat\pi^{\mathrm{PI}}(x)$ using the estimated contextual PM $\widehat A^{\mathrm{PI}}(x)$. 

However, by Proposition~\ref{prop:stability}, errors in the estimated contextual PM propagate directly to the learned equilibrium. In particular, the error of the equilibrium is Lipschitz-bounded by the error of the contextual PM estimator: $\|\widehat\pi^{\mathrm{PI}}(x)-\pi^\star(x)\| \leq \frac{1}
{\kappa|\Qcal_{jk}|}\sum_{m\in\Qcal_{jk}}\| \Delta \mu^m(x) \|$).
Hence, errors in $\widehat\mu^m$ affect the estimated contextual PM (and therefore the learned equilibrium) at first order. In semiparametric statistics, such sensitivity to direct plug-in estimation is commonly referred to as \emph{plug-in bias}~\citep{Kennedy.2023}. The effect can be particularly pronounced when $\kappa$ is small.

\textbf{Key idea: orthogonal learning.} To reduce such sensitivity to nuisance estimation errors, \method builds on orthogonal statistical learning~\citep{Chernozhukov.2018,Foster.2023}. Let $\widehat\eta$ denote the estimated nuisance functions and suppose the equilibrium mapping $\pi_\theta$ is learned by minimizing
\begin{align}
    \widehat\theta = \argmin_{\theta \in \Theta} \mathcal{L}(\theta; \widehat\eta) := \E\left[\ell(\pi_\theta(X), A(X;\widehat\eta))\right]
\end{align}
Our goal thus is to construct a loss for which the gradient with respect to $\theta$ is locally insensitive to errors in the nuisance functions. Formally, a loss is \emph{Neyman-orthogonal} if
\begin{align}
    \label{eq:orthogonality-def}
    D_\eta D_\theta \mathcal{L}(\theta; \eta_0)[\widehat\theta-\theta,\widehat\eta-\eta_0] = 0, \quad \forall \theta \in \Theta
\end{align}
where $D_\eta$ and $D_\theta$ are the Gateaux derivatives w.r.t. $\eta$ and $\theta$, under directions $h_\theta = \widehat\theta- \theta$ and $h_\eta = \widehat\eta-\eta_0$. Hence, orthogononality means that the gradient $D_\theta\mathcal{L}$ is robust to perturbations $h_\eta$ around its true value. Intuitively, orthogonality thus removes the first-order effect of nuisance estimation errors on the learning objective, thus allowing nuisance errors to enter the final estimation error only through higher-order terms~\citep{Foster.2023}.

In the following, we adapt this idea to both stages of \method. However, doing so is nontrivial and requires a tailored orthogonalization for each stage. For Stage~\circledblue{1}, we construct a debiased payoff matrix that is orthogonal to nuisance estimation errors. For Stage~\circledblue{1}, we use this estimator to define an orthogonal loss for learning the contextual equilibrium. In Section~\ref{sec:learning-equilibrium}, we formally establish the property of Neyman-orthogonality.

\subsection[Stage~1: Debiased payoff matrix]
{Stage~\protect{\circledblue{1}}: Debiased payoff matrix}
\label{subsec:dpme}

In this section, we construct the debiased payoff matrix  (DPM) for Stage~\circledblue{1 } that is robust to nuisance estimation errors. We begin with the average payoff $\psi_{jk}=\E_X[A_{jk}(X)]=\E_X\big[f_{jk}(\{\mu^m_{0,jk}(X)\}_{m\in\Qcal_{jk}})\big]$. The efficient influence function (EIF) of the average payoff provides a correction (i.e., debiasing term) to the plug-in payoff estimator, and, using this correction, we can construct a debiased PM that removes first-order sensitivity to nuisance estimation errors.

\begin{definition}[Debiased payoff matrix]
\label{def:debiased-payoff-matrix}
\em
For generic nuisances $\eta=(\mucol,\ecol)$, we define the debiased payoff matrix as ()
\begin{align}
    \ADB_{jk}(x;\eta)
    :=\E_0\big[\Gamma_{jk}(O;\eta)\mid X=x\big], (j\neq k)
    \qquad
    \ADB(x;\eta):=\big(\ADB_{jk}(x;\eta)\big)_{j,k\in\M} ,
    \label{eq:debiased-payoff-surface}
\end{align}
with diagonal entries $A_{jj}(x)$ set to 0.
Here, $\Gamma_{jk}(O;\eta)$ is the EIF-based pseudo-outcome of the payoff (see Appendix~\ref{app:proof-eif} for details):
\begin{align}
    \Gamma_{jk}(O;\eta)
    = \underbrace{f_{jk}\!\left(\{\mu^m_{jk}(X)\}_{m\in\Qcal_{jk}}\right)}_{\text{plug-in term}}
    +\sum_{m\in\Qcal_{jk}} \underbrace{
    \frac{\partial f_{jk}}{\partial u_m}\!\left(\{\mu^{m'}_{jk}(X)\}_{m'\in\Qcal_{jk}}\right)
    \frac{S^m}{e^m(X)}\big\{Z^m_{jk}-\mu^m_{jk}(X)\big\}}_{\text{debiasing term}} \notag
\end{align}
\end{definition}
\begin{theorem}[Properties of the debiased payoff matrix]
\label{thm:DPME}
Under Assumptions~\ref{ass:identifiability}--\ref{ass:smoothness}, we have: (i) $\ADB_{jk}(x;\eta_0) =\E_0\big[\Gamma_{jk}(O;\eta_0)\mid X=x\big]=A_{jk}(x)$; and (ii) for every perturbation $h=(h_\mu,h_e)$, the debiased payoff is Neyman-orthogonal w.r.t. $\eta$: $D_\eta \ADB_{jk}(x;\eta)[h]\big|_{\eta=\eta_0}=0.$
\begin{proof}
    See Appendix~\ref{app:proof-neyman}.
\end{proof}
\end{theorem}

The theorem shows that the debiased payoff matrix recovers the true contextual PM when the nuisances are correctly specified, and that it is locally insensitive to nuisance perturbations. More precisely, under smoothness condition, we can show that the the pointwise bias induced by nuisance misspecification is bounded by second-order nuisance errors:
 \begin{align}
    \big|\ADB_{jk}(x;\eta) - A_{jk}(x)\big|
    \le \frac{L_f}{2}\big\|\Delta\mu(x)\big\|_2^2
    +\sum_{m\in\Qcal_{jk}}
    \left|\frac{\partial f_{jk}}{\partial u_m}\big(\mu(x)\big)\right|
    \frac{|\Delta e^m(x)|}{e^m(x)}\,|\Delta\mu^m(x)|.
    \label{eq:reminder}
\end{align}
where $\nabla f_{jk}$ is $L_f$-Lipschitz. We formally derive this bound in Appendix~\ref{app:proof-bias}. Compared with plug-in estimator, nuisance errors affect the debiased payoff matrix only through a squared outcome-error term and an outcome--propensity error product. In practice, given fitted nuisances $\widehat\eta=(\widehat{\mucol},\widehat{\ecol})$, we construct DPM estimator $\AhatDB(x;\widehat\eta)$ by regressing $\Gamma_{jk}(O;\widehat\eta)$ on $X$ for each pair $(j,k)$ with $j\ne k$, and set $\AhatDB_{jj}:=0$.

\subsection[Stage~2: Orthognal equilibriamy learning]{Stage~\protect{\circledblue{2}}: Orthognal equilibriamy learning}
\label{sec:learning-equilibrium}

\begin{figure}
\centering
\includegraphics[width=0.8\linewidth]{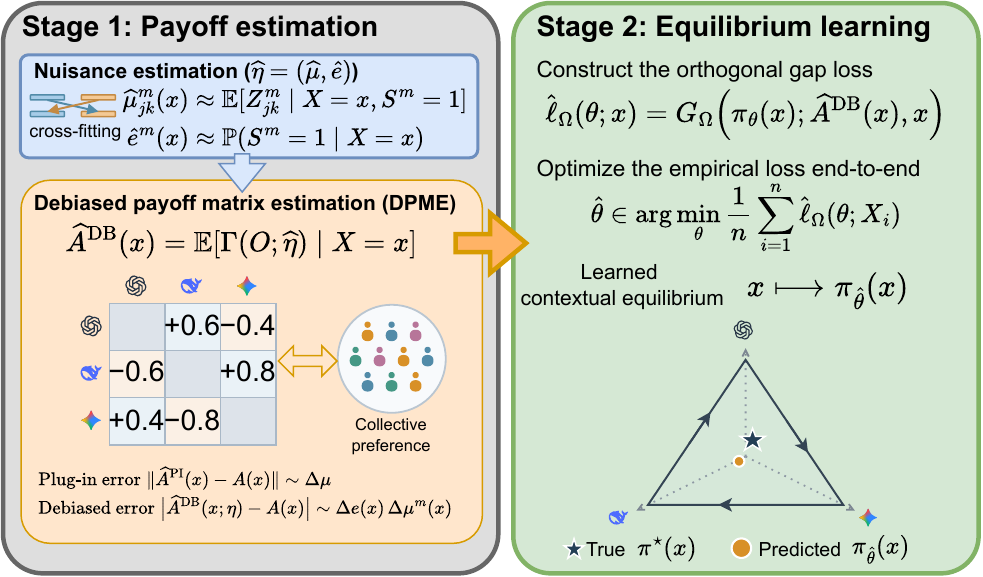}
\caption{\textbf{Overview of our \method.} In Stage~\protect\circledblue{1}, we estimate the contextual payoff matrix using the debiased payoff matrix estimator (DPME). In Stage~\protect\circledblue{2}, we use the estimated payoff matrix to learn the context-to-equilibrium mapping via an orthogonal loss.}
\label{fig:framework}
\end{figure}

In Stage~\circledblue{2}, we learn the the context-to-equilibrium mapping $\pi_\theta(\cdot)$ using the estimated contextual PM from Stage~\circledblue{1}. A straightforward approach would be to solve the estimated contextual games fore each context and then use the resulting equilibria as  labels. However, repeatedly solving such games is computationally expensive and thus \emph{not} scalable.

Instead, we directly learn the context-to-equilibrium mapping using a tailored, differentiable training loss using gradient-based optimation. To construct this loss, we build upon the \emph{gap function}, which reformulates equilibrium computation as a differentiable optimization problem for a fixed payoff matrix ~\citep{Auslender.1973, Fukushima.1992} (see Appendix~\ref{subsec:gap-function} for a background). We extend the gap function to our contextual setting with stochastic payoffs, and, to do so, we combine it with our DPM from Stage~\circledblue{1} to obtain a orthogonal loss for learning $\pi^\star(\cdot)$:
\par\vspace{3\baselineskip}\WFclear
\begin{contributionbox}
\begin{definition}[Orthogonal gap loss]
\label{def:gap-loss}
\em
For a generic nuisance $\eta$, let $\ADB(x;\eta)$ be the population pseudo-outcome regression target in Eq.~(\ref{eq:debiased-payoff-surface}). We define the population orthogonal gap loss as
\begin{align}
    \Lcal(\theta;\eta)=\E_X\Big[G_\Om\big(\pi_\theta(X);\ADB(X;\eta),X\big)\Big],
    \label{eq:gap-loss}
\end{align}
where $G_\Om$ is the regularized gap function, defined as
\begin{align}
    G_\Om(\pi;A,x)
    =\max_{q\in\DeltaK}\big\{-F_\Om(\pi,q;x)\big\}
    =\max_{q\in\DeltaK}\Big\{
    \big\langle -A\pi,\ \pi-q\big\rangle
    +\Om(\pi)-\Om(q)
    \Big\} .
    \label{eq:gap}
\end{align}
\end{definition}
\begin{theorem}[Differentiability and Neyman-orthogonality]
\label{thm:gap-loss-orthogonal}
Under Assumptions~\ref{ass:boundedness}--\ref{ass:strong-convexity}, the orthogonal gap loss $\mathcal{L}(\theta;\eta)$ is differentiable with respect to $\theta$ for every fixed nuisance $\eta$ satisfying Assumption~\ref{ass:boundedness}. Under Assumption~\ref{ass:identifiability}, the loss is also universally Neyman-orthogonal: for every $\theta\in\Theta$ and every pair of directions $(h_\theta,h_\eta)$ with bounded $h_\eta$, it holds that
\begin{align}
        D_\eta D_\theta\Lcal(\theta;\eta_0)[h_\theta,h_\eta]=0,
        \qquad \forall\theta\in\Theta.
        \label{eq:gap-orthogonality}
    \end{align}
\end{theorem}
\begin{proposition}[Oracle recovery under realizability]
\label{prop:oracle-recovery}
Let Assumptions~\ref{ass:identifiability}, \ref{ass:smoothness}, and~\ref{ass:strong-convexity} hold, and let us assume the policy class is realizable, i.e., there exists
$\theta_0\in\Theta$ such that
$\pi_{\theta_0}(x)=\pi^\star(x)$ for almost every $x$.
Then, under the oracle nuisance $\eta_0$, we have
\begin{align}
\argmin_{\theta\in\Theta}\Lcal(\theta;\eta_0)
=
\big\{\theta\in\Theta:
\pi_\theta(x)=\pi^\star(x)
\text{ for almost every }x\big\},
\end{align}
and the minimum population gap risk equals zero.
\end{proposition}
\begin{proof}
    We prove Theorem~\ref{thm:gap-loss-orthogonal} and Proposition~\ref{prop:oracle-recovery} in Appendices~\ref{app:proof-gap-loss-orthogonal} and~\ref{app:proof-oracle-recovery}, respectively.
\end{proof}
\end{contributionbox}

Intuitively, the gap function measures the largest advantage available to an optimal opponent against a candidate strategy $\pi$ and equals zero at equilibrium. We replace the fixed contextual PM $A$ with the debiased payoff matrix $\ADB(X;\eta)$, which depends on the nuisance function $\eta$, and average over contexts to obtain a population gap risk. Theorem~\ref{thm:gap-loss-orthogonal} shows that our proposed loss is universally Neyman-orthogonal and therefore locally insensitive to nuisance estimation errors. 

\textbf{Procedure:} In practice, training \method framework proceeds in two stages. \textbf{(1)}~We estimate the nuisance functions $\widehat\eta$ using $K_\textrm{CF}$-fold cross-fitting~\citep{Chernozhukov.2018}. For each fold, we fit the nuisance functions on the remaining $K_\textrm{CF}-1$ folds and evaluate the pseudo-outcomes $\Gamma_{jk}(O;\widehat\eta)$ on the held-out fold. We then regress the cross-fitted pseudo-outcomes on $X$ to obtain $\AhatDB(\cdot;\widehat\eta)$. \textbf{(2)}~We minimize the empirical orthogonal gap loss to obtain $\pi_{\widehat \theta}$:
\begin{align}
    \widehat\theta
    \in\argmin_{\theta\in\Theta}\widehat{\Lcal}_n(\theta),
    \quad \text{where }
    \widehat{\Lcal}_n(\theta)
    :=\frac{1}{n}\sum_{i=1}^n
    G_\Om\!\left(
        \pi_\theta(X_i);\AhatDB(X_i;\widehat\eta),X_i
    \right).
    \label{eq:empirical-gap-loss}
\end{align}
At inference time, the learned parameter $\widehat\theta$ is used to predict the contextual equilibrium $\pi_{\widehat\theta}(x_\star)$ for a new context $x_\star\in X$. The corresponding support $\text{supp}(\pi_{\widehat\theta}(x_\star))$ then provides the predicted set of winners that best align with the collective preference. For a downstream decision-making task with some utility function $\mathcal{U}(k, x)$, one can select a deterministic winner via $\argmax_{k \in \text{supp}(\pi_{\widehat\theta}(x_\star))} \mathcal{U}(k, x_\star)$.

\subsection{Theoretical guarantees}
\label{subsec:theoretical-benefits}

We provide theoretical guarantees for the learned equilibrium under our empirical orthogonal gap loss. To measure the quality of the equilibrium, we introduce the \emph{regularized exploitability}, which is the regularized analogue of a standard measure of equilibrium quality in two-player zero-sum games and which is commonly used to evaluate how close a learned policy is to the true Nash equilibrium~\citep{Zinkevich.2007}. For a symmetric game with context $x$, the exploitability of the learned policy $\pi_{\widehat\theta}$ under the true contextual payoff matrix $A(x)$ is given by the regularized gap function
\begin{align}
    \label{eq:exploitability}
    \Exploit_{A(x)}\big(\pi_{\widehat\theta}(x)\big)
    :=G_\Om\big(\pi_{\widehat\theta}(x);A(x),x\big).
\end{align}
We define the the corresponding average exploitability as $\Exploit(\pi_{\widehat\theta}):=\E_X[\Exploit_{A(X)}(\pi_{\widehat\theta}(X))]$, where $X\sim P_X$.

\begin{theorem}[Empirical error bound for average exploitability]
\label{thm:avg-exploit}
Let Assumptions~\ref{ass:identifiability}--\ref{ass:strong-convexity} and~\ref{ass:cross-fitting}--\ref{ass:estimation-rates} in Appendix~\ref{app:proof-assumptions} hold. Further, let $r_\mu,r_e\to0$ be the nuisance convergence rates and $e_A$ the payoff regression error in Assumption~\ref{ass:estimation-rates}. Let $\widehat{\Lcal}$ be the population counterpart of $\widehat{\Lcal}_n$. Then, the empirical minimizer $\widehat\theta$ in Eq.~(\ref{eq:empirical-gap-loss}) satisfies
\begin{equation}
\Exploit(\pi_{\widehat\theta})\le4\inf_{\theta\in\Theta}\Exploit(\pi_\theta)+4\sup_{\theta\in\Theta}\big|\widehat{\Lcal}(\theta)-\widehat{\Lcal}_n(\theta)\big|+O_p\!\left(\frac{K^2}{\kappa}\big\{e_A^2+r_\mu^4+\elb^{-2}r_\mu^2r_e^2\big\}\right).
\label{eq:avg-exploit}
\end{equation}
\end{theorem}

\begin{corollary}[Quasi-oracle convergence rate]
\label{cor:quasi-oracle}
Under standard assumptions (see Appendix~\ref{app:proof-quasi-oracle}), the empirical minimizer $\widehat\theta$ achieves the same exploitability rate as the minimization with oracle payoff matrix $A(x)$. 
\end{corollary}
\begin{proof}
    See Appendix~~\ref{app:proof-avg-exploit}, \ref{app:proof-quasi-oracle} for the proof.
\end{proof}
Theorem~\ref{thm:avg-exploit} shows that, for \method, the nuisance estimation errors affect the exploitability rate only through higher-order terms, whereas, for the plug-in learner, the corresponding contribution of the nuisance error term would be $r_\mu^2$). In the following, we verify the theoretical advantage of \method over the plug-in learner through numerical experiments.

\section{Experiments}
\label{sec:experiments}

\textbf{Baselines.} We group the baselines into two categories:: (1) \underline{Global+intransitive} baselines: \textbf{Maximal lottery}~\citep{Lanctot.2023}, \textbf{Robust maximal lottery}~\citep{Khalaf.2026}, and \textbf{Pluralistic leaderboard}~\citep{Haghtalab.2026} (2) \underline{Contextual+transitive} baselines: Bradley--Terry (\textbf{BT-Plug-in})~\citep{Bradley.1952}, \textbf{BT-regularized}, and \textbf{BT-debiased}~\citep{Frauen.2026}. We refer to our method as \textbf{NashEval-debiased}. We further compare to plug-in learner from Section~\ref{subsec:dpme} called \textbf{NashEval-plug-in} to demonstrate the benefits of debiasing. \textbf{NashEval-oracle} serves as oracle upper-bound on the performance; therein, we replace $\widehat \eta$ in \method by the true $\eta_0$ and keep everything else unchanged. Since our framework is model-agnostic, we instantiate \method and the baselines with the same machine learning backbone to ensure fair comparison (see Appendix~\ref{app:implementation} for details).

\textbf{Metrics.} We evaluate equilibrium quality through average unregularized exploitability and winner-set recovery using the $F_1$ score of the predicted equilibrium support. For \rqtwo, we additionally report the head-to-head game value, and, for the cost analysis, the average strict win rate. See Appendix~\ref{subsec:eval-metrics} for details.

\noindent\textbf{\rqone Does \method improve a equilibrium learning?}
\label{para:rq1}
We first verify the theoretical benefits of our method using a simulated dataset, where the true equilibrium is known. We consider three forms of relative feedback generated by $N_v$ voters with $N_v=1,3,5$. We simulate with $K=5$ agents and $d=5$-dimensional contexts (see Appendix~\ref{app:synthetic-dgp} for details). \underline{\textbf{Results.}} Figure~\ref{fig:rq1-results} shows the results. (i)~Across all three feedback forms, our \method framework outperforms the baselines when $N_v\geq 3$ and thus performs best. In particular, \textbf{NashEval-debiased} consistently performs closest to \textbf{NashEval-oracle} in terms of both exploitability and support $F_1$. (ii)~When $N_v=1$, preferences are transitive, and \method performs comparably to \textbf{BT-debiased}, which is expected.  (iii)~\textbf{NashEval-debiased} achieves better exploitability and $F_1$ scores than the \textbf{NashEval-plug-in}, which demonstrates the benefit of Neyman-orthogonality. 
(iv)~Finally, \textbf{BT-Reg} and \textbf{BT-debiased} have higher exploitability and lower $F_1$ score for $N_v\geq 3$. This suggests that regularization or debiasing alone cannot overcome misspecification from imposing transitivity when collective preferences are intransitive.

\begin{wrapfigure}{r}{0.55\linewidth}
\centering
\includegraphics[width=\linewidth]{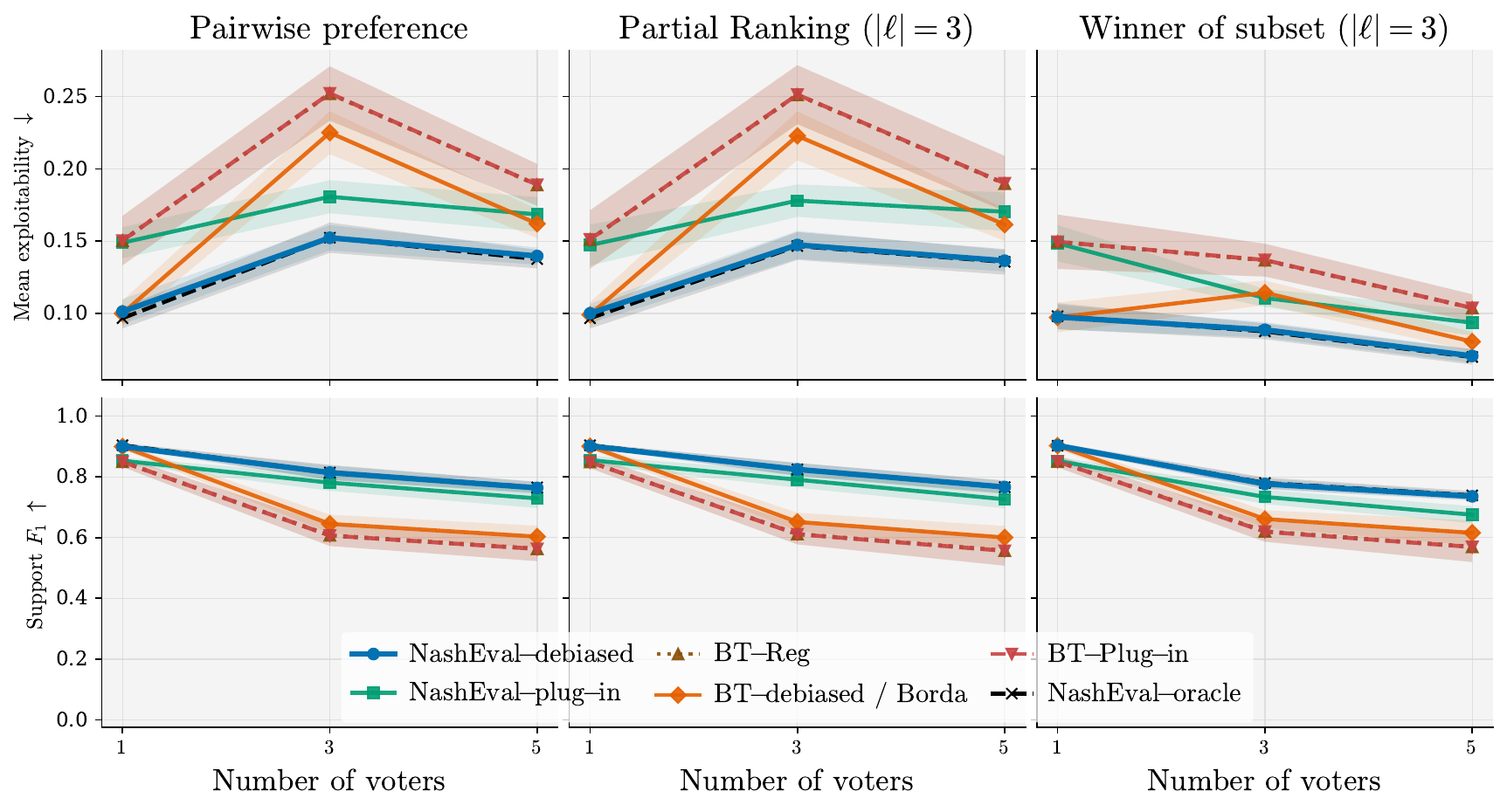}
\caption{\rqone \textbf{Equilibrium quality using a simulated dataset with ground-truth.} Shown: mean exploitability and $F_1$ score on simulated dataset for $\#$voters$=1,3,5$ over 5 random seeds.}
\label{fig:rq1-results}
\end{wrapfigure}

\noindent\textbf{\rqtwo How does \method perform in larger-scale, realistic agent evaluation settings?}
We evaluate \method in two large-scale real-world agent evaluation settings: \textbf{(1)}~evaluating LLMs on RouterBench~\citep{Hu.2024RouterBench}, using the model evaluations of \citet{Li.2026llmrouterbench}, and \textbf{(2)}~evaluating reinforcement learning (RL) agents in the multi-objective highway-driving environment of MO-Gymnasium~\citep{Felten.2023Toolkit}, which is built on HighwayEnv~\citep{Leurent.2018Highway}. For RouterBench, we use ten LLMs and $8{,}742$ prompts. For the driving task, we train five RL agents on 256 configurations (i.e., context) with different reward models. The pairwise preference data are sampled from three external judges, each of which has internally transitive preferences (see Appendix~\ref{subsec:routerbench-dataset} and \ref{subsec:rl-dataset} for details). We then measure how closely our \method is to to the true contextual evaluation. We further report the \emph{game value} against \textbf{NashEval-debiased}, which quantifies whether each baseline has a higher or lower \emph{collective} preference compared to \method.

\begin{table}[t]
\centering
\begin{minipage}[t]{0.49\linewidth}
\centering
\caption{Performance comparison for LLMs using RouterBench \citep{Hu.2024RouterBench}.}
\label{tab:rq2-routerbench}
\small
\setlength{\tabcolsep}{4pt}
\resizebox{\linewidth}{!}{%
\begin{tabular}{@{}lccc@{}}
\toprule
Method & Exploitability [$\downarrow$] & Support $F_1$ [$\uparrow$] & Game value [$\uparrow$] \\
\midrule
NashEval-debiased \textbf{(ours)} & $\mathbf{0.380 \pm 0.000}$ & $\mathbf{0.300 \pm 0.001}$ & $\mathbf{0.000 \pm 0.000}$ \\
NashEval-plug-in & $0.398 \pm 0.001$ & $0.287 \pm 0.001$ & $-0.016 \pm 0.001$ \\
BT-debiased / Borda & $0.395 \pm 0.001$ & $0.274 \pm 0.001$ & $-0.010 \pm 0.001$ \\
BT-Plug-in & $0.401 \pm 0.003$ & $0.270 \pm 0.002$ & $0.014 \pm 0.003$ \\
Maximal lottery & $0.423 \pm 0.000$ & $0.246 \pm 0.000$ & $-0.030 \pm 0.000$ \\
Robust maximal lottery & $0.423 \pm 0.000$ & $0.246 \pm 0.000$ & $-0.030 \pm 0.000$ \\
Pluralistic leaderboard & $0.470 \pm 0.000$ & $0.242 \pm 0.000$ & $-0.136 \pm 0.000$ \\
\bottomrule
\end{tabular}%
}
\end{minipage}\hfill
\begin{minipage}[t]{0.49\linewidth}
\centering
\caption{Performance comparison for RL agents using the MO-Highway benchmark \citep{Felten.2023Toolkit}.}
\label{tab:rq2-rl-agents}
\small
\setlength{\tabcolsep}{4pt}
\resizebox{\linewidth}{!}{%
\begin{tabular}{@{}lccc@{}}
\toprule
Method & Exploitability [$\downarrow$] & Support $F_1$ [$\uparrow$] & Game value [$\uparrow$] \\
\midrule
NashEval-debiased \textbf{(ours)} & $\mathbf{0.066 \pm 0.007}$ & $0.708 \pm 0.023$ & $\mathbf{0.000 \pm 0.000}$ \\
NashEval-plug-in & $0.068 \pm 0.000$ & $\mathbf{0.713 \pm 0.000}$ & $-0.027 \pm 0.006$ \\
BT-debiased / Borda & $0.099 \pm 0.003$ & $0.525 \pm 0.014$ & $-0.017 \pm 0.004$ \\
BT-Plug-In & $0.232 \pm 0.000$ & $0.438 \pm 0.000$ & $-0.072 \pm 0.004$ \\
Maximal lottery & $0.105 \pm 0.000$ & $0.505 \pm 0.000$ & $-0.006 \pm 0.006$ \\
Robust maximal lottery & $0.105 \pm 0.000$ & $0.505 \pm 0.000$ & $-0.006 \pm 0.006$ \\
Pluralistic leaderboard & $0.217 \pm 0.000$ & $0.556 \pm 0.000$ & $-0.147 \pm 0.004$ \\
\bottomrule
\end{tabular}%
}
\end{minipage}
\end{table}

\underline{\textbf{Results.}} Tables~\ref{tab:rq2-routerbench} and~\ref{tab:rq2-rl-agents} show the results. (1)~Overall, our NashEval-debiased achieves the lowest mean exploitability on both RouterBench and the RL-agent benchmark and thus performs best. (2)~On both real-world benchmarks, it also achieves lower exploitability and higher support $F_1$ than \textbf{BT-debiased / Borda} and the global evaluation baselines. This 
confirms that our \method is more accurate in recovering the contextual evaluation under intransitive preferences. (3)~The head-to-head game values favor \textbf{NashEval-debiased} over all competitors on both benchmarks. This demonstrates that \textbf{NashEval-debiased} better reflects the underlying collective preferences.

\underline{\textbf{Insights.}} 
We examine the frequency of how often each agent is selected as the winner in the RouterBench dataset. We use the same setup as in RQ2. We divide the test data by domain and then compute the average winning rate of each agent by domain. We visualize the frequency as heatmap in Figure~\ref{fig:rb-heatmap}. We observe that the estimated evaluation is largely consistent with the approximated true evaluation, which provides further evidence that \method is effective in capture context-dependent variation in agent performance. For example, frontier reasoning-based models (e.g., GPT-5, Gemini 2.5 Pro) perform better on the Math domain, in line with our expectation.  

\begin{figure}
\centering
\includegraphics[width=0.7\linewidth]{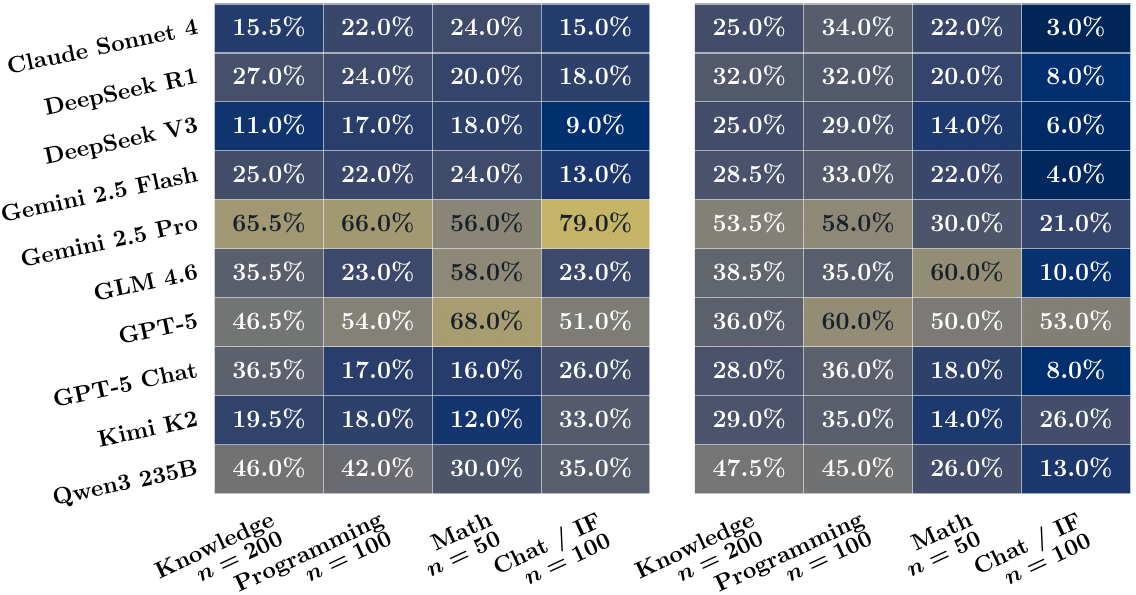}
\caption{\textbf{Insights.} Winner frequency for each agent in the RouterBench dataset. \textit{Left:} estimated evaluation from \textbf{NashEval-debiased}. \textit{Right:} (approximated) true evaluation.}
\label{fig:rb-heatmap}
\end{figure}

\textbf{Conclusion.} We develop, to the best of our knowledge, the first game-theoretic framework for context-dependent agent evaluation under potentially intransitive collective preferences. Our experiments confirm the theoretical properties of \method and demonstrate the benefit of our Neyman-orthogonal equilibrium learning.

\bibliography{references}

@inproceedings{Ong.2025,
  author = {Ong, Isaac and Almahairi, Amjad and Wu, Vincent and Chiang, Wei-Lin and Wu, Tianhao and Gonzalez, Joseph E. and Kadous, M Waleed and Stoica, Ion},
  title = {{RouteLLM}: Learning to route {LLM}s from preference data},
  booktitle = {ICLR},
  year = {2025}
}

@inproceedings{Munos.2024,
  author = {Munos, R{\'e}mi and Valko, Michal and Calandriello, Daniele and Gheshlaghi Azar, Mohammad and Rowland, Mark and Guo, Zhaohan Daniel and Tang, Yunhao and Geist, Matthieu and Mesnard, Thomas and Fiegel, C{\^o}me and Michi, Andrea and Selvi, Marco and Girgin, Sertan and Momchev, Nikola and Bachem, Olivier and Mankowitz, Daniel J. and Precup, Doina and Piot, Bilal},
  title = {{Nash} learning from human feedback},
  booktitle = {ICML},
  year = {2024}
}

@inproceedings{Christiano.2017,
  author = {Christiano, Paul F. and Leike, Jan and Brown, Tom B. and Martic, Miljan and Legg, Shane and Amodei, Dario},
  title = {Deep Reinforcement Learning from Human Preferences},
  booktitle = {NeurIPS},
  year = {2017}
}

@inproceedings{Ouyang.2022,
  author = {Ouyang, Long and Wu, Jeffrey and Jiang, Xu and Almeida, Diogo and Wainwright, Carroll L. and Mishkin, Pamela and Zhang, Chong and Agarwal, Sandhini and Slama, Katarina and Ray, Alex and Schulman, John and Hilton, Jacob and Kelton, Fraser and Miller, Luke and Simens, Maddie and Askell, Amanda and Welinder, Peter and Christiano, Paul F. and Leike, Jan and Lowe, Ryan},
  title = {Training Language Models to Follow Instructions with Human Feedback},
  booktitle = {NeurIPS},
  year = {2022}
}

@inproceedings{Rafailov.2023,
  author = {Rafailov, Rafael and Sharma, Archit and Mitchell, Eric and Manning, Christopher D. and Ermon, Stefano and Finn, Chelsea},
  title = {Direct Preference Optimization: Your Language Model Is Secretly a Reward Model},
  booktitle = {NeurIPS},
  year = {2023}
}

@article{Rosset.2024,
  author = {Rosset, Corby and Cheng, Ching-An and Mitra, Arindam and Santacroce, Michael and Awadallah, Ahmed and Xie, Tengyang},
  title = {Direct {Nash} Optimization: Teaching Language Models to Self-Improve with General Preferences},
  journal = {arXiv preprint arXiv:2404.03715},
  year = {2024}
}

@inproceedings{Wu.2025sppo,
  author = {Wu, Yue and Sun, Zhiqing and Yuan, Huizhuo and Ji, Kaixuan and Yang, Yiming and Gu, Quanquan},
  title = {Self-Play Preference Optimization for Language Model Alignment},
  booktitle = {ICLR},
  year = {2025}
}

@inproceedings{Zhang.2025inpo,
  author = {Zhang, Yuheng and Yu, Dian and Peng, Baolin and Song, Linfeng and Tian, Ye and Huo, Mingyue and Jiang, Nan and Mi, Haitao and Yu, Dong},
  title = {Iterative {Nash} Policy Optimization: Aligning {LLM}s with General Preferences via No-Regret Learning},
  booktitle = {ICLR},
  year = {2025}
}

@article{Bradley.1952,
  author = {Bradley, Ralph Allan and Terry, Milton E.},
  title = {Rank analysis of incomplete block designs: I. The method of paired comparisons},
  journal = {Biometrika},
  volume = {39},
  number = {3/4},
  pages = {324--345},
  year = {1952}
}

@article{Yue.2012,
  title = {The K-armed dueling bandits problem},
  journal = {Journal of Computer and System Sciences},
  author = {Yisong Yue and Josef Broder and Robert Kleinberg and Thorsten Joachims},
  volume = {78},
  number = {5},
  pages = {1538-1556},
  year = {2012}
}

@inproceedings{Saha.2021,
  author = {Saha, Aadirupa},
  title = {Optimal algorithms for stochastic contextual preference bandits},
  booktitle = {NeurIPS},
  year = {2021}
}

@inproceedings{Bengs.2022,
  author = {Bengs, Viktor and Saha, Aadirupa and H{\"u}llermeier, Eyke},
  title = {Stochastic contextual dueling bandits under linear stochastic transitivity models},
  booktitle = {ICML},
  pages = {1764--1786},
  year = {2022}
}

@inproceedings{Di.2024,
  author = {Di, Qiwei and Jin, Tao and Wu, Yue and Zhao, Heyang and Farnoud, Farzad and Gu, Quanquan},
  title = {Variance-aware regret bounds for stochastic contextual dueling bandits},
  booktitle = {ICLR},
  year = {2024}
}

@inproceedings{Wu.2024,
  author = {Wu, Yue and Jin, Tao and Lou, Hao and Farnoud, Farzad and Gu, Quanquan},
  title = {Borda regret minimization for generalized linear dueling bandits},
  booktitle = {ICML},
  year = {2024}
}

@inproceedings{Frauen.2026,
  author = {Frauen, Dennis and Deviyani, Athiya and van der Schaar, Mihaela and Feuerriegel, Stefan},
  title = {Nonparametric {LLM} evaluation from preference data},
  booktitle = {ICML},
  year = {2026}
}

@article{Rubin.1974,
  author  = {Rubin, Donald B.},
  title   = {Estimating causal effects of treatments in randomized and nonrandomized studies},
  journal = {Journal of Educational Psychology},
  volume  = {66},
  number  = {5},
  pages   = {688--701},
  year    = {1974}
}

@article{Rubin.1976,
  author  = {Rubin, Donald B.},
  title   = {Inference and missing data},
  journal = {Biometrika},
  volume  = {63},
  number  = {3},
  pages   = {581--592},
  year    = {1976}
}

@article{Chernozhukov.2018,
  author = {Chernozhukov, Victor and Chetverikov, Denis and Demirer, Mert and Duflo, Esther and Hansen, Christian and Newey, Whitney and Robins, James},
  title = {Double/debiased machine learning for treatment and structural parameters},
  journal = {The Econometrics Journal},
  volume = {21},
  number = {1},
  pages = {C1--C68},
  year = {2018},
  doi = {10.1111/ectj.12097}
}

@article{Cui.2022,
  title={When are offline two-player zero-sum Markov games solvable?},
  author={Cui, Qiwen and Du, Simon},
  journal={NeurIPS},
  year={2022}
}

@article{Zhang.2026,
  author = {Zhang, Yuheng and Chen, Claire and Jiang, Nan},
  title = {Beyond pessimism: Offline learning in {KL}-regularized games},
  journal = {arXiv preprint arXiv:2604.06738},
  year = {2026}
}

@inproceedings{Zhong.2022,
  author = {Zhong, Han and Xiong, Wei and Tan, Jiyuan and Wang, Liwei and Zhang, Tong and Wang, Zhaoran and Yang, Zhuoran},
  title = {Pessimistic Minimax Value Iteration: Provably Efficient Equilibrium Learning from Offline Datasets},
  booktitle = {ICML},
  series = {Proceedings of Machine Learning Research},
  volume = {162},
  pages = {27117--27142},
  year = {2022}
}

@inproceedings{Chen.2026,
  author = {Chen, Claire and Zhang, Yuheng and Liu, Xinyu and Xie, Zixuan and Liu, Shuze Daniel and Jiang, Nan},
  title = {Offline Two-Player Zero-Sum Markov Games with {KL} Regularization},
  booktitle = {ICML},
  year = {2026},
  note = {arXiv:2605.13025}
}

@unpublished{Copeland.1951,
  author = {Copeland, Arthur H.},
  title  = {A ``Reasonable'' Social Welfare Function},
  note   = {Mimeographed notes, Seminar on Applications of Mathematics
            to the Social Sciences, University of Michigan},
  year   = {1951}
}

@inproceedings{Kreweras.1965,
  author = {Kreweras, Germain},
  title = {Aggregation of preference orderings},
  booktitle = {Mathematics and Social Sciences I},
  year = {1965}
}

@article{Kramer.1973,
 author = {Gerald H. Kramer},
 journal = {Econometrica},
 number = {2},
 pages = {285--297},
 title = {On a Class of Equilibrium Conditions for Majority Rule},
 volume = {41},
 year = {1973}
}

@article{Fishburn.1984a,
  author = {Fishburn, Peter C.},
  title = {Probabilistic social choice based on simple voting comparisons},
  journal = {The Review of Economic Studies},
  volume = {51},
  number = {4},
  pages = {683--692},
  year = {1984}
}

@article{Brandl.2016,
  author = {Brandl, Florian and Brandt, Felix and Seedig, Hans Georg},
  title = {Consistent probabilistic social choice},
  journal = {Econometrica},
  volume = {84},
  number = {5},
  pages = {1839--1880},
  year = {2016}
}

@article{Brandl.2020,
  author = {Brandl, Florian and Brandt, Felix},
  title = {Arrovian aggregation of convex preferences},
  journal = {Econometrica},
  volume = {88},
  number = {2},
  pages = {799--844},
  year = {2020},
  doi = {10.3982/ECTA15749}
}

@article{Brandl.2022,
  author = {Brandl, Florian and Brandt, Felix and Stricker, Christian},
  title = {An analytical and experimental comparison of maximal lottery schemes},
  journal = {Social Choice and Welfare},
  volume = {58},
  number = {1},
  pages = {5--38},
  year = {2022},
  doi = {10.1007/s00355-021-01326-x}
}

@article{Tideman.1987,
  author  = {Tideman, T. Nicolaus},
  title   = {Independence of clones as a criterion for voting rules},
  journal = {Social Choice and Welfare},
  volume  = {4},
  number  = {3},
  pages   = {185--206},
  year    = {1987},
  doi     = {10.1007/BF00433944}
}

@article{Laffond.1993,
  author = {Laffond, Gilbert and Laslier, Jean-Fran{\c c}ois and Le Breton, Michel},
  title = {The bipartisan set of a tournament game},
  journal = {Games and Economic Behavior},
  volume = {5},
  number = {1},
  pages = {182--201},
  year = {1993},
  doi = {10.1006/game.1993.1010}
}

@article{Dutta.1999,
  author = {Dutta, Bhaskar and Laslier, Jean-Fran{\c c}ois},
  title = {Comparison functions and choice correspondences},
  journal = {Social Choice and Welfare},
  volume = {16},
  number = {4},
  pages = {513--532},
  year = {1999}
}

@book{Laslier.1997,
  author = {Laslier, Jean-Fran{\c c}ois},
  title = {Tournament solutions and majority voting},
  publisher = {Springer-Verlag},
  year = {1997}
}

@article{Felsenthal.1992,
  author = {Felsenthal, Dan S. and Machover, Mosh{\'e}},
  title = {After two centuries, should {Condorcet}'s voting procedure be implemented?},
  journal = {Behavioral Science},
  volume = {37},
  number = {4},
  pages = {250--274},
  year = {1992}
}

@inproceedings{Rivest.2010,
  author = {Rivest, Ronald L. and Shen, Emily},
  title = {An optimal single-winner preferential voting system based on game theory},
  booktitle = {COMSOC},
  year = {2010}
}

@inproceedings{Dudik.2015,
  author = {Dud{\'i}k, Miroslav and Hofmann, Katja and Schapire, Robert E. and Slivkins, Aleksandrs and Zoghi, Masrour},
  title = {Contextual dueling bandits},
  booktitle = {COLT},
  year = {2015}
}

@incollection{Brandt.2017,
  author = {Brandt, Felix},
  title = {Rolling the dice: Recent results in probabilistic social choice},
  booktitle = {Trends in Computational Social Choice},
  year = {2017}
}

@inproceedings{Sokota.2022,
      title={A unified approach to reinforcement learning, quantal response equilibria, and two-player zero-sum games}, 
      author={Samuel Sokota and Ryan D'Orazio and J. Zico Kolter and Nicolas Loizou and Marc Lanctot and Ioannis Mitliagkas and Noam Brown and Christian Kroer},
      year={2023},
      booktitle={International Conference on Learning Representations}
}

@book{vanderVaart.1998,
  author    = {van der Vaart, A. W.},
  title     = {Asymptotic Statistics},
  series    = {Cambridge Series in Statistical and Probabilistic Mathematics},
  publisher = {Cambridge University Press},
  year      = {1998},
  doi       = {10.1017/CBO9780511802256}
}

@book{Bickel.1998,
  author    = {Bickel, Peter J. and Klaassen, Chris A. J. and Ritov, Ya'acov and Wellner, Jon A.},
  title     = {Efficient and Adaptive Estimation for Semiparametric Models},
  publisher = {Springer},
  address   = {New York},
  year      = {1998},
  isbn      = {978-0-387-98473-5},
  note      = {Reprint of the 1993 original}
}

@article{Hampel.1974,
  author  = {Hampel, Frank R.},
  title   = {The Influence Curve and Its Role in Robust Estimation},
  journal = {Journal of the American Statistical Association},
  volume  = {69},
  number  = {346},
  pages   = {383--393},
  year    = {1974}
}

@article{Nie.2021,
  author  = {Nie, Xinkun and Wager, Stefan},
  title   = {Quasi-oracle estimation of heterogeneous treatment effects},
  journal = {Biometrika},
  volume  = {108},
  number  = {2},
  pages   = {299--319},
  year    = {2021}
}

@article{Kennedy.2023,
  author  = {Kennedy, Edward H.},
  title   = {Towards optimal doubly robust estimation of heterogeneous causal effects},
  journal = {Electronic Journal of Statistics},
  volume  = {17},
  number  = {2},
  pages   = {3008--3049},
  year    = {2023},
  doi     = {10.1214/23-EJS2157}
}

@incollection{Kennedy.2024review,
  author    = {Kennedy, Edward H.},
  title     = {Semiparametric Doubly Robust Targeted Double Machine Learning: A Review},
  booktitle = {Handbook of Statistical Methods for Precision Medicine},
  pages     = {207--236},
  year      = {2024},
  publisher = {Chapman and Hall/CRC}
}

@article{Foster.2023,
  author  = {Foster, Dylan J. and Syrgkanis, Vasilis},
  title   = {Orthogonal statistical learning},
  journal = {The Annals of Statistics},
  volume  = {51},
  number  = {3},
  pages   = {879--908},
  year    = {2023},
  doi     = {10.1214/23-AOS2258}
}

@inproceedings{Frauen.2025survival,
  author    = {Frauen, Dennis and Schr{\"o}der, Maresa and Hess, Konstantin and Feuerriegel, Stefan},
  title     = {Orthogonal Survival Learners for Estimating Heterogeneous Treatment Effects from Time-to-Event Data},
  booktitle = {NeurIPS},
  year      = {2025}
}

@inproceedings{Melnychuk.2026,
  author    = {Melnychuk, Valentyn and Frauen, Dennis and Schweisthal, Jonas and Feuerriegel, Stefan},
  title     = {Orthogonal Representation Learning for Estimating Causal Quantities},
  booktitle = {AISTATS},
  year      = {2026}
}

@inproceedings{Dudik.2011,
  author    = {Dud{\'i}k, Miroslav and Langford, John and Li, Lihong},
  title     = {Doubly Robust Policy Evaluation and Learning},
  booktitle = {ICML},
  year      = {2011}
}

@inproceedings{Jiang.2016,
  author    = {Jiang, Nan and Li, Lihong},
  title     = {Doubly Robust Off-policy Value Evaluation for Reinforcement Learning},
  booktitle = {ICML},
  year      = {2016}
}

@inproceedings{Abe.2021,
  author = {Abe, Kenshi and Kaneko, Yusuke},
  title = {Off-Policy Exploitability-Evaluation in Two-Player Zero-Sum Markov Games},
  year = {2021},
  booktitle = {AAMAS}
}

@inproceedings{Swamy.2024,
author = {Swamy, Gokul and Dann, Christoph and Kidambi, Rahul and Wu, Zhiwei Steven and Agarwal, Alekh},
title = {A minimaximalist approach to reinforcement learning from human feedback},
year = {2024},
booktitle = {ICML}
}

@book{Facchinei.2003,
  author    = {Facchinei, Francisco and Pang, Jong-Shi},
  title     = {Finite-Dimensional Variational Inequalities and Complementarity Problems},
  series    = {Springer Series in Operations Research},
  publisher = {Springer},
  address   = {New York},
  year      = {2003}
}

@book{Dontchev.2009,
  author    = {Dontchev, Asen L. and Rockafellar, R. Tyrrell},
  title     = {Implicit Functions and Solution Mappings: A View from Variational Analysis},
  series    = {Springer Monographs in Mathematics},
  publisher = {Springer},
  year      = {2009}
}

@article{Auslender.1973,
  author  = {Auslender, A.},
  title   = {Br{\`e}ve communication. R{\'e}solution num{\'e}rique
             d'in{\'e}galit{\'e}s variationnelles},
  journal = {Revue fran{\c c}aise d'automatique informatique
             recherche op{\'e}rationnelle. Math{\'e}matique},
  volume  = {7},
  number  = {R2},
  pages   = {67--72},
  year    = {1973}
}

@article{Fukushima.1992,
  author  = {Fukushima, Masao},
  title   = {Equivalent differentiable optimization problems and descent methods for asymmetric variational inequality problems},
  journal = {Mathematical Programming},
  volume  = {53},
  number  = {1},
  pages   = {99--110},
  year    = {1992},
  doi     = {10.1007/BF01585696}
}

@article{Larsson.1994,
  author  = {Larsson, Torbj{\"o}rn and Patriksson, Michael},
  title   = {A class of gap functions for variational inequalities},
  journal = {Mathematical Programming},
  volume  = {64},
  number  = {1},
  pages   = {53--79},
  year    = {1994}
}

@article{Kohlberg.1986,
  author  = {Kohlberg, Elon and Mertens, Jean-Fran{\c c}ois},
  title   = {On the strategic stability of equilibria},
  journal = {Econometrica},
  volume  = {54},
  number  = {5},
  pages   = {1003--1037},
  year    = {1986},
  doi     = {10.2307/1912320}
}

@article{Sion.1958,
  author  = {Sion, Maurice},
  title   = {On General Minimax Theorems},
  journal = {Pacific Journal of Mathematics},
  volume  = {8},
  number  = {1},
  pages   = {171--176},
  year    = {1958}
}

@inproceedings{Khalaf.2026,
  title={Robust {AI} evaluation through maximal lotteries},
  author={Khalaf, Hadi and Wang, Serena L and Halpern, Daniel and Shapira, Itai and Calmon, Flavio du Pin and Procaccia, Ariel D},
  booktitle={ICML},
  year={2026}
}

@article{Lanctot.2023,
  title={Evaluating agents using social choice theory},
  author={Lanctot, Marc and Larson, Kate and Bachrach, Yoram and Marris, Luke and Li, Zun and Bhoopchand, Avishkar and Anthony, Thomas and Tanner, Brian and Koop, Anna},
  journal={arXiv preprint arXiv:2312.03121},
  year={2023}
}

@inproceedings{Ameli.2025,
  author = {Ameli, Siavash and Zhuang, Siyuan and Stoica, Ion and Mahoney, Michael W.},
  title = {A statistical framework for ranking {LLM}-based chatbots},
  booktitle = {ICLR},
  year = {2025}
}

@inproceedings{Balduzzi.2018,
  author = {Balduzzi, David and Tuyls, Karl and Perolat, Julien and Graepel, Thore},
  title = {Re-evaluating evaluation},
  booktitle = {NeurIPS},
  year = {2018}
}

@article{Omidshafiei.2019,
  author = {Omidshafiei, Shayegan and Papadimitriou, Christos and Piliouras, Georgios and Tuyls, Karl and Rowland, Mark and Lespiau, Jean-Baptiste and Czarnecki, Wojciech M. and Lanctot, Marc and Perolat, Julien and Munos, Remi},
  title = {{$\alpha$}-{Rank}: Multi-Agent Evaluation by Evolution},
  journal = {Scientific Reports},
  volume = {9},
  pages = {9937},
  year = {2019}
}

@inproceedings{Rowland.2019,
  author    = {Rowland, Mark and Omidshafiei, Shayegan and Tuyls, Karl and Perolat, Julien and Valko, Michal and Piliouras, Georgios and Munos, R{\'e}mi},
  title     = {Multiagent Evaluation under Incomplete Information},
  booktitle = {NeurIPS},
  year      = {2019},
}

@inproceedings{Chiang.2024,
  author = {Chiang, Wei-Lin and Zheng, Lianmin and Sheng, Ying and Angelopoulos, Anastasios Nikolas and Li, Tianle and Li, Dacheng and Zhu, Banghua and Zhang, Hao and Jordan, Michael I. and Gonzalez, Joseph E. and Stoica, Ion},
  title = {{Chatbot Arena}: An open platform for evaluating {LLM}s by human preference},
  booktitle = {ICML},
  year = {2024}
}

@inproceedings{Chi.2025,
  author    = {Chi, Wayne and Chen, Valerie and Angelopoulos, Anastasios Nikolas and Chiang, Wei-Lin and Mittal, Aditya and Jain, Naman and Zhang, Tianjun and Stoica, Ion and Donahue, Chris and Talwalkar, Ameet},
  title     = {{Copilot Arena}: A Platform for Code {LLM} Evaluation in the Wild},
  booktitle = {Proceedings of the 42nd International Conference on Machine Learning},
  series    = {Proceedings of Machine Learning Research},
  volume    = {267},
  pages     = {10354--10382},
  publisher = {PMLR},
  year      = {2025},
}

@article{Frick.2025,
  author = {Frick, Evan and Chen, Connor and Tennyson, Joseph and Li, Tianle and Chiang, Wei-Lin and Angelopoulos, Anastasios N. and Stoica, Ion},
  title = {Prompt-to-leaderboard},
  journal = {arXiv preprint arXiv:2502.14855},
  year = {2025}
}

@book{Condorcet.1785,
  author = {de Condorcet, Marquis},
  title = {Essai sur l'application de l'analyse {\`a} la probabilit{\'e} des d{\'e}cisions rendues {\`a} la pluralit{\'e} des voix},
  publisher = {Imprimerie Royale},
  address = {Paris},
  year = {1785}
}

@inproceedings{Zhang.2026rankllm,
  author = {Zhang, Ziqian and Hu, Xingjian and Huang, Yue and Zhang, Kai and Chen, Ruoxi and Liu, Yixin and Wen, Qingsong and Xu, Kaidi and Zhang, Xiangliang and Gong, Neil Zhenqiang and Sun, Lichao},
  title = {{RankLLM}: Weighted ranking of {LLMs} by quantifying question difficulty},
  booktitle = {ICLR},
  year = {2026}
}

@article{Neumann.1928,
  author  = {von Neumann, John},
  title   = {Zur Theorie der Gesellschaftsspiele},
  journal = {Mathematische Annalen},
  year    = {1928},
  volume  = {100},
  number  = {1},
  pages   = {295--320}
}

@inproceedings{Korkmaz.2026,
  author    = {Korkmaz, Ezgi},
  title     = {The Axiomatic Value of Regularization in {AI} Alignment from Human Preferences},
  booktitle = {ICML},
  year      = {2026},
  note      = {Spotlight},
}

@article{Halpern.2026,
  author  = {Halpern, Daniel and Micha, Evi and Procaccia, Ariel D. and Schiffer, Benjamin and Shapira, Itai and Zhang, Shirley},
  title   = {{AI} Alignment From Social Choice Perspectives},
  journal = {arXiv preprint arXiv:2606.21550},
  year    = {2026},
  note    = {Accepted for publication in ACM SIGecom Exchanges},
}

@book{Stone.2011,
  author    = {Stone, Peter},
  title     = {The Luck of the Draw: The Role of Lotteries in Decision Making},
  publisher = {Oxford University Press},
  year      = {2011},
  doi       = {10.1093/acprof:oso/9780199756100.001.0001}
}

@article{Zhang.2025,
  author = {Zhang, Yanzhao and Li, Mingxin and Long, Dingkun and Zhang, Xin and Lin, Huan and Yang, Baosong and Xie, Pengjun and Yang, An and Liu, Dayiheng and Lin, Junyang and Huang, Fei and Zhou, Jingren},
  title = {{Qwen3} Embedding: Advancing text embedding and reranking through foundation models},
  journal = {arXiv preprint arXiv:2506.05176},
  year = {2025}
}

@article{SchechterVera.2025,
  author = {Schechter Vera, Henrique and Dua, Sahil and Zhang, Biao and Salz, Daniel and Mullins, Ryan and others},
  title = {{EmbeddingGemma}: Powerful and lightweight text representations},
  journal = {arXiv preprint arXiv:2509.20354},
  year = {2025}
}

@article{Chiang.2025dueling,
  author        = {Chiang, Chao-Kai and Ishida, Takashi and Sugiyama, Masashi},
  title         = {{LLM} Routing with Dueling Feedback},
  journal       = {arXiv preprint arXiv:2510.00841},
  year          = {2025}
}

@inproceedings{Zhao.2024eagle,
  author    = {Zhao, Zesen and Jin, Shuowei and Mao, Zhuoqing Morley},
  title     = {Eagle: Efficient Training-Free Router for Multi-{LLM} Inference},
  booktitle = {Proceedings of the Workshop on Machine Learning for Systems at NeurIPS 2024},
  pages     = {1--8},
  year      = {2024}
}

@inproceedings{Hu.2024RouterBench,
  author    = {Hu, Qitian Jason and Bieker, Jacob and Li, Xiuyu and Jiang, Nan and Keigwin, Benjamin and Ranganath, Gaurav and Keutzer, Kurt and Upadhyay, Shriyash Kaustubh},
  title     = {{RouterBench}: A Benchmark for Multi-{LLM} Routing System},
  booktitle = {ICML 2024 Agentic Markets Workshop},
  year      = {2024},
  note      = {Poster}
}

@inproceedings{Li.2026llmrouterbench,
    title = "{LLMR}outer{B}ench: A Massive Benchmark and Unified Framework for {LLM} Routing",
    author = "Li, Hao  and
      Zhang, Yiqun  and
      Guo, Zhaoyan  and
      Wang, Chenxu  and
      Tang, Shengji  and
      Zhang, Qiaosheng  and
      Chen, Yang  and
      Qi, Biqing  and
      Ye, Peng  and
      Bai, Lei  and
      Wang, Zhen  and
      Hu, Shuyue",
    editor = "Liakata, Maria  and
      Moreira, Viviane P.  and
      Zhang, Jiajun  and
      Jurgens, David",
    booktitle = "Findings of the {A}ssociation for {C}omputational {L}inguistics: {ACL} 2026",
    month = jul,
    year = {2026}
}

@inproceedings{Haghtalab.2026,
  title         = {Pluralistic leaderboards},
  author        = {Haghtalab, Nika and Procaccia, Ariel D. and Shao, Han and Wang, Serena Lutong and Yang, Kunhe},
  booktitle     = {ICML},
  year          = {2026}
}

@inproceedings{Zinkevich.2007,
  author    = {Zinkevich, Martin and Johanson, Michael and Bowling, Michael and Piccione, Carmelo},
  title     = {Regret Minimization in Games with Incomplete Information},
  booktitle = {NeurIPS},
  year      = {2007}
}

@inproceedings{Wellman.2006,
  author    = {Wellman, Michael P.},
  title     = {Methods for Empirical Game-Theoretic Analysis},
  booktitle = {AAAI},
  pages     = {1552--1556},
  year      = {2006},
}

@inproceedings{Jordan.2008,
  author    = {Jordan, Patrick R. and Vorobeychik, Yevgeniy and Wellman, Michael P.},
  title     = {Searching for Approximate Equilibria in Empirical Games},
  booktitle = {AAMAS},
  pages     = {1063--1070},
  year      = {2008},
}

@article{Vorobeychik.2010,
  author  = {Vorobeychik, Yevgeniy},
  title   = {Probabilistic Analysis of Simulation-Based Games},
  journal = {ACM Transactions on Modeling and Computer Simulation},
  volume  = {20},
  number  = {3},
  pages   = {16:1--16:25},
  year    = {2010},
  doi     = {10.1145/1842713.1842719}
}

@article{Viqueira.2019,
  author  = {Areyan Viqueira, Enrique and Cousins, Cyrus and Upfal, Eli and Greenwald, Amy},
  title   = {Learning Equilibria of Simulation-Based Games},
  journal = {arXiv preprint arXiv:1905.13379},
  year    = {2019},
}

@article{Fearnley.2015,
  author  = {Fearnley, John and Gairing, Martin and Goldberg, Paul W. and Savani, Rahul},
  title   = {Learning Equilibria of Games via Payoff Queries},
  journal = {Journal of Machine Learning Research},
  volume  = {16},
  number  = {39},
  pages   = {1305--1344},
  year    = {2015},
}

@inproceedings{Felten.2023Toolkit,
  author    = {Felten, Florian and Alegre, Lucas N. and Now{\'e}, Ann and Bazzan, Ana L. C. and Talbi, El Ghazali and Danoy, Gr{\'e}goire and Silva, Bruno C. {\relax da}},
  title     = {A Toolkit for Reliable Benchmarking and Research in Multi-Objective Reinforcement Learning},
  booktitle = {Advances in Neural Information Processing Systems},
  volume    = {36},
  year      = {2023},
  url       = {https://proceedings.neurips.cc/paper_files/paper/2023/hash/4aa8891583f07ae200ba07843954caeb-Abstract.html}
}

@misc{Leurent.2018Highway,
  author       = {Leurent, Edouard},
  title        = {An Environment for Autonomous Driving Decision-Making},
  year         = {2018},
  howpublished = {GitHub repository},
  url          = {https://github.com/Farama-Foundation/HighwayEnv}
}

@article{Mnih.2015DQN,
  author  = {Mnih, Volodymyr and Kavukcuoglu, Koray and Silver, David and Rusu, Andrei A. and Veness, Joel and Bellemare, Marc G. and Graves, Alex and Riedmiller, Martin and Fidjeland, Andreas K. and Ostrovski, Georg and Petersen, Stig and Beattie, Charles and Sadik, Amir and Antonoglou, Ioannis and King, Helen and Kumaran, Dharshan and Wierstra, Daan and Legg, Shane and Hassabis, Demis},
  title   = {Human-level control through deep reinforcement learning},
  journal = {Nature},
  volume  = {518},
  number  = {7540},
  pages   = {529--533},
  year    = {2015},
  doi     = {10.1038/nature14236}
}

@article{Raffin.2021SB3,
  author  = {Raffin, Antonin and Hill, Ashley and Gleave, Adam and Kanervisto, Anssi and Ernestus, Maximilian and Dormann, Noah},
  title   = {{Stable-Baselines3}: Reliable Reinforcement Learning Implementations},
  journal = {Journal of Machine Learning Research},
  volume  = {22},
  number  = {268},
  pages   = {1--8},
  year    = {2021},
  url     = {https://jmlr.org/papers/v22/20-1364.html}
}

@article{Bartlett.2002,
  author  = {Bartlett, Peter L. and Mendelson, Shahar},
  title   = {Rademacher and {G}aussian complexities: Risk bounds and structural results},
  journal = {Journal of Machine Learning Research},
  volume  = {3},
  pages   = {463--482},
  year    = {2002}
}

@inproceedings{Xu.2025context,
  title = {Does Context Matter? {ContextualJudgeBench} for Evaluating {LLM}-based Judges in Contextual Settings},
  author = {Xu, Austin and Bansal, Srijan and Ming, Yifei and Yavuz, Semih and Joty, Shafiq},
  booktitle = {Proceedings of the 63rd Annual Meeting of the Association for Computational Linguistics (Volume 1: Long Papers)},
  year = {2025},
  pages = {9541--9564}
}
\bibliographystyle{plainnat}

\newpage

\newpage
\appendix
\section{Related Work}
\label{sec:related-work}

\textbf{Agent evaluation from relative feedback.} Relative feedback, particularly pairwise preferences, is widely used to evaluate generative agents. Chatbot Arena and its successor LMArena collect head-to-head judgments on LLM responses~\citep{Chiang.2024}, while Copilot Arena applies the same evaluation to coding agents~\citep{Chi.2025}. A standard approach aggregates these comparisons with a score-based model such as Bradley--Terry to produce a global leaderboard~\citep{Bradley.1952,Chiang.2024}. Subsequent work proposes improved evaluation and ranking procedures by modeling ties and dependence among agents~\citep{Ameli.2025}, estimating nonparametric ranking functionals~\citep{Frauen.2026}, learning prompt-dependent Bradley--Terry scores~\citep{Frick.2025}, or incorporating question difficulty into rankings~\citep{Zhang.2026rankllm}. \citet{Frick.2025} is contextual but retains a Bradley--Terry model at each prompt. Despite these differences, these methods ultimately summarize agents by scalar scores or an ordered leaderboard, which fail to represent intransitive preferences. To accommodate potentially intransitive preferences, \citet{Rivest.2010} use the equilibrium distribution of the pairwise majority-margin game as a randomized single-winner voting rule. Recent agent-evaluation work extends this social-choice perspective to heterogeneous tasks or populations through locally stable top-$k$ sets~\citep{Haghtalab.2026} and maximal lotteries~\citep{Lanctot.2023,Khalaf.2026}. These social-choice methods nevertheless aggregate feedback across contexts to produce a global evaluation, which can obscure context-dependent variation in agents' relative capabilities. Other related game-theoretic methods rank agents from empirical interaction payoffs using Nash averaging or $\alpha$-Rank, including under incomplete payoff information~\citep{Balduzzi.2018,Omidshafiei.2019,Rowland.2019}. To our knowledge, no existing method simultaneously captures context-dependent variation in agents' relative capabilities and represents the collective preference of a heterogeneous population.

\textbf{Social choice and maximal lotteries.} Social choice theory studies how individual preferences can be aggregated into collective decisions, including settings where majority preferences are cyclic and no deterministic Condorcet winner exists~\citep{Condorcet.1785,Brandt.2017}. A central randomized rule is the maximal lottery, introduced by \citet{Kreweras.1965} and developed by \citet{Fishburn.1984a}: it selects a distribution over alternatives that maximizes the worst-case expected majority margin, equivalent to an equilibrium strategy of the induced two-player zero-sum game. Subsequent work has studied maximal lotteries and related tournament-based rules through game-theoretic, axiomatic, computational, and empirical analyses~\citep{Felsenthal.1992,Laffond.1993,Laslier.1997,Dutta.1999,Brandl.2020,Brandl.2022}. Maximal lotteries extend the Condorcet principle and satisfy strong population- and composition-consistency properties, while robustness to cloned alternatives is a standard criterion for voting rules~\citep{Tideman.1987,Brandl.2016}. More broadly, lotteries are considered suitable for decision making when a deterministic choice is undesirable~\citep{Stone.2011}. Recent work brings this perspective to AI alignment by treating heterogeneous human feedback as a preference-aggregation problem and studying the axiomatic consequences of aggregation and regularization~\citep{Halpern.2026,Korkmaz.2026}. Unlike this literature, which aggregates preferences into a global majority profile without modeling context-dependent variation, we learn a context-specific maximal lottery (with convex regularization) from selectively observed offline feedback.

\textbf{Double/debiased machine learning.} Double/debiased machine learning (DML) builds on semiparametric efficiency theory~\citep{Bickel.1998,Kennedy.2024review} to support valid inference on target parameters while using flexible machine learning models for nuisance functions. Its main tools include influence-function analysis in robust estimation~\citep{Hampel.1974}, efficient influence functions from semiparametric theory~\citep{Bickel.1998}, and Neyman-orthogonal moment equations~\citep{Chernozhukov.2018}. \citet{Foster.2023} extend this principle to general two-stage statistical learning: Neyman-orthogonality removes the first-order effect of nuisance estimation error on the target loss, allowing the learner to attain quasi-oracle rates under suitable conditions~\citep{Nie.2021}. Such orthogonal learning paradigm is used in many scenarios, such as heterogeneous treatment-effect estimation~\citep{Kennedy.2023}, survival analysis~\citep{Frauen.2025survival}, off-policy evaluation and learning~\citep{Dudik.2011,Jiang.2016}, and representation learning for causal quantities~\citep{Melnychuk.2026}. We show that the same principle applies to contextual equilibrium learning from stochastic relative feedback by constructing a debiased payoff estimator and an orthogonal gap loss.

\textbf{LLM routing from preferential feedback.} LLM routing studies the selection of the best LLM(s) for a given prompt. A specific line of work is to learn a predictive router from pairwise preferential feedback~\citep{Zhao.2024eagle, Ong.2025, Chiang.2025dueling, Frick.2025}. This can be viewed as a special case of our general formulation: the context is the prompt, the agents are the LLMs. However, these works assume that the preference is transitive: they assign implicit contextual scores and then approximate the preference via a utility function (or link function). This leads to misspecification when preferences cycle. In contrast, our framework does not assume transitivity and instead learns a contextual equilibrium that yields better alignment.

\textbf{Game-theoretic RLHF.} 
Canonical RLHF methods learn a scalar reward model from pairwise comparisons~\citep[e.g.,][]{Christiano.2017,Ouyang.2022}, while direct preference optimization eliminates explicit reward-model fitting but retains the Bradley--Terry preference assumption~\citep{Rafailov.2023}. To accommodate general, potentially intransitive preferences, recent works instead formulate language-model alignment as a two-player zero-sum game and trains a Nash policy through self-play or no-regret optimization~\citep{Munos.2024,Swamy.2024,Rosset.2024,Wu.2025sppo,Zhang.2025inpo}. Although conceptually similar, its goal is to learn the response-generation Nash policy, which is different to our agent evaluation setting with the target being the Nash policy over a set of candidates.

\textbf{Contextual dueling bandits.} This line of work studies how to learn a context-dependent policy when the bandit can only collect pairwise or subsetwise preferential feedback~\citep[e.g.][]{Yue.2012, Dudik.2015, Saha.2021, Bengs.2022,Di.2024,Wu.2024}. \citet{Saha.2021}, \citet{Bengs.2022} and \citet{Di.2024} all implicitly assume stochastic transitivity, which we do not. \citet{Wu.2024} uses Borda score to define a transitivity-robust regret, but still learns a single winner that is neither Condorcet-consistent nor independent of clones (Appendix~\ref{sec:benefits-vnw}). \citet{Dudik.2015} is the first to propose the concept of "von Neumann winner" in the dueling bandit setting, which allows for preference cycles and upon which we build our \method. But their estimator of the preference matrix is not debiased and thus suffers from plug-in bias. Finally, all of these works are targeted for an interactive setting, where the bandit iteratively selects which pair to compare, digests the feedback, and then approximates the optimal policy. In contrast, our work is targeted for an offline setting, where the feedback is collected beforehand. Thus, these works are not applicable to our setting.

\textbf{Offline reinforcement learning in games.} Offline RL in games studies how to learn equilibrium policies from offline multi-agent behaviors, and thus is related to our setting. Existing methods learn such policies directly from observed rewards, with or without KL regularization~\citep{Cui.2022,Zhong.2022,Zhang.2026,Chen.2026}. Closest to our debiasing approach, \citet{Abe.2021} construct doubly robust off-policy estimators of the exploitability of candidate policy profiles and use them for policy selection; we instead debias the contextual PM and learn the entire context-to-equilibrium mapping with an orthogonal gap loss. Finally, RL in games has a fundamentally different observation model: the payoff is sampled directly from the reward model for joint actions, whereas we observe only the relative feedback of the agents within a menu. Hence, these methods are \emph{not} applicable to our setting.

\textbf{Equilibrium learning from stochastic/noisy games.} A related literature studies equilibrium approximation when the payoff function is unavailable analytically but can be accessed through payoff queries or stochastic simulation. This line develops search procedures, probabilistic error bounds, and finite-sample guarantees for equilibria computed from empirical games~\citep[e.g.,][]{Wellman.2006,Jordan.2008,Vorobeychik.2010,Viqueira.2019,Fearnley.2015}. However, these equilibrium learning methods does not apply debiasing w.r.t. the nuisance functions, and thus are subject to plug-in bias. In contrast, our approach uses a debiased payoff matrix (DPM) that is robust to nuisance misspecification, and optimizes a Neyman-orthogonal loss to learn the contextual equilibria.

\newpage
\section{Additional experiments}
\label{sec:additional-experiments}
We complement the main experiments with additional results. In Section~\ref{subsec:aggregation-hides-winner}, we compare pooled and prompt-specific equilibria on RouterBench to illustrate how aggregation can hide contextual winners. In Section~\ref{subsec:lm-arena-experiment}, we examine subgroup-specific equilibria and apply NashEval-debiased to human feedback from LM Arena. We use these analyses to illustrate variation in agent capabilities across contexts and the presence of multiple agents in an equilibrium's support. Because we do not know the true equilibria in LM Arena, we treat comparisons with empirical subgroup lotteries as descriptive evidence rather than an accuracy assessment.

In Section~\ref{subsec:cost-constrained-performance}, we compare cost--performance trade-offs to assess whether contextual equilibrium evaluation remains useful when we account for inference costs. In Section~\ref{subsec:ablations}, we replace the outcome and payoff regression backbones with MLPs to examine whether the main synthetic findings persist under an alternative model instantiation. In Section~\ref{subsec:reg-strengths}, we vary the regularization strength to assess the sensitivity of exploitability and winner-set recovery to this choice.

\subsection{Aggregation hides true contextual winners}
\label{subsec:aggregation-hides-winner}
Existing maximal-lottery-based agent evaluators aggregate preferences across all or a subset of contexts to compute a global maximal lottery~\citep{Lanctot.2023,Khalaf.2026}. Using the RouterBench dataset, we show that this aggregation can hide the true contextual winners and lead to suboptimal evaluation.

We use MMLUPro and LiveCodeBench as examples. Figures~\ref{fig:rb-agg-individual-vNw-mmlu} and~\ref{fig:rb-agg-individual-vNw-code} compare the equilibrium distribution obtained from aggregated preferences with the distributions for individual prompts (i.e., contexts). The aggregated preferences yield a Condorcet winner: Gemini 2.5 Pro for MMLUPro and GPT-5 for LiveCodeBench. Across individual prompts, however, the contextual equilibria are highly heterogeneous and need not contain a Condorcet winner. In some contexts, the aggregated winner is not even in the contextual equilibrium support (see Contexts 2 and 3 for LiveCodeBench). Therefore, \textbf{aggregation hides the true contextual winners and can lead to suboptimal evaluation}. This motivates our work on a contextual evaluator that adaptively selects the best agent or agents for each context.

\begin{figure}[h]
\centering
\includegraphics[width=1.0\textwidth]{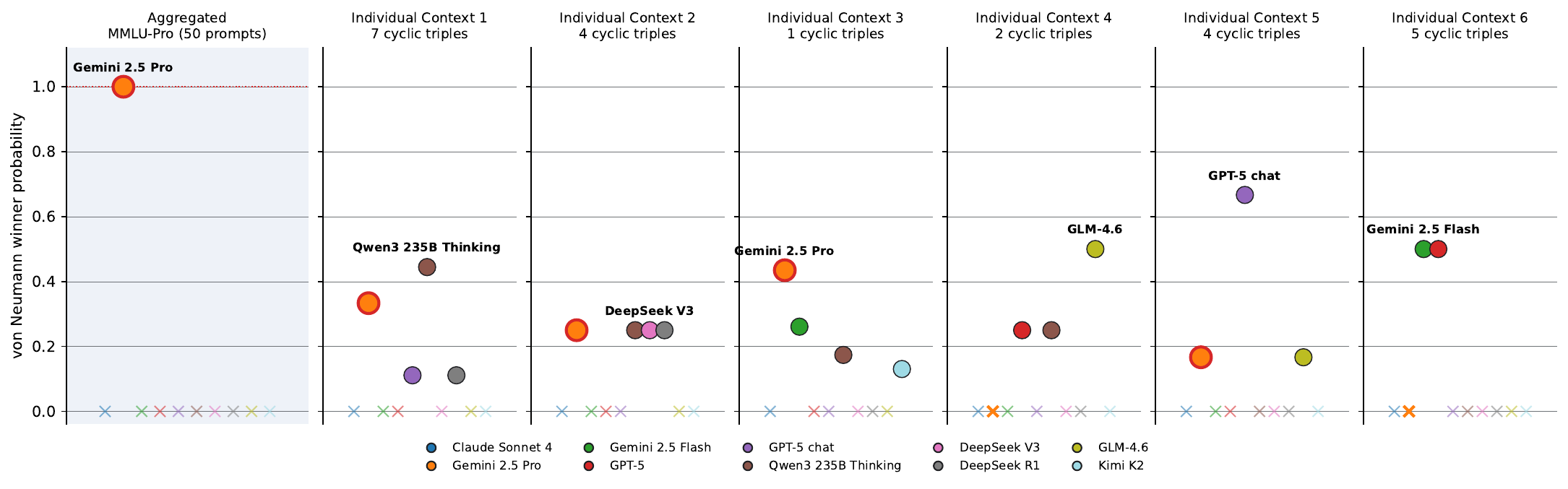}
\caption{Equilibrium distributions for aggregated MMLUPro preferences and individual prompts. Aggregation yields a Condorcet winner, whereas the contextual equilibria are heterogeneous and need not contain NashEval-debiased. \textbf{Therefore, aggregation hides the true contextual winners and can lead to suboptimal evaluation.}}
\label{fig:rb-agg-individual-vNw-mmlu}
\end{figure}

\begin{figure}[h]
\centering
\includegraphics[width=1.0\textwidth]{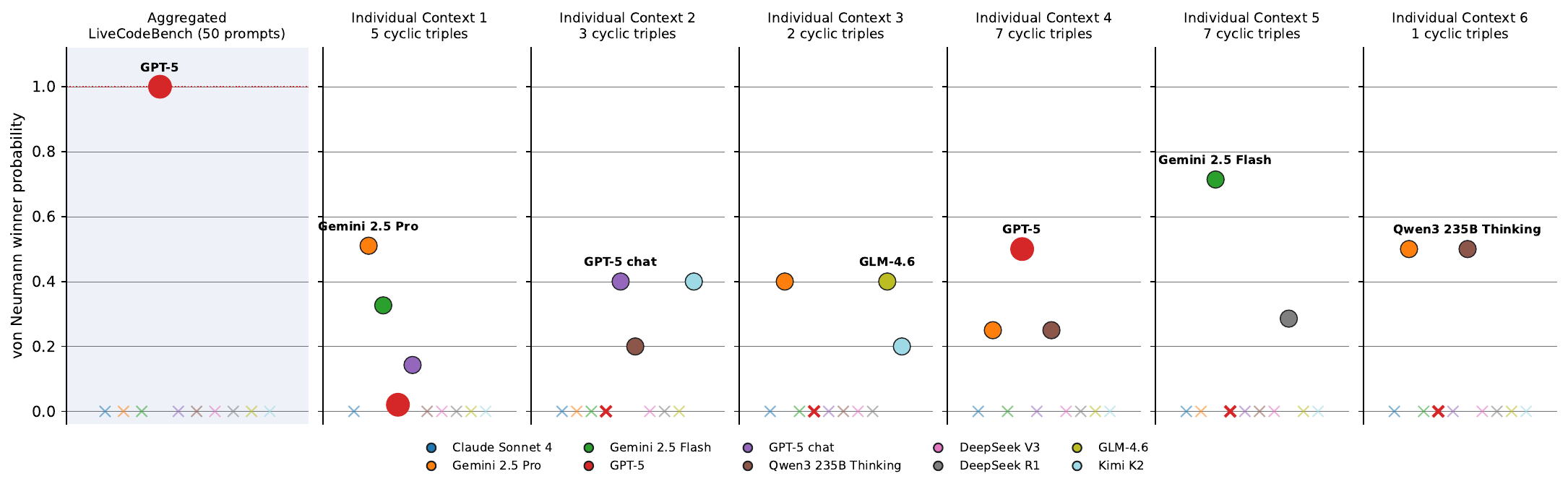}
\caption{Equilibrium distributions for aggregated LiveCodeBench preferences and individual prompts. Aggregation yields a Condorcet winner, whereas the contextual equilibria are heterogeneous and need not contain NashEval-debiased. Therefore, \textbf{aggregation hides the true contextual winners and can lead to suboptimal evaluation}.}
\label{fig:rb-agg-individual-vNw-code}
\end{figure}

\subsection{Experiment on LM Arena}
\label{subsec:lm-arena-experiment}
We demonstrate our \textbf{NashEval-debiased} on the Arena Human Preference 140K data. In LMArena, a user compares responses from two anonymously presented language models and supplies a preference, a tie, or a both-bad judgment~\citep{Chiang.2024}. The full release contains 135,634 human votes involving 53 models and 128,336 distinct first-user prompts (Appendix~\ref{subsec:lm-arena-dataset}). We restrict both training and analysis to the 17,480 comparisons between the 15 models in Table~\ref{tab:selected-llm-models}, which cover 17,296 distinct prompts. 

\subsubsection{Heterogeneity of the agent capability across contexts}
\label{subsec:heterogeneity-lmarena}
In this section, we provide evidence of the heterogeneity of agent capability across contexts in the LM Arena dataset. We stratify the LMArena data by discrete context attributes (see Table~\ref{tab:arena-language-distribution},\ref{tab:arena-attribute-distribution}) and construct majority-margin matrices using win-minus-loss counts for each stratification, with unobserved model pairs replaced by ties. We then compute an equilibrium distribution for each subgroup having at least 90\% pair coverage. When multiple equilibria exist, we choose the one with minimal $l_2$ norm.

\textbf{Results.} We quantify heterogeneity using pairwise total-variation distances between the equilibrium distributions. For each single- or two-attribute combination, we report the mean and maximum total-variation distance in Table~\ref{tab:arena-vnw-heterogeneity}. We further visualize the subgroup-specific equilibria in Figures~\ref{fig:arena-maximal-lotteries} and~\ref{fig:arena-maximal-lotteries-all-attributes}. These equilibria are highly heterogeneous across subgroups. For example, Gemini 2.5 Pro is not in the equilibrium support for either $(\text{code},\text{math})=(\text{No},\text{yes})$ or $(\text{yes},\text{no})$. Therefore, a single global equilibrium cannot represent how agent performance varies across contexts.

\begin{figure}[htbp]
    \centering
    \begin{minipage}[t]{0.32\textwidth}
        \centering
        \includegraphics[width=\linewidth]{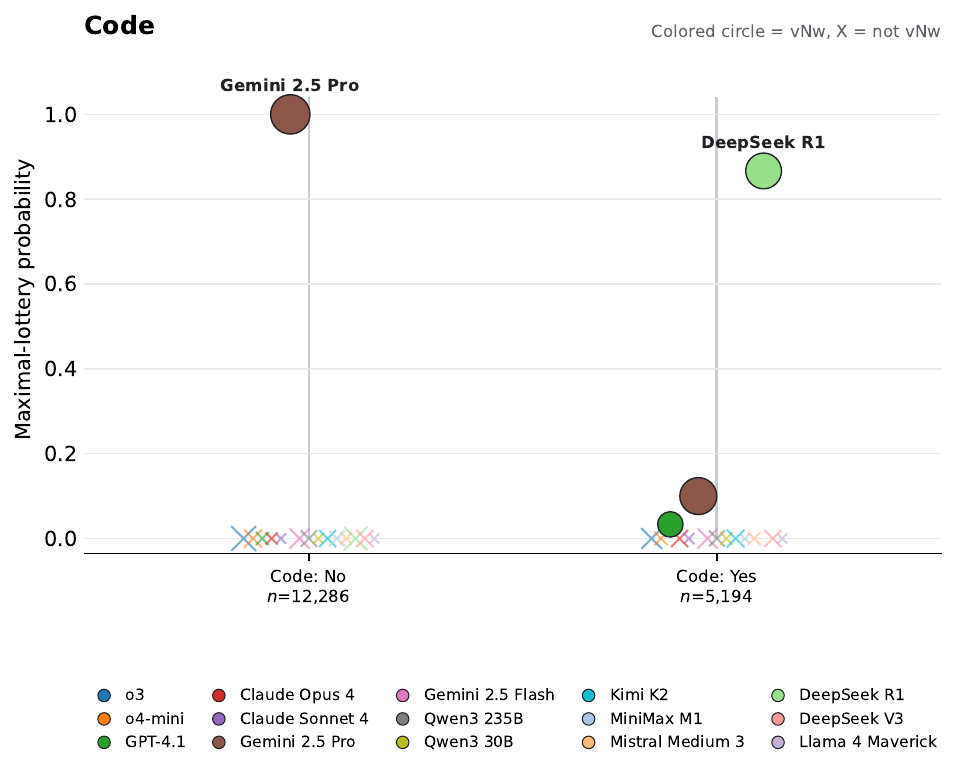}
        \par\small (a) Code
    \end{minipage}
    \hfill
    \begin{minipage}[t]{0.32\textwidth}
        \centering
        \includegraphics[width=\linewidth]{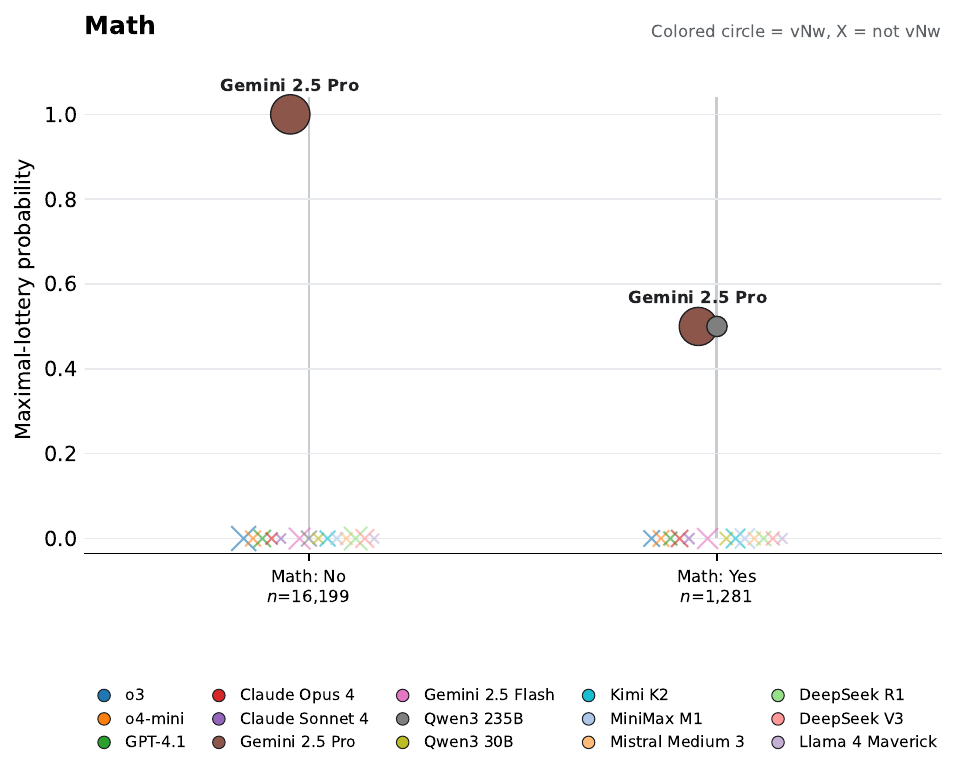}
        \par\small (b) Mathematics
    \end{minipage}
    \hfill
    \begin{minipage}[t]{0.32\textwidth}
        \centering
        \includegraphics[width=\linewidth]{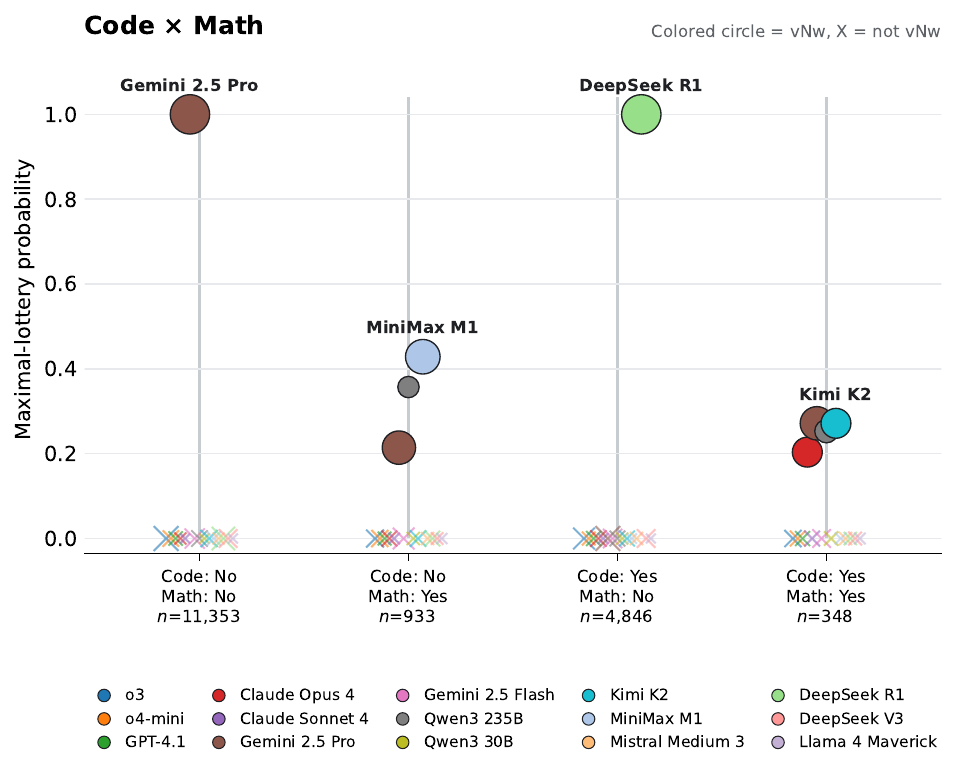}
        \par\small (c) Code and mathematics
    \end{minipage}
    \caption{Subgroup-specific equilibrium distributions across code and mathematics contexts. Each plotted circle is an agent in the equilibrium support, and the y-axis shows its equilibrium probability. Further stratification can introduce new agents into the support (e.g., MiniMax M1), indicating that \textbf{a global equilibrium does not reflect context-specific agent capabilities}.}
    \label{fig:arena-maximal-lotteries}
\end{figure}

\begin{figure}[htbp]
    \centering
    \includegraphics[width=\textwidth]{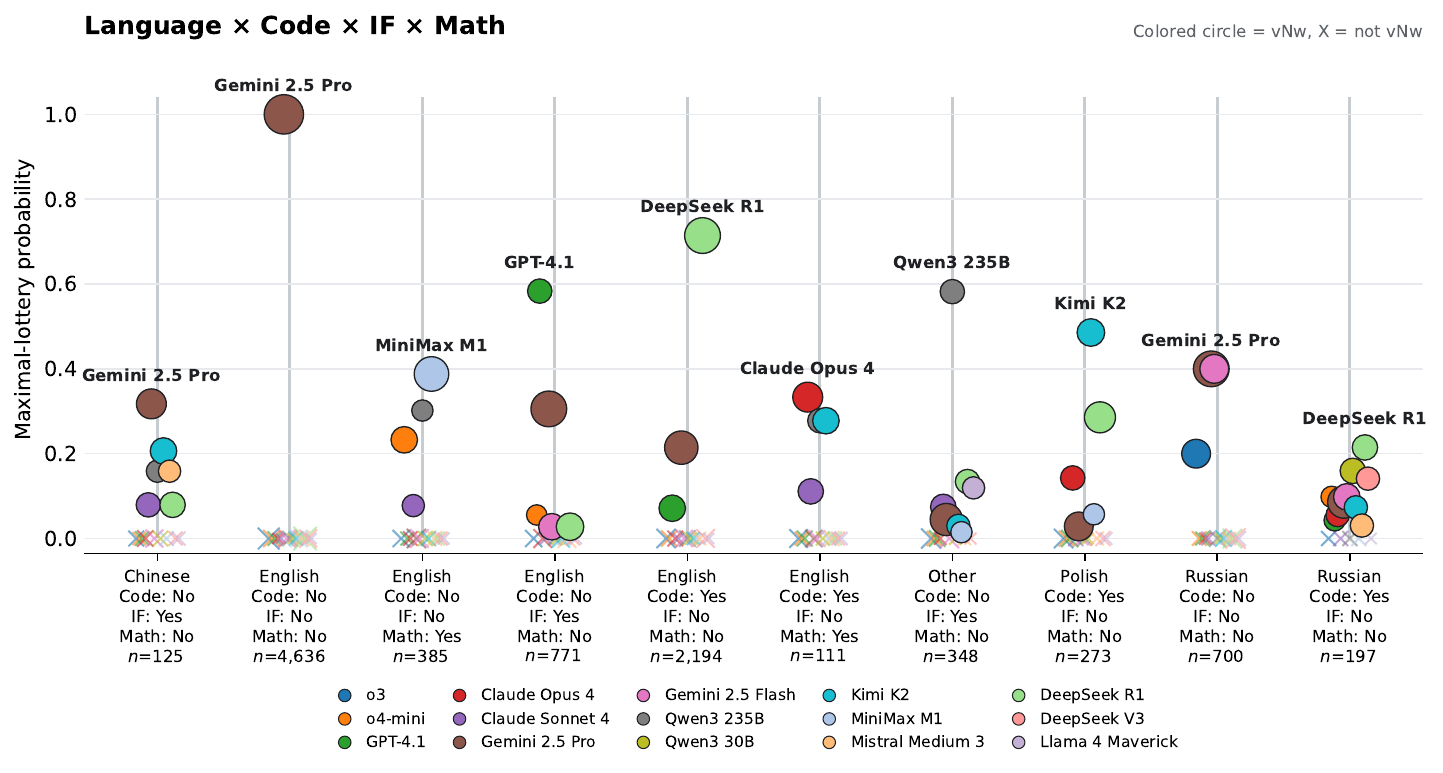}
    \caption{Subgroup-specific equilibrium distributions across language, code, instruction-following, and mathematics contexts. Each plotted circle is an agent in the equilibrium support, and the y-axis shows its equilibrium probability. Agent capabilities vary substantially across contexts.}
    \label{fig:arena-maximal-lotteries-all-attributes}
\end{figure}

\begin{table}[t]
\centering
\caption{Heterogeneity of subgroup-specific equilibrium distributions in LMArena illustrated by total variation within each attribution composition. Only subgroups with at least 90\% pair coverage are included. \textbf{Subgroups vary substantially in their winning agents within each stratification.}}
\label{tab:arena-vnw-heterogeneity}
\small
\begin{tabular}{@{}lrrrr@{}}
\toprule
\textbf{Stratification}
& \textbf{Subgroups}
& \textbf{Mean TV}
& \textbf{Max TV}
& \textbf{Distinct winners} \\
\midrule
\multicolumn{5}{l}{\emph{Single attribute}} \\
Language                       & 5 & 0.612 & 0.866 & 3 \\
Code                           & 2 & 0.900 & 0.900 & 2 \\
Instruction following          & 2 & 0.000 & 0.000 & 1 \\
Mathematics                    & 2 & 0.500 & 0.500 & 1 \\
\midrule
\multicolumn{5}{l}{\emph{Two attributes}} \\
Language $\times$ code
                               & 6 & 0.460 & 0.600 & 1 \\
Language $\times$ instruction following
                               & 6 & 0.736 & 1.000 & 4 \\
Language $\times$ mathematics
                               & 5 & 0.734 & 1.000 & 4 \\
Code $\times$ instruction following
                               & 4 & 0.438 & 0.875 & 2 \\
Code $\times$ mathematics
                               & 3 & 1.000 & 1.000 & 3 \\
Instruction following $\times$ mathematics
                               & 3 & 0.333 & 0.500 & 1 \\
\bottomrule
\end{tabular}
\end{table}

For this experiment, we define the context using only the recorded language, code, mathematics, and instruction-following labels. We display a selection of subgroups that have at least 100 comparisons and observed feedback for at least 94 of the 105 selected-model pairs. The true context-specific equilibrium is unavailable in LM Arena. We therefore use the empirical maximal lottery within each subgroup as a reference, i.e. disagreement or agreement do not necessarily imply the accuracy of the method. 

\begin{figure}
    \centering
    \includegraphics[width=\linewidth]{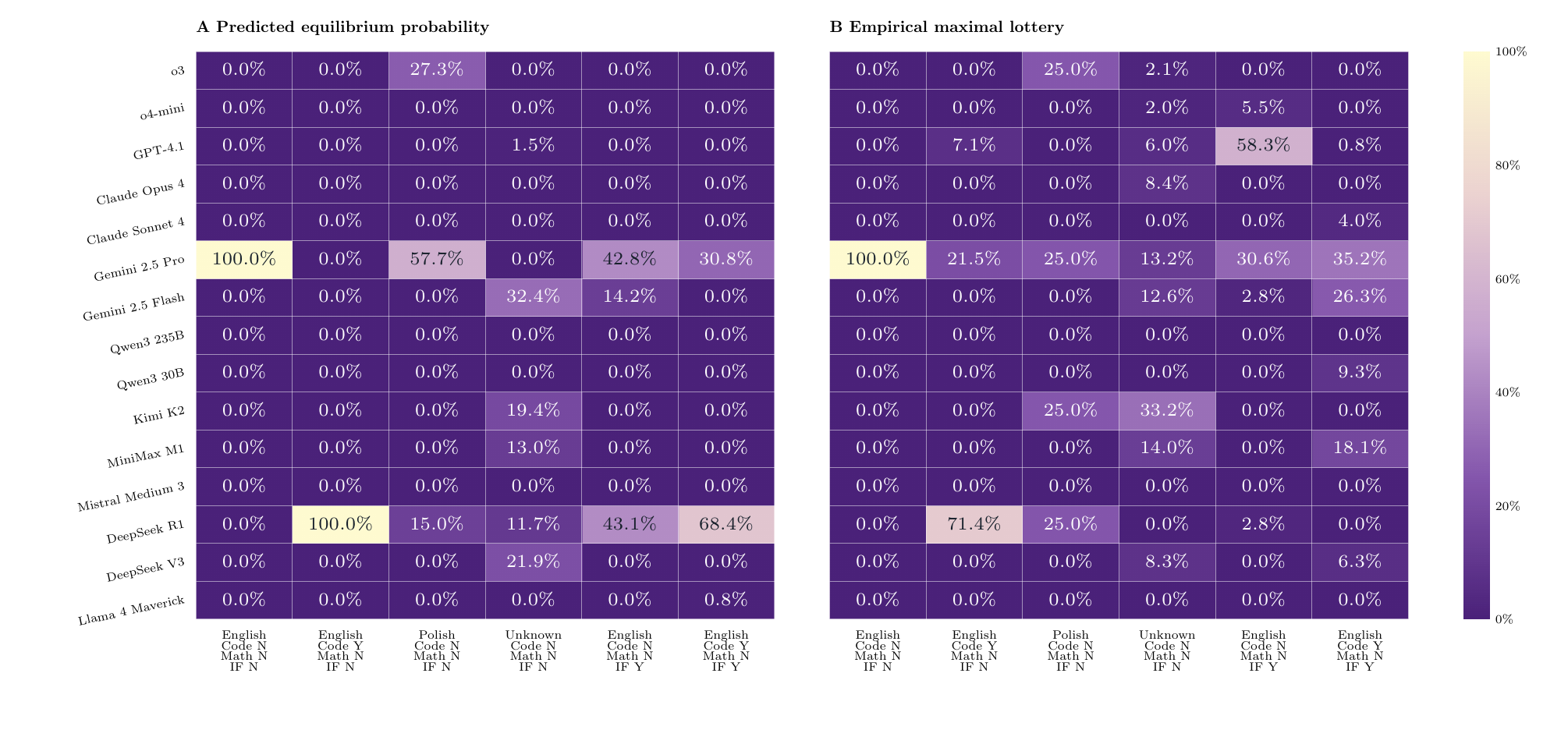}
    \caption{Attribute-specific equilibria for six LM Arena subgroups defined by language, code, mathematics, and instruction-following labels. Panel A shows the NashEval-debiased equilibrium probability of each selected model. Panel B shows the regularized maximal lottery from pairwise win-minus-loss counts pooled within the same subgroup. Both panels use $\rho=0.01$ and report percentages. Note: this is different from the heatmap Figure~\ref{fig:rb-heatmap}, which reports the winner frequencies of the agents.}
    \label{fig:arena-heatmap}
\end{figure}

\textbf{Results.} Figure~\ref{fig:arena-heatmap} shows that the evaluation of our \textbf{NashEval-debiased} is largely consistent with the empirical maximal lottery, which indicates the effectiveness of our \method. For English prompts without code, mathematics, or instruction following, both equilibria assign all probability to Gemini 2.5 Pro. For English coding prompts without mathematics or instruction following, the predicted equilibrium assigns all probability to DeepSeek R1, while the empirical lottery assigns it 71.4\%. The predicted equilibrium changes across these contexts: DeepSeek R1 receives no probability in the first group but is the sole supported model in the second. For English instruction-following prompts without code or mathematics, the predicted equilibrium assigns 43.1\% to DeepSeek R1 and 42.8\% to Gemini 2.5 Pro, whereas the empirical lottery assigns 58.3\% to GPT-4.1. Claude Opus 4 receives zero predicted probability in all six displayed groups, despite having the highest listed price.

\subsection{Cost-constrained performance.}
\label{subsec:cost-constrained-performance}
\textbf{Experimental setup.} We reuse the RouterBench payoff models from RQ2 across five seeds and 500 test prompts. For NashEval-debiased, NashEval-plug-in, and Maximal lottery, we incorporate the prices in Table~\ref{tab:rb-models} through
\begin{equation}
 \Omega_\beta(\pi)=\frac{\rho}{2}\|\pi\|_2^2+\beta\widetilde c^\top\pi,
 \qquad \rho=0.01,\qquad
 \widetilde c_j=\frac{c_j-\min_k c_k}{\max_k c_k-\min_k c_k}.
 \label{eq:routerbench-cost-regularizer}
\end{equation}
Here, $c_j$ is the cost of $1{,}000$ input and $1{,}000$ output tokens. We vary $\beta$ and compare against \textbf{BT-cost}, which maximizes its predicted pairwise score under a per-prompt expected-cost budget following Theorem~1 of \citet{Frick.2025} (Appendix~\ref{para:bt-cost-constrained}). We compare all methods by realized mean expected cost. For the equilibrium methods, we use Eq.~(\ref{eq:routerbench-cost-regularizer}) with the price weights
\begin{align}
 \beta\in\{0.025t:t=0,\ldots,12\}
 \cup\{0.35,0.4,0.5,0.75,1,1.5,2,3,5,10,20\}.
\end{align}

\textbf{Metrics.} We report unregularized exploitability against an unrestricted opponent and average strict win rate against the ten single-agent opponents with equal weight. We give ties, including self-comparisons, zero credit and retain them in the denominator. 

\textbf{Results.} Figures~\ref{fig:cost-constrained-pareto} and~\ref{fig:average-win-rate-cost-frontier} show that NashEval-debiased offers the strongest empirical trade-off through much of the intermediate-cost region, achieving the pareto frontier among intermediate cost ranges for both metrics. This demonstrates that our method is able to produce the cost-constrained contextual evaluations that reflect the collective preference of the population.

\begin{figure}
\centering
\includegraphics[width=0.7\textwidth]{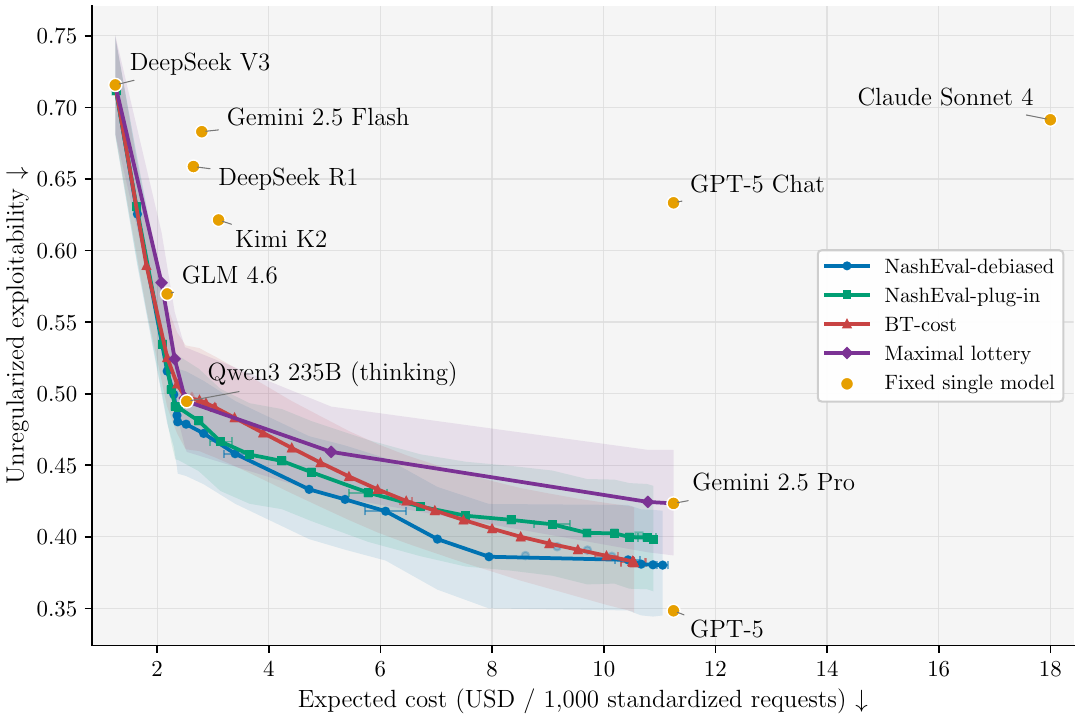}
\caption{Cost--exploitability trade-offs on RouterBench across five seeds and 500 test prompts. Our method largely achieves the empirical Pareto frontier over costs of approximately USD $2$--$11$ per $1{,}000$ standardized requests among the evaluated policies. Therefore, under the same cost budge, our method produces contextual evaluations that better reflect the collective preference of the population.}
\label{fig:cost-constrained-pareto}
\end{figure}

\begin{figure}
\centering
\includegraphics[width=0.7\textwidth]{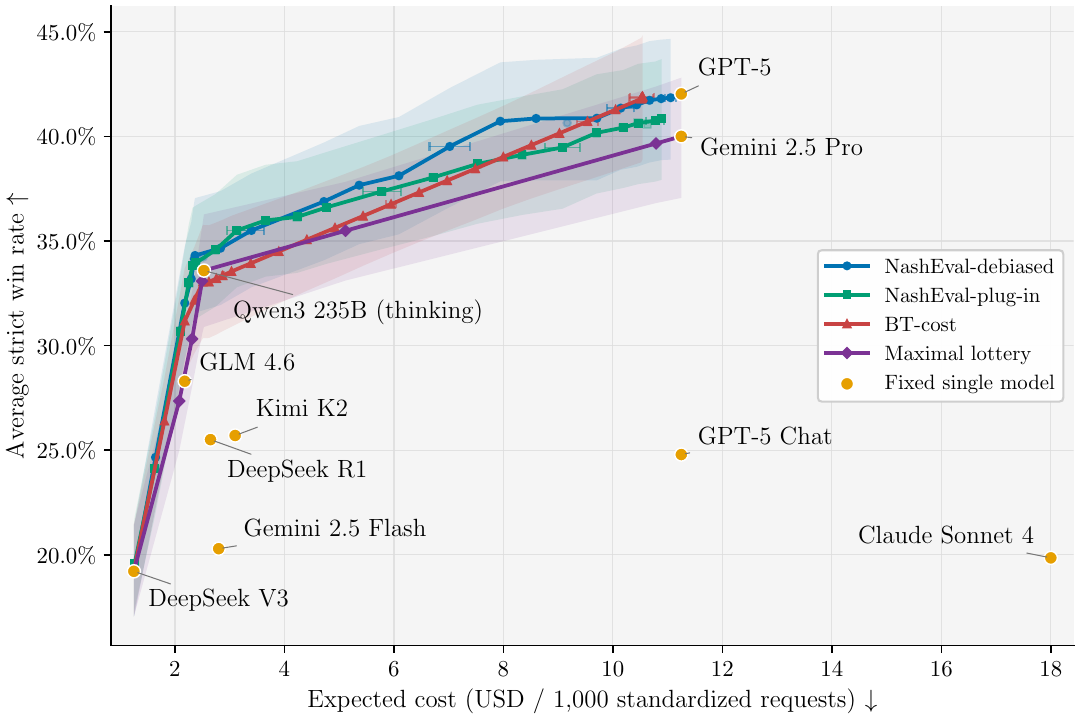}
\caption{Cost--average-win-rate trade-offs on RouterBench. Our method largely achieves the empirical Pareto frontier at intermediate cost range $[2, 11]$. Therefore, under the same cost budge, our method produces contextual evaluations that better reflect the collective preference of the population.}
\label{fig:average-win-rate-cost-frontier}
\end{figure}

\subsection{Ablations on backbone.}
\label{subsec:ablations}

\textbf{Experimental setup.} We conduct ablation study on the backbone of the outcome nuisance estimator and the second-stage payoff regressor on the simulated dataset from RQ1. We jointly replace LightGBM in these two components with multilayer perceptrons (MLPs). We retain the LightGBM propensity estimator and compute equilibria with the same numerical solver as in RQ1, without introducing a learned policy predictor. For NashEval-oracle, we use the same MLP payoff regressor as for NashEval-debiased, but construct the pseudo-outcomes with the true nuisances. NashEval-plug-in uses the fold-averaged MLP outcome predictions without a separate payoff regression. The remained is unchanged compared with RQ1. We evaluate unregularized exploitability and support $F_1$, using the same support threshold of $10^{-3}$, and summarize each metric by its mean and standard deviation across 5 random seeds.seeds.

We use two hidden layers with widths $(64,64)$ for each outcome nuisance and $(128,128)$ for each pairwise payoff regressor, with ReLU activations. We train both networks with Adam at learning rate $3\times10^{-3}$ and an $L_2$ weight penalty of $10^{-4}$, using batch sizes of $512$ and $1{,}024$, respectively. We allow at most $300$ epochs and apply early stopping on a held-out subset of the corresponding training data with patience 12. We fix these MLP settings before evaluating the test sets. We regress the unclipped debiased pseudo-outcomes, clip predicted margins to $[-1,1]$, and impose antisymmetry before solving for the equilibrium.

We report the mean and std of the (unregularized) exploitability and $F_1$ score over five random seeds in Figure~\ref{fig:ablations}. For $N_v \geq 3$, our \method consistently outperforms other non-oracle baselines. For $N_v=1$, preferences are transitive and the BT-model-based baselines are slightly better. The result is consistent with the main findings in RQ1 in Section~\ref{sec:experiments}. This proves the effectiveness of our \method regardless of the backbone instantiations, and that the higher performance is attributed to the method itself.

\begin{figure}
\centering
\includegraphics[width=0.7\textwidth]{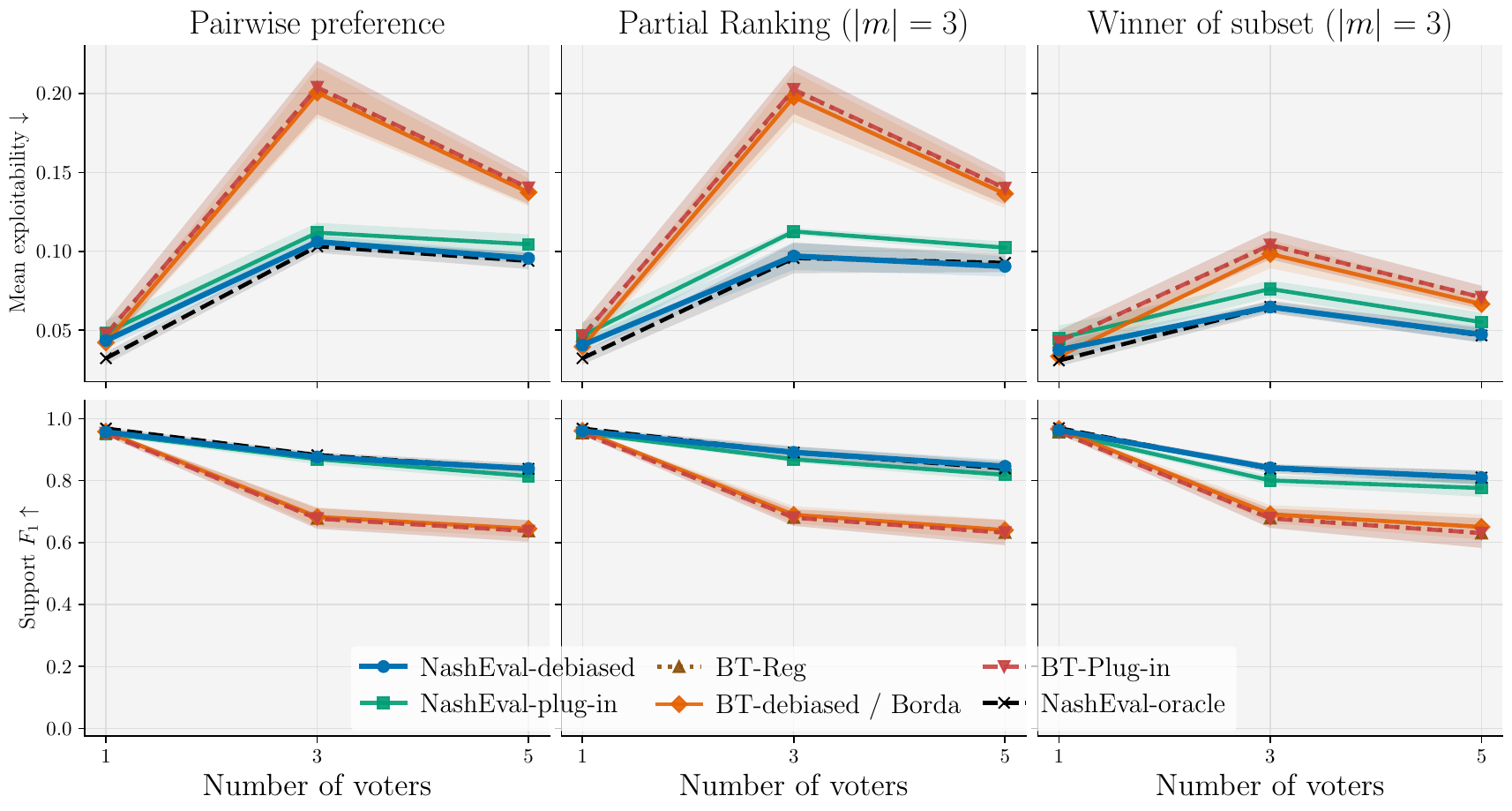}
\caption{Ablation study on the outcome nuisance and payoff backbone on the simulated dataset. We note that our \method outperforms the non-oracle baselines on $N_v\geq 3$ and achieves comparable performance with \textbf{BT}-based baselines on $N_v=1$. This result demonstrate the effectiveness of the learning framework itself.}
\label{fig:ablations}
\end{figure}

\subsection{Performance over varying regularization strengths.}
\label{subsec:reg-strengths}
We investigate the effect of the regularization strength $\rho$ on the performance of our method on the simulated dataset. The setting is similar to RQ1, but we fix the number of voters to $3$. Then we vary $\rho$ within the range $\{0.0, 0.001, 0.01, 0.1, 1.0\}$ and compare against other contextual evaluation baselines. We report the mean and std of the (unregularized) exploitability and $F_1$ score over five random seeds in Figure~\ref{fig:rho-comparison}. We observe no significant change in both metrics when $\rho\leq 0.01$, with performance slightly decreasing for $\rho$ beyond that range. We also observe a slight improvement in $F_1$ score when $\rho$ is increased from $0.1$ to $1.0$. This is likely due to the sparsity imposed by the $L_2$ regularizer, which reduces the size of the support and thus the chance of false positives. In summary, our \method is robust to the choice of $\rho$.

\begin{figure}
\centering
\includegraphics[width=0.8\linewidth]{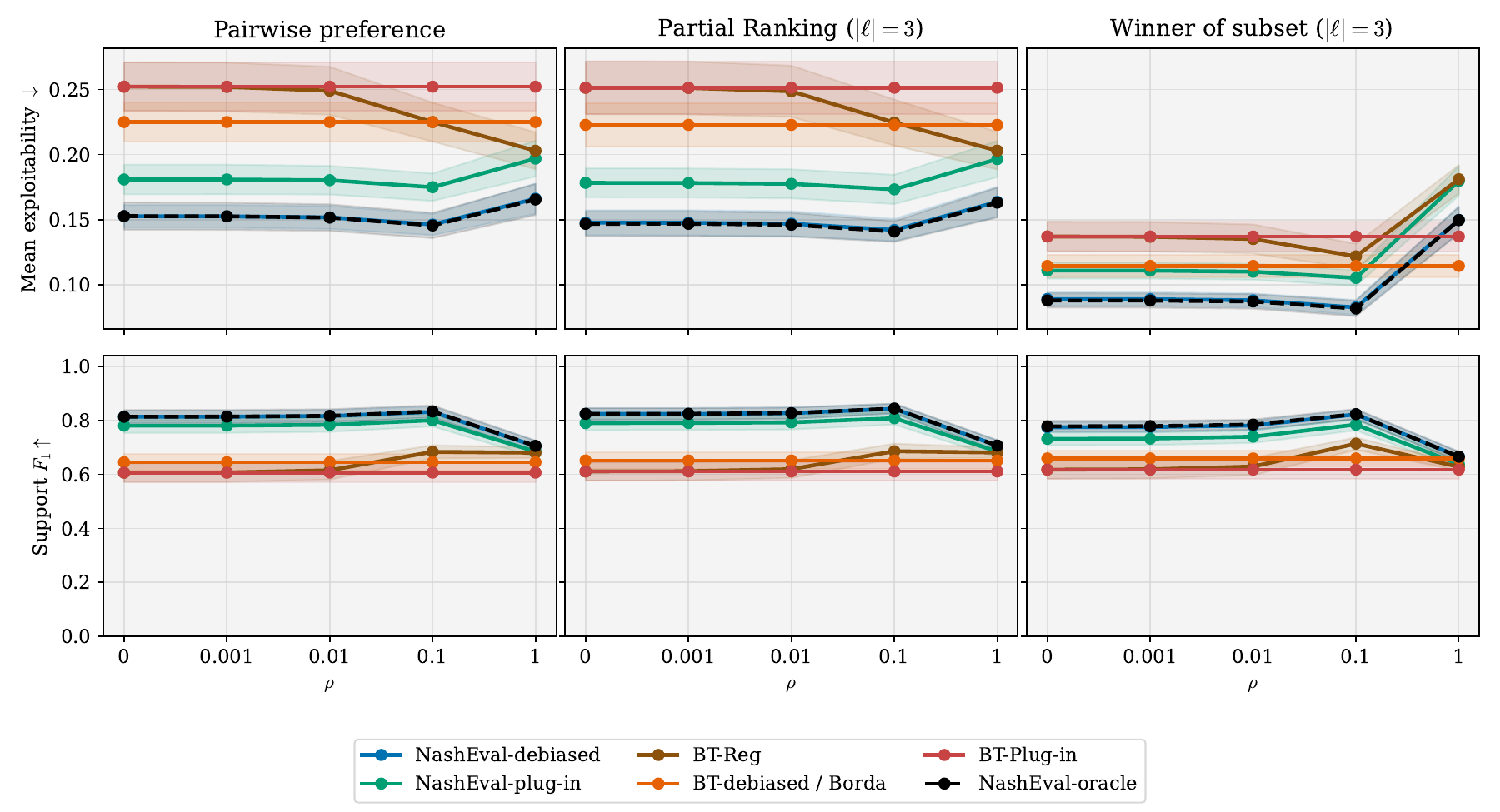}
\caption{Mean $\pm$ std Mean exploitability and $F_1$ score on simulated dataset for $\#$voters$=3$ over 5 random seeds, plotted against varying regularization strengths $\rho$.}
\label{fig:rho-comparison}
\end{figure}

\subsection{Comparison with the numerical equilibrium solver}
\label{subsec:numerical-equilibrium}
A na\'ive approach is to only estimate the contextual PM, and then solve the equilibrium every time a new context is encountered. We compare our \textbf{NashEval-debiased} and \textbf{NashEval-plug-in} with the numerical equilibrium solver using both the debiased and plug-in PM estimators. We adopt the same settings as in RQ1, where we use unregularized exploitability and support $F_1$ scores to evaluate the quality of the learned equilibrium while recording the mean runtime during inference. We use a custom solver that enumerates over all possible supports ($2^K - 1$) to search for equilibria that satisfy the KKT conditions.

\begin{figure}
\centering
\includegraphics[width=0.8\linewidth]{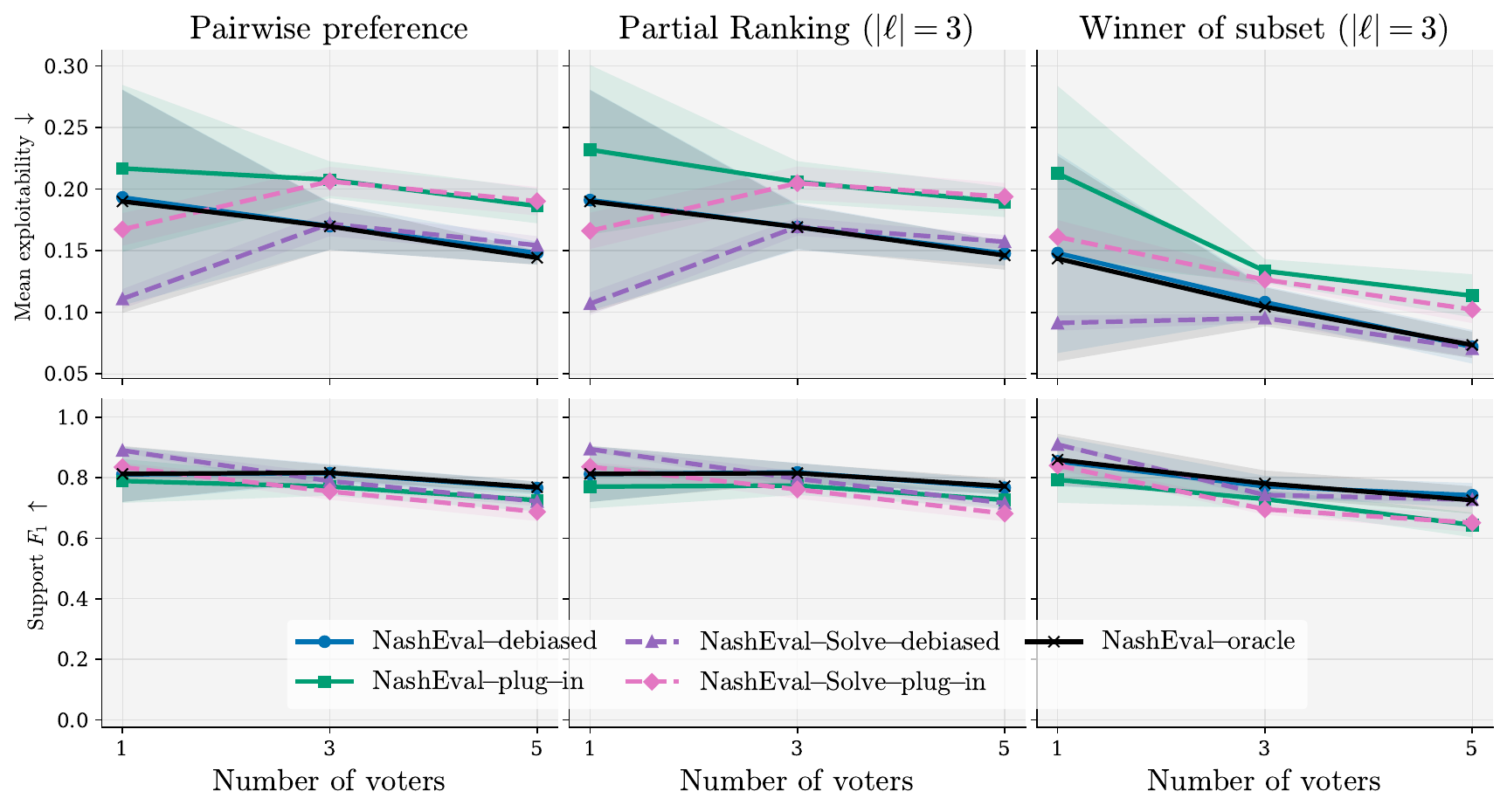}
\caption{Performance comparison of our \method with the numerical solver. Mean $\pm$ std Mean exploitability and $F_1$ score on simulated dataset for $\#$voters$=1,3,5$ over 5 random seeds. }
\label{fig:rq1-solver-comparison}
\end{figure}

\textbf{Results.} Results on Figure~\ref{fig:rq1-solver-comparison} show that our \method performs comparably to the numerical solver when number of voters is larger than 1. Runtime comparison is given in Figure~\ref{fig:inference-runtime}. We see that our \method is much more efficient during inference: The runtime of \textbf{NashEval-debiased} remain stable,  while the inference runtime of the numerical solver increases exponentially when $K$ increases.

\begin{figure}
    \centering
    \includegraphics[width=0.5\linewidth]{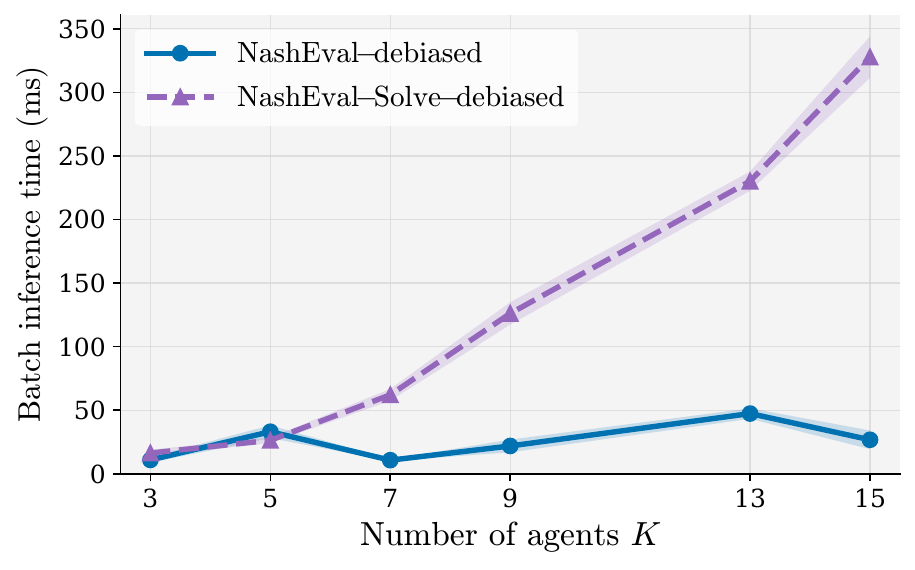}
    \caption{Inference runtime comparison between our \textbf{NashEval-debiased} and the numerical equilibrium solver. Our \method is much more efficient during inference time.}
    \label{fig:inference-runtime}
\end{figure}

\newpage
\section{Preliminaries of Game Theory}
\label{app:game-theory}

In this appendix, we introduce the game-theoretic concepts needed to understand our solution concept and methods. In Section~\ref{app:vi}, we introduce variational inequalities and monotone operators, which provide the main tools for characterizing and solving equilibria in two-player zero-sum games.
In Section~\ref{app:selfplay}, we introduce the self-play operator and show that its variational-inequality solutions coincide with the symmetric equilibria of our regularized game.
In Section~\ref{app:exist-unique}, we prove existence and uniqueness of the contextual regularized equilibrium.
In Section~\ref{subsec:gap-function}, we introduce the gap function and establish its consistency, equilibrium-error control, and differentiability.
In Section~\ref{app:payoff-stability}, we prove that the equilibrium is stable under perturbations of the payoff matrix.

\paragraph{Notation and standing conventions.} We conduct the game-theoretic analysis at a fixed context $x$. Unless stated otherwise, we suppress the context from the payoff notation and write $A=A(x)$. The regularizer $\Om$ does not depend on $x$. Recall that $A=-A^\top$ and that $\Om$ is continuously differentiable on a neighborhood of $\DeltaK$ and $\kappa$-strongly convex on $\DeltaK$ with respect to $\|\cdot\|$, where $\kappa>0$. Thus, for every $p,q\in\DeltaK$,
\begin{equation}
\Om(q)\ge\Om(p)+\langle\nabla\Om(p),q-p\rangle+\frac{\kappa}{2}\|q-p\|^2.
\label{eq:strong-convexity}
\end{equation}

\paragraph{Game-theoretic setting.} Game theory studies strategic interactions in which each player's payoff depends on the strategies chosen by all players. A game specifies the players, their feasible strategy sets, their payoff functions, and a solution concept that describes stable joint behavior. In this work, we study a contextual, regularized, two-player zero-sum game over the shared mixed-strategy simplex $\DeltaK$. At each context $x$, the row player receives $F_\Om(\pi,q;x)$ and the column player receives $-F_\Om(\pi,q;x)$, where the skew-symmetric matrix $A(x)$ encodes collective pairwise preferences and the context-independent regularizer $\Om$ induces strong convexity. We seek the symmetric saddle point of this game, which coincides with its Nash equilibrium and defines the contextual von Neumann winner $\pi^\star(x)$.

\subsection{Variational inequalities and monotone operators}
\label{app:vi}

In this subsection, we introduce variational inequalities and strong monotonicity, which provide the equilibrium characterization and uniqueness argument used throughout this appendix.
A \emph{variational inequality} is the problem of finding, for an operator $V:\DeltaK\to\R^K$ and the convex compact set $\DeltaK$, a point $\pi^\star\in\DeltaK$ such that
\begin{equation}
\big\langle V(\pi^\star),\ \pi-\pi^\star\big\rangle\ \ge\ 0
\qquad\text{for all }\pi\in\DeltaK.
\label{eq:vi}
\end{equation}
We write $\mathrm{VI}(V,\DeltaK)$ for this problem. Geometrically, Eq.~(\ref{eq:vi}) says that $-V(\pi^\star)$ makes a non-acute angle with every feasible direction at $\pi^\star$, so no feasible direction decreases the objective represented by $V$. See \citet{Facchinei.2003}.

The operator property that governs uniqueness is monotonicity. An operator $V$ is \emph{monotone} on $\DeltaK$ if
\begin{align*}
\langle V(p)-V(q),\ p-q\rangle\ \ge\ 0
\qquad\text{for all }p,q\in\DeltaK,
\end{align*}
and \emph{$\kappa$-strongly monotone} with respect to $\|\cdot\|$ if the stronger inequality
\begin{equation}
\langle V(p)-V(q),\ p-q\rangle\ \ge\ \kappa\,\|p-q\|^2
\label{eq:strong-monotone}
\end{equation}
holds for the same quantifiers. We use two standard facts \citep{Facchinei.2003}: the gradient of a differentiable $\kappa$-strongly convex function is $\kappa$-strongly monotone, and a strongly monotone VI has at most one solution. In our game, the linear part $\pi\mapsto-A\pi$ contributes no curvature because $\langle-Ad,d\rangle=-d^\top Ad=0$ for $A=-A^\top$.

\subsection{The self-play operator}
\label{app:selfplay}

In this subsection, we define the self-play operator, characterize symmetric equilibria as solutions of its variational inequality, and establish the strong monotonicity needed to prove uniqueness.
For the regularized game $F_\Om(\pi,q)=\pi^\top Aq-[\Om(\pi)-\Om(q)]$ of Definition~\ref{def:contextual-game}, define the \emph{self-play operator}
\begin{equation}
V_\Om(\pi)\ =\ -A\pi+\nabla\Om(\pi).
\label{eq:selfplay-operator}
\end{equation}
This is the standard self-play operator for a regularized zero-sum game \citep{Sokota.2022}. The next lemma gives its exact saddle-point characterization.

\begin{lemma}[Self-play variational inequality]
\label{lem:selfplay}
A pair $(\pi^\star,\pi^\star)$ with $\pi^\star\in\DeltaK$ is a symmetric saddle point (i.e. a Nash equilibrium) of $F_\Om$ if and only if
\begin{equation}
\big\langle -A\pi^\star+\nabla\Om(\pi^\star),\ \pi-\pi^\star\big\rangle\ \ge\ 0
\qquad\text{for all }\pi\in\DeltaK,
\label{eq:equilibrium-vi}
\end{equation}
Moreover $V_\Om$ is $\kappa$-strongly monotone with respect to $\|\cdot\|$.
\end{lemma}

\begin{proof}
\emph{Necessity.} Suppose $(\pi^\star,\pi^\star)$ is a saddle point. We fix the column strategy at $\pi^\star$ and discard the term $\Om(\pi^\star)$, which does not depend on $\pi$. The row player then maximizes
\begin{equation}
f_{\pi^\star}(\pi)\ =\ \pi^\top A\pi^\star-\Om(\pi)
\label{eq:rowobj}
\end{equation}
over $\DeltaK$. The function $f_{\pi^\star}$ is concave because it equals a linear function minus a convex function. The saddle-point property implies that $\pi^\star$ maximizes it over the convex set $\DeltaK$. The first-order optimality condition for this differentiable concave problem is exactly Eq.~(\ref{eq:equilibrium-vi}).

\emph{Sufficiency.} Conversely, suppose Eq.~(\ref{eq:equilibrium-vi}) holds. By convexity, $\Om(\pi)\ge\Om(\pi^\star)+\langle\nabla\Om(\pi^\star),\pi-\pi^\star\rangle$ for all $\pi\in\DeltaK$. Hence, for every $\pi$,
\begin{align*}
f_{\pi^\star}(\pi)-f_{\pi^\star}(\pi^\star)
=\langle A\pi^\star,\pi-\pi^\star\rangle-\big[\Om(\pi)-\Om(\pi^\star)\big]
\le\big\langle A\pi^\star-\nabla\Om(\pi^\star),\ \pi-\pi^\star\big\rangle
\le0,
\end{align*}
where the final inequality follows from Eq.~(\ref{eq:equilibrium-vi}). We add $\Om(\pi^\star)$ to both sides and obtain $F_\Om(\pi,\pi^\star)\le F_\Om(\pi^\star,\pi^\star)$ for all $\pi\in\DeltaK$. Therefore, $\pi^\star$ is a best response for the row player. For the column player, skew-symmetry gives $F_\Om(p,q)=-F_\Om(q,p)$, so
\begin{align*}
F_\Om(\pi^\star,q)=-F_\Om(q,\pi^\star)\ \ge\ -F_\Om(\pi^\star,\pi^\star)=F_\Om(\pi^\star,\pi^\star),
\end{align*}
Here we use $F_\Om(\pi^\star,\pi^\star)=(\pi^\star)^\top A\pi^\star=0$, which follows from $d^\top Ad=0$ for skew-symmetric $A$. Thus, $\pi^\star$ is also a best response for the column player, and $(\pi^\star,\pi^\star)$ is a saddle point with value zero. This direction uses skewness twice and fails without it.

\emph{Strong monotonicity.} Let $p,q\in\DeltaK$ and put $d=p-q$. Then
\begin{equation}
\langle V_\Om(p)-V_\Om(q),d\rangle
=\underbrace{-\,d^\top Ad}_{=\,0}+\langle\nabla\Om(p)-\nabla\Om(q),d\rangle
=\langle\nabla\Om(p)-\nabla\Om(q),d\rangle .
\label{eq:mono-split}
\end{equation}
The first term vanishes because $d^\top Ad=\tfrac12 d^\top(A+A^\top)d=0$ for skew-symmetric $A$. For the second, we write Eq.~(\ref{eq:strong-convexity}) twice, once at $(p,q)$ and once at $(q,p)$:
\begin{align*}
\Om(q)\ \ge\ \Om(p)+\langle\nabla\Om(p),q-p\rangle+\tfrac{\kappa}{2}\|d\|^2,
\qquad
\Om(p)\ \ge\ \Om(q)+\langle\nabla\Om(q),p-q\rangle+\tfrac{\kappa}{2}\|d\|^2 .
\end{align*}
We add the two inequalities and cancel $\Om(p)+\Om(q)$ to obtain $\langle\nabla\Om(p)-\nabla\Om(q),d\rangle\ge\kappa\|d\|^2$. We substitute this result into Eq.~(\ref{eq:mono-split}) and obtain Eq.~(\ref{eq:strong-monotone}) for $V_\Om$.
\end{proof}

\subsection{Existence and uniqueness}
\label{app:exist-unique}

In this subsection, we combine the minimax theorem with the self-play characterization to prove that the contextual regularized equilibrium exists, is symmetric, and is unique, which makes the target mapping $x\mapsto\pi^\star(x)$ well defined.
\begin{proposition}[Existence and uniqueness of the regularized equilibrium]
\label{prop:exist}
Let $A=-A^\top$ and let $\Om$ be continuously differentiable on a neighborhood of $\DeltaK$ and $\kappa$-strongly convex on $\DeltaK$, with $\kappa>0$. Then the game $F_\Om$ of Definition~\ref{def:contextual-game} admits a saddle point, every saddle point is symmetric of the form $(\pi^\star,\pi^\star)$, the game value is zero, and $\pi^\star$ is \emph{unique}. Consequently $\pi^\star(x)$ is a well-defined single-valued function of $x$.
\end{proposition}

\begin{proof}
\emph{Existence.} The strategy set $\DeltaK$ is compact and convex, and $F_\Om$ is upper-semicontinuous and concave in $\pi$ and lower-semicontinuous and convex in $q$. Hence Sion's minimax theorem \citep{Sion.1958} gives
\begin{align*}
\max_{\pi\in\DeltaK}\ \min_{q\in\DeltaK}F_\Om(\pi,q)
\ =\ \min_{q\in\DeltaK}\ \max_{\pi\in\DeltaK}F_\Om(\pi,q)
\ =:\ \mathrm{val}(F_\Om),
\end{align*}
with both extrema attained, so a saddle point exists.

\emph{Value zero and symmetry.} Skew-symmetry of $A$ gives $F_\Om(p,q)=-F_\Om(q,p)$ for all $p,q$, since $p^\top Aq=-q^\top Ap$ and the bracket $[\Om(p)-\Om(q)]$ also changes sign. The game is thus symmetric. If we exchange the roles of the players, we map saddle points to saddle points and negate the value. Hence, $\mathrm{val}(F_\Om)=-\mathrm{val}(F_\Om)$, so $\mathrm{val}(F_\Om)=0$. Let $(\pi^\star,q^\star)$ be any saddle point. Because the value is zero and $F_\Om(p,p)=0$ for every $p$, antisymmetry implies that $(q^\star,\pi^\star)$ is also a saddle point. The standard interchangeability of saddle-point components then implies that $(\pi^\star,\pi^\star)$ and $(q^\star,q^\star)$ are saddle points. It therefore suffices to establish uniqueness among symmetric saddle points.

\emph{Uniqueness.} By Lemma~\ref{lem:selfplay}, symmetric saddle points are exactly the solutions of $\mathrm{VI}(V_\Om,\DeltaK)$, so it suffices to show that this variational inequality has at most one solution. Suppose $\pi_1,\pi_2\in\DeltaK$ both solve it. We test the inequality for $\pi_1$ at the feasible point $\pi_2$ and test the inequality for $\pi_2$ at the feasible point $\pi_1$. This gives
\begin{align*}
\langle V_\Om(\pi_1),\ \pi_2-\pi_1\rangle\ \ge\ 0,
\qquad
\langle V_\Om(\pi_2),\ \pi_1-\pi_2\rangle\ \ge\ 0 .
\end{align*}
We add the two inequalities and write $d=\pi_1-\pi_2$. We obtain
\begin{align*}
\langle V_\Om(\pi_1)-V_\Om(\pi_2),\ d\rangle\ \le\ 0 .
\end{align*}
On the other hand, Lemma~\ref{lem:selfplay} states that $V_\Om$ is $\kappa$-strongly monotone, so $\langle V_\Om(\pi_1)-V_\Om(\pi_2),d\rangle\ge\kappa\|d\|^2$. Together, these two bounds imply
\begin{align*}
\kappa\,\|\pi_1-\pi_2\|^2\ \le\ 0 ,
\end{align*}
and since $\kappa>0$ this forces $\|\pi_1-\pi_2\|=0$, i.e.\ $\pi_1=\pi_2$. The symmetric saddle point is therefore unique. For any saddle point $(\pi^\star,q^\star)$, the preceding argument shows that both $(\pi^\star,\pi^\star)$ and $(q^\star,q^\star)$ are symmetric saddle points. Uniqueness gives $\pi^\star=q^\star$, so every saddle point is symmetric and $x\mapsto\pi^\star(x)$ is single-valued.
\end{proof}

\subsection{Gap functions}
\label{subsec:gap-function}

In this subsection, we specialize gap functions to the self-play variational inequality and establish the properties that make the resulting gap a valid differentiable loss for learning the contextual equilibrium.
Gap functions convert a variational inequality into scalar optimization by assigning each candidate solution a nonnegative measure of its violation. The classical construction appears, up to an equivalent change of sign, in \citet{Auslender.1973}. \citet{Fukushima.1992} developed a differentiable regularized gap function, and \citet{Larsson.1994} placed this construction within a broader class of gap functions for variational inequalities. For a general operator $V$ on a feasible set $C$, the classical gap has the form $\sup_{y\in C}\langle V(z),z-y\rangle$: it is nonnegative because $y=z$ is feasible, and it vanishes exactly when $z$ solves $\mathrm{VI}(V,C)$.

In our game, Lemma~\ref{lem:selfplay} identifies equilibrium with the variational inequality induced by the self-play operator $V_\Om(\pi)=-A\pi+\nabla\Om(\pi)$. The regularized gap in Definition~\ref{def:gap-loss} is equivalently the maximum payoff available to an opponent against $\pi$, relative to self-play. Its inner maximizer is the unique regularized best response
\begin{equation}
T_\Om(\pi;A,x)
:=\argmin_{q\in\DeltaK}\big\{\pi^\top Aq+\Om(q)\big\}.
\label{eq:regularized-best-response}
\end{equation}

We define the regularized gap function as
\begin{align}
G_\Om(\pi;A,x)
&:=\max_{q\in\DeltaK}
\Big\{\langle-A\pi,\pi-q\rangle+\Om(\pi)-\Om(q)\Big\} \notag\\
&=\Om(\pi)-\min_{q\in\DeltaK}\big\{\pi^\top Aq+\Om(q)\big\} \notag\\
&=\Om(\pi)-\pi^\top A T_\Om(\pi;A,x)-\Om\big(T_\Om(\pi;A,x)\big).
\label{eq:regularized-gap-app}
\end{align}
The second equality uses $\pi^\top A\pi=0$, which follows from the skew-symmetry of $A$.

Consequently, we can solve the game by minimizing $G_\Om(\pi;A,x)$ over $\DeltaK$. The minimum value is zero, and the equilibrium $\pi^\star(x)$ is its unique minimizer. Strong convexity makes the inner best response unique, while Danskin's theorem supplies a gradient for outer minimization.

\begin{proposition}[Properties of the gap function]\citep{Fukushima.1992,Larsson.1994}
\label{prop:gap-loss-properties}
Fix a context $x$ and a payoff $A$, and let $\pi^\star=\pi^\star(x)$ be the equilibrium. Then, for every $\pi\in\DeltaK$:
\begin{enumerate}
    \item \emph{(Nonnegativity and consistency)} The gap is nonnegative and vanishes exactly at the equilibrium,
    \begin{align}
        G_\Om(\pi;A,x)\ \ge\ 0,
        \qquad
        G_\Om(\pi;A,x)=0\iff\pi=\pi^\star .
        \label{eq:gap-nonneg}
    \end{align}
    \item \emph{(Equilibrium error control)} The squared equilibrium error is bounded by the gap:
    \begin{align}
        \big\|\pi-\pi^\star\big\|^2\ \le\ \frac{2}{\kappa}\,G_\Om(\pi;A,x).
        \label{eq:gap-to-equilibrium}
    \end{align}
    \item \emph{(Differentiability)} The gap $G_\Om(\cdot;A,x)$ is differentiable, with
    \begin{align}
        \nabla_\pi G_\Om(\pi;A,x)
        =-A\,T_\Om(\pi;A,x)+\nabla\Om(\pi),
        \label{eq:gap-gradient}
    \end{align}
    which involves no derivative of $T_\Om$ with respect to $\pi$.
\end{enumerate}
\end{proposition}

\begin{proof}
\emph{Nonnegativity and consistency.} For fixed $\pi$, define $h_\pi(q):=\pi^\top Aq+\Om(q)$. Since $\pi^\top A\pi=0$, Eq.~(\ref{eq:regularized-gap-app}) gives
\begin{align*}
G_\Om(\pi;A,x)=h_\pi(\pi)-\min_{q\in\DeltaK}h_\pi(q)\ge0.
\end{align*}
Equality holds if and only if $\pi$ minimizes $h_\pi$ over $\DeltaK$. The first-order condition for this convex problem is
\begin{align*}
\big\langle-A\pi+\nabla\Om(\pi),q-\pi\big\rangle\ge0
\qquad\text{for every }q\in\DeltaK.
\end{align*}
Lemma~\ref{lem:selfplay} shows that this condition holds if and only if $\pi$ is the unique equilibrium $\pi^\star$.

\emph{Equilibrium error control.} Let $d=\pi-\pi^\star$. We evaluate the maximum in Eq.~(\ref{eq:regularized-gap-app}) at $q=\pi^\star$ and obtain
\begin{align*}
G_\Om(\pi;A,x)
\ge \langle-A\pi,d\rangle+\Om(\pi)-\Om(\pi^\star).
\end{align*}
Strong convexity at $\pi^\star$ gives
\begin{align*}
\Om(\pi)-\Om(\pi^\star)
\ge\langle\nabla\Om(\pi^\star),d\rangle
+\frac{\kappa}{2}\|d\|^2.
\end{align*}
Skew-symmetry gives $\langle-A(\pi-\pi^\star),d\rangle=0$, while the equilibrium variational inequality gives $\langle-A\pi^\star+\nabla\Om(\pi^\star),d\rangle\ge0$. Therefore,
\begin{align*}
G_\Om(\pi;A,x)\ge\frac{\kappa}{2}\|\pi-\pi^\star\|^2,
\end{align*}
which proves Eq.~(\ref{eq:gap-to-equilibrium}).

\emph{Differentiability.} The maximization objective in Eq.~(\ref{eq:regularized-gap-app}) is strongly concave in $q$, so it has the unique maximizer $T_\Om(\pi;A,x)$. For fixed $q$, skew-symmetry gives
\begin{align*}
\nabla_\pi\Big[\langle-A\pi,\pi-q\rangle+\Om(\pi)-\Om(q)\Big]
=-Aq+\nabla\Om(\pi).
\end{align*}
Danskin's theorem evaluates this derivative at the unique maximizer and gives Eq.~(\ref{eq:gap-gradient}).
\end{proof}

\subsection{Equilibrium stability with respect to the payoff}
\label{app:payoff-stability}

In this subsection, we apply variational-inequality sensitivity analysis to bound changes in the contextual equilibrium by perturbations of the payoff matrix, thereby establishing the stability result used in the main analysis.

We denote the dual norm $\|\cdot\|_\star$ as 
\begin{align*}
    \|v\|_\star:=\sup_{\|u\|\le1}\langle v,u\rangle,
    \qquad
    \|B\|_\star:=\sup_{\|u\|\le1}\|Bu\|_\star.
\end{align*}

\begin{proof}[Proof of Proposition~\ref{prop:stability}]
We adapt the standard two-VI sensitivity argument~\citep[Theorem~2F.9]{Dontchev.2009}. We retain the perturbation term induced by the payoff matrix. We fix $x,x'\in\X$ and write
\begin{align*}
\pi=\pi^\star(x),
\qquad
\pi'=\pi^\star(x'),
\qquad
d=\pi-\pi'.
\end{align*}
For brevity, we write $A=A(x)$ and $A'=A(x')$. By Lemma~\ref{lem:selfplay}, we test the two equilibrium variational inequalities at $\pi'$ and $\pi$, respectively, and obtain
\begin{align}
\langle-A\pi+\nabla\Om(\pi),\pi'-\pi\rangle&\ge0,
\label{eq:stability-vi-a}\\
\langle-A'\pi'+\nabla\Om(\pi'),\pi-\pi'\rangle&\ge0.
\label{eq:stability-vi-aprime}
\end{align}
We rewrite both inequalities in terms of $d$ and add them to obtain
\begin{align}
0
&\le \langle A\pi-A'\pi'-\{\nabla\Om(\pi)-\nabla\Om(\pi')\},d\rangle \notag\\
&=\langle(A-A')\pi',d\rangle
  +\underbrace{\langle Ad,d\rangle}_{=0}
  -\langle\nabla\Om(\pi)-\nabla\Om(\pi'),d\rangle .
\label{eq:stability-combine}
\end{align}
The skew-symmetry of $A$ gives $\langle Ad,d\rangle=0$. Because $\Om$ is differentiable and $\kappa$-strongly convex, its gradient is $\kappa$-strongly monotone, so
\begin{align*}
\langle\nabla\Om(\pi)-\nabla\Om(\pi'),d\rangle\ge\kappa\|d\|^2.
\end{align*}
We combine this inequality with Eq.~(\ref{eq:stability-combine}) and apply dual H\"older to obtain
\begin{align*}
\kappa\|d\|^2
\le \langle(A-A')\pi',d\rangle
\le \|(A-A')\pi'\|_*\,\|d\|.
\end{align*}
If $d=0$, the claim is immediate. Otherwise, we divide by $\kappa\|d\|$ and obtain
\begin{align*}
\|\pi^\star(x)-\pi^\star(x')\|
\le\kappa^{-1}\|\{A(x)-A(x')\}\pi^\star(x')\|_*.
\end{align*}
Finally, $\|\pi'\|\le1$ by the norm normalization in the main text, and therefore
\begin{align*}
\|(A-A')\pi'\|_*
\le\|A-A'\|_\star\|\pi'\|
\le\|A-A'\|_\star.
\end{align*}
This proves both inequalities in Proposition~\ref{prop:stability}.
\end{proof}

\subsubsection{Instability of the unregularized equilibrium/maximal lottery}
\label{app:unregularized-instability}

Without regularization, a single-valued equilibrium selection is not uniformly continuous in the payoff matrix near games with multiple equilibria \citep{Kohlberg.1986,Facchinei.2003}. For $\epsilon>0$, we consider the two skew-symmetric games
\begin{align*}
A_\epsilon^+=\epsilon
\begin{pmatrix}0&1\\-1&0\end{pmatrix},
\qquad
A_\epsilon^-=-A_\epsilon^+.
\end{align*}
The unique maximal lottery of $A_\epsilon^+$ is $e_1$, while that of $A_\epsilon^-$ is $e_2$. Hence
\begin{align*}
\|A_\epsilon^+-A_\epsilon^-\|_{\mathrm{op}}=2\epsilon\longrightarrow0,
\qquad
\|e_1-e_2\|_2=\sqrt{2}.
\end{align*}
Thus, a vanishing payoff perturbation can produce a constant change in the selected equilibrium. The discontinuity occurs at the limiting tie game $A=0$, whose equilibrium set is the whole simplex. We need additional regularity or strong convexity to obtain a uniformly stable single-valued target.

\newpage
\section{Background on efficient influence functions and orthogonal statistical learning}
\label{sec:ortho-stats}
Our work draws on efficient semiparametric inference and orthogonal statistical learning. In Section~\ref{subsec:efficient-influence-function}, we define the efficient influence function, use the von Mises expansion to isolate plug-in bias, and specialize these concepts to the average contextual payoff. In Section~\ref{subsec:pseudo-outcome}, we use the efficient influence function to construct a one-step estimator and then convert its uncentered form into a regression target for the contextual payoff. In Section~\ref{subsec:orthogonal-statistical-learning-background}, we extend this bias correction from a scalar functional to our context-to-equilibrium mapping and explain the resulting robustness guarantee. We refer the reader to \citet{Kennedy.2024review} for a detailed treatment of influence-function-based estimation and to \citet{Foster.2023} for the general theory of orthogonal statistical learning.

\subsection{Efficient influence function}
\label{subsec:efficient-influence-function}

Let $O\sim P_0$, where the true distribution $P_0$ belongs to a statistical model $\mathcal P$, and let $\psi:\mathcal P\to\mathbb R$ be a scalar target functional. Consider any regular one-dimensional parametric submodel $\{P_t:|t|<\epsilon\}\subseteq\mathcal P$ through $P_0$, with $P_{t=0}=P_0$ and score $s(O)=\left.\partial_t\log p_t(O)\right|_{t=0}$. The efficient influence function (EIF) $\varphi(O;P_0)$ is the mean-zero, finite-variance element of the tangent space that represents every pathwise derivative:
\begin{equation}
    \left.\frac{\mathrm d}{\mathrm dt}\psi(P_t)\right|_{t=0}
    =\E_0\big[\varphi(O;P_0)s(O)\big],
    \qquad \E_0\big[\varphi(O;P_0)\big]=0.
    \label{eq:eif-pathwise-definition}
\end{equation}
In a nonparametric model, this representation determines the EIF uniquely. In a semiparametric model, the EIF is the influence function in the tangent space with the smallest variance, and its variance gives the semiparametric efficiency bound~\citep{Bickel.1998,vanderVaart.1998,Kennedy.2024review}.

A smooth functional admits the von Mises expansion
\begin{equation}
    \psi(\overline P)-\psi(P)
    =\int \varphi(o;\overline P)\,\mathrm d(\overline P-P)(o)
    +R_2(\overline P,P),
    \label{eq:von-mises-expansion}
\end{equation}
where $R_2(\overline P,P)$ is a second-order remainder composed of products or squares of differences between components of $\overline P$ and $P$. The EIF gives the linear response of $\psi$ to a local perturbation of the data distribution, while $R_2$ collects the nonlinear residual. We set $\overline P=\widehat P$ and $P=P_0$ in Eq.~(\ref{eq:von-mises-expansion}). Since $\int\varphi(o;\widehat P)\,\mathrm d\widehat P(o)=0$, the plug-in error satisfies
\begin{equation}
    \psi(\widehat P)-\psi(P_0)
    =-\E_0\big[\varphi(O;\widehat P)\big]
    +R_2(\widehat P,P_0).
    \label{eq:plugin-von-mises-bias}
\end{equation}
Thus, the first term in Eq.~(\ref{eq:plugin-von-mises-bias}) is the first-order plug-in bias, conditional on the fitted distribution $\widehat P$. A direct plug-in estimator retains this term, so its bias can have the same first-order magnitude as the nuisance-estimation error. An EIF-based correction estimates and cancels this linear term, leaving only the second-order remainder~\citep{Kennedy.2024review}.

In our setting, we derive the EIF for the average payoff because point evaluation of a conditional regression generally need not be pathwise differentiable when $X$ is continuous~\citep{Kennedy.2024review}. For a pair $(j,k)$, this scalar functional is
\begin{equation}
    \psi_{jk}(P)
    :=\E_P\!\left[f_{jk}\!\left(\{\mu^m_{P,jk}(X)\}_{m\in\Qcal_{jk}}\right)\right],
    \label{eq:average-payoff-functional}
\end{equation}
where $\eta_P=(\mucol_P,\ecol_P)$ collects the outcome regressions and menu-selection propensities induced by $P$. Under MAR, positivity, and the regularity conditions of Lemma~\ref{lem:payoff-eif}, its EIF is
\begin{equation}
    \varphi_{jk}(O;P)
    =\Gamma_{jk}(O;\eta_P)-\psi_{jk}(P),
    \label{eq:payoff-eif-background}
\end{equation}
where $\Gamma_{jk}$ is defined as
\begin{align}
    \Gamma_{jk}(O;\eta)
    = f_{jk}\!\left(\{\mu^m_{jk}(X)\}_{m\in\Qcal_{jk}}\right)
    +\sum_{m\in\Qcal_{jk}} 
    \frac{\partial f_{jk}}{\partial u_m}\!\left(\{\mu^{m'}_{jk}(X)\}_{m'\in\Qcal_{jk}}\right)
    \frac{S^m}{e^m(X)}\big\{Z^m_{jk}-\mu^m_{jk}(X)\big\}
    \label{eq:payoff-pseudo-outcome}
\end{align}

\subsection{One-step estimator and pseudo-outcome}
\label{subsec:pseudo-outcome}

Suppose that we estimate $P_0$ by $\widehat P$ on an auxiliary sample or through cross-fitting, and let $\mathbb P_n$ denote the empirical measure on the evaluation observations. The generic one-step estimator is
\begin{equation}
    \widehat\psi^{\mathrm{1S}}
    :=\psi(\widehat P)+\mathbb P_n\big[\varphi(O;\widehat P)\big].
    \label{eq:one-step-general}
\end{equation}
Equation~(\ref{eq:plugin-von-mises-bias}) shows why this correction works. The expectation of the added EIF term estimates the negative of the first-order plug-in bias, so the remaining bias is $R_2(\widehat P,P_0)$.

For the average payoff, Eqs.~(\ref{eq:payoff-eif-background}) and~(\ref{eq:one-step-general}) reduce the one-step estimator to
\begin{equation}
    \widehat\psi^{\mathrm{DB}}_{jk}
    =\mathbb P_n\big[\Gamma_{jk}(O;\widehat\eta)\big].
    \label{eq:payoff-one-step}
\end{equation}
The plug-in payoff inside $\Gamma_{jk}$ and the centered EIF correction combine into a single sample average. This construction explains why we call $\Gamma_{jk}$ an EIF-based pseudo-outcome: it is the uncentered EIF, $\Gamma_{jk}(O;\eta_P)=\psi_{jk}(P)+\varphi_{jk}(O;P)$.

Our target is the contextual payoff $A_{jk}(x)$ rather than only its population average. The pseudo-outcome makes this conditional target accessible because
\begin{equation}
    \E_0\big[\Gamma_{jk}(O;\eta_0)\mid X=x\big]=A_{jk}(x).
    \label{eq:pseudo-outcome-conditional-payoff}
\end{equation}
For a generic nuisance $\eta$, the same conditional expectation defines the nuisance-indexed surface $\ADB_{jk}(x;\eta)$ in Eq.~(\ref{eq:debiased-payoff-surface}). Theorem~\ref{thm:DPME} shows that this surface equals $A_{jk}(x)$ at $\eta_0$ and has zero first derivative with respect to the nuisance. Equation~(\ref{eq:reminder}) further shows that its conditional bias depends only on $\|\Delta\mu\|_2^2$ and products of $\Delta e^m$ and $\Delta\mu^m$. We therefore regress the cross-fitted values $\Gamma_{jk}(O;\widehat\eta)$ on $X$ to estimate the full contextual surface $\AhatDB_{jk}(x)$, which we call the \textbf{debiased payoff matrix estimator (DPME)}.

\subsection{Orthogonal statistical learning}
\label{subsec:orthogonal-statistical-learning-background}

Classical one-step estimation typically targets a finite-dimensional functional such as $\psi_{jk}$. Our target $x\mapsto\pi^\star(x)$ is a function, so we use orthogonal statistical learning (OSL) to estimate it through a two-stage risk minimization problem~\citep{Foster.2023}. Let $\theta$ index a target function and let $\eta$ denote nuisance functions estimated in the first stage. For a population risk $\mathcal R(\theta;\eta)$, Neyman-orthogonality requires the target gradient to be locally insensitive to nuisance perturbations:
\begin{equation}
    \left.D_\eta D_\theta\mathcal R(\theta;\eta)
    [h_\theta,h_\eta]\right|_{\eta=\eta_0}=0.
    \label{eq:orthogonal-learning-background}
\end{equation}
When Eq.~(\ref{eq:orthogonal-learning-background}) holds for every $\theta$, we call the risk universally orthogonal. This property removes the first-order effect of nuisance estimation from the learning objective rather than only from a scalar estimator.

\textbf{Neyman-orthogonality.} Neyman-orthogonality describes the derivative of a target moment or loss gradient with respect to nuisance functions. An EIF-based correction often constructs an orthogonal moment because it cancels the first-order nuisance derivative identified by the von Mises expansion. Therefore, the EIF can often be used to construct an orthogonal loss function for OSL.

We apply the same principle twice. First, the EIF-based pseudo-outcome yields the orthogonal payoff surface $\ADB(x;\eta)$. Second, we insert this surface into the regularized gap risk
\begin{equation}
    \Lcal(\theta;\eta)
    =\E_X\!\left[G_\Om\big(\pi_\theta(X);\ADB(X;\eta),X\big)\right].
    \label{eq:orthogonal-gap-background}
\end{equation}
The derivative $D_\eta\ADB(x;\eta)[h_\eta]$ vanishes at $\eta_0$ by Theorem~\ref{thm:DPME}. The chain rule therefore makes the mixed derivative of Eq.~(\ref{eq:orthogonal-gap-background}) vanish for every $\theta$, as formalized in Theorem~\ref{thm:gap-loss-orthogonal}. In practice, we cross-fit $\widehat\eta$, construct $\Gamma_{jk}(O;\widehat\eta)$, regress these pseudo-outcomes on $X$ to obtain $\AhatDB$, and minimize the empirical gap loss to learn $\pi_{\widehat\theta}$.

Orthogonality changes the leading nuisance contribution from first order to second order. In our payoff learner, Eq.~(\ref{eq:reminder}) gives the rate $O_p(r_\mu^2+r_\mu r_e)$ under the stated norm conditions, instead of the $O_p(r_\mu)$ plug-in bias. For a linear aggregator, the term $r_\mu^2$ vanishes, and the remaining product term gives the usual double-robust property. This higher-order dependence permits slower and more flexible nuisance estimation while retaining the oracle learning rate when the conditions of Corollary~\ref{cor:quasi-oracle} hold. This creates favorable empirical generalization properties, and allows the use of general machine learning models.

\newpage
\section{Benefits of equilibrium-based evaluation: a social-choice perspective}
\label{sec:benefits-vnw}

\subsection{How do cycles arise?}
\label{example:arising-cycles}
Here we give a simple example of how the cyclic preference arises during the aggregation of voters. Consider three equally weighted raters and three alternatives $a,b,c$, with preferences $a\succ_1 b\succ_1 c$, $b\succ_2 c\succ_2 a$, and $c\succ_3 a\succ_3 b$. Each rater is individually transitive. In the order $(a,b,c)$, their majority-margin matrices are
\begin{align*}
A^{(1)}=\begin{pmatrix}0&1&1\\-1&0&1\\-1&-1&0\end{pmatrix},\quad
A^{(2)}=\begin{pmatrix}0&-1&-1\\1&0&1\\1&-1&0\end{pmatrix},\quad
A^{(3)}=\begin{pmatrix}0&1&-1\\-1&0&-1\\1&1&0\end{pmatrix}.
\end{align*}
Aggregating the three raters gives
\begin{align*}
A=\frac{1}{3}\sum_{h=1}^3 A^{(h)}
=\frac{1}{3}\begin{pmatrix}0&1&-1\\-1&0&1\\1&-1&0\end{pmatrix},
\end{align*}
so that $a\succ b$, $b\succ c$, and $c\succ a$. Thus, cycles can arise from aggregation even when no individual rater is intransitive.

\subsection{Why can an equilibrium policy better align with human preferences?}
In the profile above, each alternative is ranked first, second, and third by exactly one rater. A policy that respects this diversity should therefore treat the alternatives symmetrically. The equilibrium does so: for this problem, it is the uniform policy $\pi^*=(1/3,1/3,1/3)$. In contrast, any deterministic selection (i.e., a pure policy) would arbitrarily favor one of the three equally represented groups.

\subsection{Social choice axioms}
The maximal-lottery rule satisfies several desirable social-choice axioms~\citep{Brandl.2016}:
\begin{itemize}
    \item \emph{Condorcet consistency:} If one alternative defeats every other alternative by a positive margin, the unique maximal lottery places all mass on it.
    \item \emph{Population consistency:} If a lottery is maximal for two raters separately, it remains maximal after their profiles are combined.
    \item \emph{Composition consistency:} A component of similar alternatives can be evaluated internally without changing its relationship to alternatives outside the component. Thus, adding variants or copies of one agent does not distort the probabilities assigned to unrelated agents.
    \item \emph{Cloning consistency:} Introducing a clone of an alternative does not change the probabilities assigned to alternatives outside the cloned component.
\end{itemize}
For agent evaluation, these axioms provide concrete operational guarantees. An agent that defeats every competitor is selected with probability one; combining feedback from user groups that agree on an evaluator preserves that evaluator; and adding duplicate or closely related variants of one agent does not change the probabilities assigned to unrelated agents. Thus, the evaluation remains consistent when feedback is pooled and when the candidate set contains redundant agent variants.

Adding regularization may change which axioms can be satisfied. For the KL-regularized equilibrium, \citet{Korkmaz.2026} establishes independence of irrelevant alternatives, population consistency, a regularized form of Pareto optimality, and approximate strategy-proofness. These guarantees motivate regularization in preference-based AI systems, but they depend on the choice of regularizer and do not generally hold for arbitrary strongly convex regularizers.

\newpage
\newpage
\section{instantiations}
\label{app:instantiations}

\subsection{A summary of different comparative feedback}
\label{subsec:feedback-modalities}
We summarize the different relative feedback modalities and how to convert them into ordered partitions in Table~\ref{tab:modalities}. We further provide an example of the conversion on a menu of five items.

\begin{table}[t]
\centering
\caption{Turning comparative feedback modalities into ordered partitions, illustrated on the menu $m=\{1,\dots,5\}$. Note the zero-verdict may have different interpretations for different modalities.}
\label{tab:modalities}
\small
\setlength{\tabcolsep}{5pt}
\begin{tabularx}{\textwidth}{@{}lll>{\raggedright\arraybackslash}X@{}}
\toprule
Modality & Raw response & Ordered partition $Y_m$ & Zero verdicts\\
\midrule
Pairwise duel      & $1\succ3$                          & $(\{1\},\{3\})$                         & a declared tie \\
Full ranking       & $1\succ2\succ3\succ4\succ5$        & $(\{1\},\{2\},\{3\},\{4\},\{5\})$       & none arise \\
Weak ranking       & $\{1,2\}\succ\{3\}\succ\{4,5\}$    & $(\{1,2\},\{3\},\{4,5\})$               & revealed indifference \\
Top-$r$ ($r=2$)    & $T=\{1,2\}$                        & $(\{1,2\},\{3,4,5\})$                   & censoring: the order inside $T$ and inside its complement is never asked \\
Best--worst        & best $1$, worst $5$                & $(\{1\},\{2,3,4\},\{5\})$               & censoring within the middle block \\
Cardinal scores    & $s=(5,5,3,1,1)$                    & $(\{1,2\},\{3\},\{4,5\})$               & equal scores \\
\bottomrule
\end{tabularx}
\end{table}

\textbf{Example.} Take the menu $m=\{1,\dots,5\}$ and the reported weak ranking $Y_m=(\{1,2\}\succ\{3\}\succ\{4,5\})$, so that $b_{Y_m}(1)=b_{Y_m}(2)=1$, $b_{Y_m}(3)=2$ and $b_{Y_m}(4)=b_{Y_m}(5)=3$. Applying \eqref{eq:verdict} to each of the $\binom{5}{2}=10$ pairs gives
\begin{align*}
Z^m=\begin{pmatrix}
\phantom{-}0 & \phantom{-}0 & \phantom{-}1 & \phantom{-}1 & \phantom{-}1\\
\phantom{-}0 & \phantom{-}0 & \phantom{-}1 & \phantom{-}1 & \phantom{-}1\\
-1 & -1 & \phantom{-}0 & \phantom{-}1 & \phantom{-}1\\
-1 & -1 & -1 & \phantom{-}0 & \phantom{-}0\\
-1 & -1 & -1 & \phantom{-}0 & \phantom{-}0
\end{pmatrix}.
\end{align*}

\subsection{Instantiations of the aggregator function}
\label{subsec:aggregator-instantiations}
We provide three instantiations of the aggregator $f_{jk}$ in Definition~\ref{def:payoff}. In addition to smoothness and oddness, we typically require idempotence, $f_{jk}(u,\dots,u)=u$. Idempotence ensures that $A_{jk}(x)=u$ whenever every contributing menu has the same payoff $g^m_{jk}(x)=u$. The choice of aggregator determines how menu-specific preferences contribute to the contextual payoff and can therefore change the equilibrium.

\textbf{When aggregation is useful.} An aggregator is useful when the same pair appears in multiple overlapping menus and those menus provide distinct evidence about the pairwise preference. This occurs, for example, when evaluators report full rankings, weak rankings, or top-$r$ choices over menus of different sizes. Linear weights can represent a prespecified deployment distribution over menus, balance menu sizes, or downweight menus known to produce less reliable or more heavily censored comparisons. A nonlinear aggregator can make the influence of a menu-specific margin depend on its magnitude. The cubic-link construction below, for example, gives larger absolute margins more influence on the transformed scale. In all these cases, the aggregator forms part of the estimand: different choices can produce different payoff matrices and equilibria when preferences depend on menu composition. Therefore, researchers should specify the weights independently of the observed outcomes or treat estimated weights as additional nuisance functions. Every menu assigned positive weight must also satisfy the menu-specific coverage and positivity assumptions.

\textbf{Notes on pairwise feedback.} When every menu contains exactly two agents, each covered pair $(j,k)$ has the single menu $m=\{j,k\}$, so $\Qcal_{jk}=\{\{j,k\}\}$. No cross-menu aggregation is needed, and we use the identity map
\begin{equation}
    f_{jk}(u_{\{j,k\}})=u_{\{j,k\}},
    \qquad
    A_{jk}(x)=g^{\{j,k\}}_{jk}(x).
    \label{eq:pairwise-identity-aggregator}
\end{equation}

\textbf{Linear aggregator.} Let $w^m_{jk}\geq0$ be fixed menu weights satisfying $\sum_{m\in\Qcal_{jk}}w^m_{jk}=1$ and $w^m_{kj}=w^m_{jk}$. We define
\begin{equation}
    f_{jk}\big((u_m)_{m\in\Qcal_{jk}}\big)
    :=\sum_{m\in\Qcal_{jk}}w^m_{jk}u_m.
    \label{eq:linear-menu-aggregator}
\end{equation}
This aggregator is smooth, odd, idempotent, and valued in $[-1,1]$. Uniform weights recover the mean aggregator in Section~\ref{sec:problem}, while nonuniform weights enable practitioners to customize the importance of each menu.

\textbf{Nonlinear aggregator.} We can apply a smooth odd link before averaging. For $\alpha\geq0$, define the strictly increasing bijection $\phi_\alpha:[-1,1]\to[-1,1]$ by
\begin{equation}
    \phi_\alpha(u):=\frac{u+\alpha u^3}{1+\alpha},
\end{equation}
and use the link-transformed weighted aggregator
\begin{equation}
    f_{jk}^{(\alpha)}\big((u_m)_{m\in\Qcal_{jk}}\big)
    :=\phi_\alpha^{-1}\!\left(
        \sum_{m\in\Qcal_{jk}}w^m_{jk}\phi_\alpha(u_m)
    \right).
    \label{eq:nonlinear-menu-aggregator}
\end{equation}
The oddness of $\phi_\alpha$ and $\phi_\alpha^{-1}$ makes $f_{jk}^{(\alpha)}$ odd. The normalized weights make it idempotent. Moreover, $\phi_\alpha'(u)=(1+3\alpha u^2)/(1+\alpha)>0$, so its inverse is smooth and the weighted average remains in its range. The choice $\alpha=0$ recovers the linear aggregator in Eq.~(\ref{eq:linear-menu-aggregator}). When $\alpha>0$ and at least two menus receive positive weight, the construction generally gives a nonlinear aggregation rule. Unlike the linear case, this nonlinear aggregator can have nonzero curvature, so the quadratic Taylor remainder in Eq.~(\ref{eq:reminder}) need not vanish.

\subsection{Instantiations of the cost function}
\label{subsec:cost-instantiations}

In the main section, we use the linear expected cost $\Phi(\pi)=c^\top\pi$, where $c\in\mathbb R_+^K$ contains the deployment costs. We next give a nonlinear soft-budget extension and an exact hard-budget formulation.

\textbf{Convex cost disutility.} Let $\varphi:\mathbb R_+\to\mathbb R$ be a continuously differentiable, convex, and nondecreasing disutility of expected spend. We define
\begin{equation}
    \Phi(\pi):=\varphi(c^\top\pi),
    \qquad
    \Om(\pi):=\Om_0(\pi)+\Phi(\pi),
    \label{eq:convex-cost-disutility}
\end{equation}
where $\Om_0$ is the differentiable, $\kappa$-strongly convex base regularizer. Convexity of $\Phi$ preserves the strong-convexity modulus of $\Om_0$. A smooth family that interpolates between linear pricing and a soft budget is
\begin{equation}
    \varphi(s)
    :=\lambda_0s+\bar\lambda\gamma
    \log\!\left(1+\exp\!\left(\frac{s-B}{\gamma}\right)\right),
    \qquad
    \varphi'(s)=\lambda_0+\bar\lambda\,
    \sigma\!\left(\frac{s-B}{\gamma}\right),
    \label{eq:soft-budget-disutility}
\end{equation}
where $\lambda_0,\bar\lambda\geq0$, $B$ is the soft budget, $\gamma>0$ controls the transition width, and $\sigma(t)=(1+e^{-t})^{-1}$. The marginal price therefore increases smoothly from $\lambda_0$ to $\lambda_0+\bar\lambda$ as expected spend crosses $B$. The choice $\bar\lambda=0$ recovers linear pricing, while the limit $\gamma\downarrow0$ yields the hinge disutility $\lambda_0s+\bar\lambda(s-B)_+$.

\textbf{Hard cost budget.} Let $C\geq\min_{j\in\M}c_j$ and define the nonempty compact convex strategy set
\begin{equation}
    \Delta_K(C):=\{\pi\in\DeltaK:c^\top\pi\leq C\}.
    \label{eq:hard-budget-simplex}
\end{equation}
We impose the budget on both players and define the constrained contextual equilibrium as
\begin{equation}
    \pi_C^\star(x)
    :=\argmax_{\pi\in\Delta_K(C)}
      \min_{q\in\Delta_K(C)}
      \left\{\pi^\top A(x)q-
      [\Om_0(\pi)-\Om_0(q)]\right\}.
    \label{eq:hard-budget-equilibrium}
\end{equation}
Using the same feasible set for both players preserves antisymmetry and the zero game value. Compactness and convexity give existence, while strong convexity of $\Om_0$ gives uniqueness by the same argument as in Appendix~\ref{app:game-theory}. Alternatively, we can also choose a large $\bar\lambda$ (e.g. $\bar\lambda=100$) and still use the unconstrained optimization with Eq.~(\ref{eq:convex-cost-disutility}). This yields a approximately-hard budget formulation and usually is easier to solve.

\newpage
\section{Simulated agent-evalaution dataset}
\label{app:synthetic-dgp}

In this section, we construct a simulated agent-evaluation dataset with context-dependent agent performance and heterogeneous voter preferences. The data-generating process (DGP) generates pairwise comparisons, partial rankings, or winner-only feedback and provides the exact contextual payoff matrix for each modality, making precise evaluation possible.

\subsection{Context and vector performance}
\label{subsec:syn-vector-performance}

We draw contexts independently as $X_i\sim\operatorname{Unif}([0,1]^d)$ with $d=5$ and use $K=5$ agents. Agent $j$ has a performance vector $r_j(x)\in\mathbb R^p$ with $p=12$. We partition its coordinates into three groups $\mathcal G_g=\{4g+1,\ldots,4g+4\}$ indexed by $g\in\{0,1,2\}$.
For each coordinate $t\in\mathcal G_g$, we define
\begin{equation}
 r_{jt}(x)=a_s s_j(x)+a_c C_{jt}+a_h\tanh\{h_{jt}(x)\},
 \qquad (a_s,a_c,a_h)=(1,2,0.15).
 \label{eq:syn-vector-performance}
\end{equation}
Here, $s_j(x)$ represents quality shared across criteria, $C_{jt}$ represents group specialization, and $h_{jt}(x)$ represents criterion-specific contextual variation.

We generate $s_j$ and $h_{jt}$ from two independent random Fourier feature fields. For a generic output $a$, we define the unstandardized field as
\begin{equation}
\begin{aligned}
 f_a(x)&=\sqrt{\frac{2}{D}}\sum_{b=1}^{D}
          V_{ba}\cos(\omega_b^\top x+\phi_b),\\
 \omega_b&\sim\mathcal N(0,\ell_f^{-2}I_d),\qquad
 \phi_b\sim\operatorname{Unif}(0,2\pi),\qquad
 V_{ba}\sim\mathcal N(0,1).
\end{aligned}
\label{eq:syn-rff}
\end{equation}
We draw all frequencies, phases, and coefficients independently. Within each field, the outputs share frequencies and phases but use separate coefficient columns. The shared-quality and criterion-specific fields have $K$ and $Kp$ outputs, respectively. We use $D=128$ and $\ell_f=0.5$, then center and scale each output to unit empirical standard deviation on 4,096 independent uniform calibration contexts. The standardized outputs define $s_j(x)$ and $h_{jt}(x)$.

To specify specialization, we draw a uniform permutation $(a_1,\ldots,a_K)$ of $\{0,\ldots,K-1\}$ and set $(o_0,o_1,o_2)=(0,2,4)$. Then
\begin{equation}
 C_{jt}=2-\frac{2}{K-1}\bigl[(a_j-o_g)\bmod K\bigr],
 \qquad t\in\mathcal G_g.
 \label{eq:syn-specialization}
\end{equation}
We hold all latent field draws and specialization profiles fixed within each replication.

\subsection{Voters and population preferences}
\label{subsec:syn-voters}

Voter $v$ assigns utility
\begin{equation}
 u_{vj}(x)=w_v^\top r_j(x),\qquad w_v\in\Delta_p.
 \label{eq:syn-voter-utility}
\end{equation}
The population contains $N_v\in\{1,3,5\}$ voters. We draw a pool of five voters with preferred groups $(g_1,\ldots,g_5)=(0,1,2,0,1)$ and use its first $N_v$ members. For each voter, we independently draw $\theta_v\sim\operatorname{Dirichlet}(\alpha\mathbf 1_4)$ with $\alpha=20$. We construct $\widetilde\theta_v\in\mathbb R^p$ by placing these four coordinates in $\mathcal G_{g_v}$ and setting the remaining coordinates to zero. We then set
\begin{equation}
 w_v=\frac{\epsilon}{p}\mathbf 1_p+(1-\epsilon)\widetilde\theta_v,
 \qquad \epsilon=0.1.
 \label{eq:syn-voter-weights}
\end{equation}
Each member of the finite population has weight $1/N_v$, and each response samples a voter uniformly from this population.

For distinct utilities, we define the latent population margin as
\begin{equation}
 M_{jk}(x)=\frac{1}{N_v}\sum_{v=1}^{N_v}
            \operatorname{sign}\{u_{vj}(x)-u_{vk}(x)\},
 \qquad M_{kj}(x)=-M_{jk}(x),\quad M_{jj}(x)=0.
 \label{eq:syn-latent-margin}
\end{equation}
The continuous field draws make exact utility ties a probability-zero event, and the implementation resolves numerical ties by agent index. With $N_v=1$, $M(x)$ follows the voter's transitive ordering. With $N_v\in\{3,5\}$, majority aggregation can generate cycles. The populations of sizes three and five contain preferred-group counts $(1,1,1)$ and $(2,2,1)$, respectively.

\subsection{Menus, selection, and feedback}
\label{subsec:syn-feedback}

A menu $m$ contains the agents presented to one voter. Pairwise feedback uses all ten menus of size two. Partial-ranking feedback uses the five menus
\begin{equation}
 \Qcal_{\mathrm{rank}}=
 \bigl\{\{1,2,3\},\{1,2,5\},\{1,4,5\},\{2,3,4\},\{3,4,5\}\bigr\}.
 \label{eq:syn-ranking-menus}
\end{equation}
This catalogue covers every pair. Winner feedback uses all $\binom{5}{3}=10$ menus of size three. A pair belongs to one or two ranking menus and exactly three winner menus.

For each replication, we draw $b\sim\mathcal N(0,I_d)$ and set $\bar b=b/\|b\|_2$. We define $\sigma(z)=(1+e^{-z})^{-1}$ and
\begin{equation}
\begin{aligned}
 e^m_0(x)&=e_0(x)
 =e_{\min}+(e_{\max}-e_{\min})
   \sigma\!\left(\eta\sqrt{12}\,\bar b^\top
                          (x-\tfrac12\mathbf 1_d)\right),\\
 S_i^m\mid X_i&\sim\operatorname{Bernoulli}\{e_0(X_i)\},
 \qquad (e_{\min},e_{\max},\eta)=(0.25,0.75,2).
\end{aligned}
\label{eq:syn-selection}
\end{equation}
We draw menu selections independently conditional on the context and independently of voter responses. Thus, the DGP satisfies MAR conditional on $X$. Each $e^m_0(x)$ is a menu inclusion probability, so one context can generate multiple records and the probabilities need not sum to one. The lower bound $e_{\min}=0.25$ bounds inverse propensities by four.

For each selected $(i,m)$, we independently sample $V_{im}\sim\operatorname{Unif}\{1,\ldots,N_v\}$. The selected voter orders the agents in $m$ by $u_{V_{im},j}(X_i)$. Pairwise feedback reports the preferred agent, partial-ranking feedback reports the complete order, and winner feedback reports only the highest-utility agent. All comparisons extracted from one response use the same voter. Records at the same context sample voters independently.

For a pair $j,k\in m$, the encoded margin from pairwise or ranking feedback is
\begin{equation}
 Z^m_{jk}=\operatorname{sign}\{u_{V_{im},j}(X_i)-u_{V_{im},k}(X_i)\}.
 \label{eq:syn-ranking-verdict}
\end{equation}
For winner feedback, we set $W_{im}=\argmax_{a\in m}u_{V_{im},a}(X_i)$ and define
\begin{equation}
 Z^m_{jk}=\mathbf 1\{W_{im}=j\}-\mathbf 1\{W_{im}=k\}.
 \label{eq:syn-winner-verdict}
\end{equation}
Eq.~(\ref{eq:syn-winner-verdict}) assigns zero when another menu member wins because winner-only feedback does not reveal the relative order of $j$ and $k$.

\subsection{Ground truth game and experimental parameters}
\label{subsec:syn-ground-truth}

We let $\Qcal_{jk}=\{m:j,k\in m\}$ and $g^m_{jk}(x)=\E[Z^m_{jk}\mid X=x]$. Because the DGP satisfies MAR conditional on $X$, we have $g^m_{jk}(x)=\mu^m_{0,jk}(x)$. We define the ground-truth game by averaging uniformly over the menus containing each pair:
\begin{equation}
 A_{jk}(x)=\sum_{m\in\Qcal_{jk}}\lambda(m\mid j,k)g^m_{jk}(x),
 \qquad \lambda(m\mid j,k)=\frac{1}{|\Qcal_{jk}|}.
 \label{eq:syn-menu-game}
\end{equation}
For pairwise and ranking feedback, $g^m_{jk}(x)=M_{jk}(x)$ on every containing menu, so $A(x)=M(x)$. For winner feedback, we define
\begin{equation}
 p^m_j(x)=\frac{1}{N_v}\sum_{v=1}^{N_v}
   \mathbf 1\!\left\{j=\argmax_{a\in m}u_{va}(x)\right\},
 \qquad g^m_{jk}(x)=p^m_j(x)-p^m_k(x).
 \label{eq:syn-winner-truth}
\end{equation}

Table~\ref{tab:syn-config} gives the parameters used in Figure~\ref{fig:rq1-results}. Within each replication, we share the performance fields, specialization permutation, selection direction, voter pool, and training and test contexts across the nine modality--population-size settings.
\begin{table}[tbp]
\centering
\caption{Parameters of the simulated dataset.}
\label{tab:syn-config}
\small
\setlength{\tabcolsep}{4pt}
\renewcommand{\arraystretch}{1.12}
\begin{tabularx}{\textwidth}{@{}l l X@{}}
\toprule
Parameter & Value & Meaning \\
\midrule
$K,d,p$ & $5,5,12$ & Agents, context dimensions, and performance criteria. \\
$X$ & $\operatorname{Unif}([0,1]^5)$ & Independent training and test contexts. \\
$D,\ell_f$ & $128,0.5$ & Fourier features and field length scale. \\
$n_{\mathrm{cal}}$ & $4{,}096$ & Uniform contexts for field standardization. \\
$(a_s,a_c,a_h)$ & $(1,2,0.15)$ & Shared quality, specialization, and criterion variation amplitudes. \\
$(o_0,o_1,o_2)$ & $(0,2,4)$ & Offsets of the three specialization profiles. \\
$N_v$ & $1,3,5$ & Equally weighted voters, drawn as prefixes of one pool. \\
$(g_1,\ldots,g_5)$ & $(0,1,2,0,1)$ & Preferred criterion groups in the voter pool. \\
$\epsilon,\alpha$ & $0.1,20$ & Uniform weight allocation and Dirichlet concentration. \\
Menu sizes & $2,3,3$ & Pairwise, partial ranking, and winner feedback. \\
Menu counts & $10,5,10$ & Catalogues for the three respective modalities. \\
$(e_{\min},e_{\max},\eta)$ & $(0.25,0.75,2)$ & Inclusion bounds and selection logit scale. \\
$n,n_{\mathrm{test}}$ & $20{,}000,4{,}096$ & Training and test contexts per experiment. \\
Test replications & $5$ & Seeds 70--74, each covering all nine settings. \\
\bottomrule
\end{tabularx}
\end{table}

\newpage
\section{Details of the real-world datasets}
\label{sec:real-world-datasets}
We describe the data sources, context variables, agent populations, and feedback protocols used in our application experiments. In Section~\ref{subsec:lm-arena-dataset}, we describe the human preference data from LM Arena and the context attributes we use to examine heterogeneity across user prompts and tasks. In Section~\ref{subsec:routerbench-dataset}, we explain how we combine benchmark scores with LLM judgments to construct offline pairwise feedback and complete test reference matrices for evaluating contextual policies and their performance under domain shift. In Section~\ref{subsec:rl-dataset}, we describe how we train driving agents with different reward weights and construct comparisons from simulated judges with distinct utilities. This setting extends our evaluation beyond language models to agents that trade off multiple performance objectives. We distinguish observed human judgments in LM Arena from the constructed preferences in RouterBench and the driving simulator, and provide the data-construction details needed to interpret and reproduce these experiments.

\subsection{LM Arena dataset}
\label{subsec:lm-arena-dataset}

LMArena is a crowdsourced platform in which a (human) user submits a prompt, receives responses from two anonymously presented language models, and indicates which response is better, whether the responses tie, or whether both are unsatisfactory~\citep{Chiang.2024}. The Arena Human Preference 140K dataset records these votes together with the conversation histories and associated metadata. The dataset fits our framework: the observed menu $m$ are the pairs $(j, k)$, with pairwise preference $Y_m$. The prompt and its metadata constitute the context $X$. For each context, only one menu's outcome is selected, and the others are not observed. The crowdsourced votes span 126 detected language labels and a variety of tasks, providing naturally heterogeneous evaluation contexts and human preferences.

We use the latest public release, which contains 135,634 votes collected from April 17, 2025 to July 24, 2025, covering 53 models. Our experiments retain the 16 models listed in Table~\ref{tab:selected-llm-models} to ensure minimal coverage over the pairs. For each comparison, we construct the context from the conversation and salient metadata, including language, task-category annotations (e.g., mathematics, creative writing, instruction following, and hard prompts), and whether the conversation involves code.

Tables~\ref{tab:arena-language-distribution} and~\ref{tab:arena-attribute-distribution} report the marginal distributions of the principal discrete context attributes. We aggregate the recorded language labels into five categories: English, Polish, Russian, Chinese (including both \texttt{zh} and \texttt{zh-Hant}), and Other. The code, instruction-following, and mathematics indicators are taken directly from the dataset annotations. 

\begin{table}[htbp]
\centering
\begin{minipage}[t]{0.40\textwidth}
\centering
\caption{Marginal distribution of language groups.}
\label{tab:arena-language-distribution}
\small
\begin{tabular}{@{}lrr@{}}
\toprule
\textbf{Language group} & \textbf{Count} & \textbf{Share} \\
\midrule
English & 71,175 & 52.5\% \\
Polish  & 13,813 & 10.2\% \\
Russian &  9,263 &  6.8\% \\
Chinese &  7,346 &  5.4\% \\
Other   & 34,037 & 25.1\% \\
\bottomrule
\end{tabular}
\end{minipage}
\hfill
\begin{minipage}[t]{0.55\textwidth}
\centering
\caption{Marginal distributions of binary context attributes.}
\label{tab:arena-attribute-distribution}
\small
\begin{tabular}{@{}llrr@{}}
\toprule
\textbf{Attribute} & \textbf{Value} & \textbf{Count} & \textbf{Share} \\
\midrule
Code & Yes & 39,363 & 29.0\% \\
     & No  & 96,271 & 71.0\% \\
Instruction following & Yes & 24,666 & 18.2\% \\
                      & No  & 110,968 & 81.8\% \\
Mathematics & Yes & 10,892 &  8.0\% \\
            & No  & 124,742 & 92.0\% \\
\bottomrule
\end{tabular}
\end{minipage}
\end{table}

\begin{table}[htbp]
\centering
\caption{Selected agents from the Arena Human Preference 140K dataset, with list prices in USD per $10^6$ tokens.}
\label{tab:selected-llm-models}
\small
\begin{tabular}{@{}llrr@{}}
\toprule
\textbf{Provider} & \textbf{Agent} & \textbf{Input} & \textbf{Output} \\
\midrule
\multirow{3}{*}{OpenAI}
& \texttt{o3-2025-04-16} & 2.00 & 8.00 \\
& \texttt{o4-mini-2025-04-16} & 1.10 & 4.40 \\
& \texttt{gpt-4.1-2025-04-14} & 2.00 & 8.00 \\
\midrule
\multirow{2}{*}{Anthropic}
& \texttt{claude-opus-4-20250514} & 15.00 & 75.00 \\
& \texttt{claude-sonnet-4-20250514} & 3.00 & 15.00 \\
\midrule
\multirow{2}{*}{Google}
& \texttt{gemini-2.5-pro} & 1.25 & 10.00 \\
& \texttt{gemini-2.5-flash} & 0.30 & 2.50 \\
\midrule
\multirow{2}{*}{Alibaba (Qwen)}
& \texttt{qwen3-235b-a22b-instruct-2507} & 0.09 & 0.55 \\
& \texttt{qwen3-30b-a3b} & 0.12 & 0.50 \\
\midrule
Moonshot AI & \texttt{kimi-k2-0711-preview} & 0.57 & 2.30 \\
\midrule
MiniMax & \texttt{minimax-m1} & 0.55 & 2.20 \\
\midrule
Mistral AI & \texttt{mistral-medium-2505} & 0.40 & 2.00 \\
\midrule
\multirow{2}{*}{DeepSeek}
& \texttt{deepseek-r1-0528} & 0.50 & 2.15 \\
& \texttt{deepseek-v3-0324} & 0.25 & 1.00 \\
\midrule
Meta & \texttt{llama-4-maverick-17b-128e-instruct} & 0.20 & 0.80 \\
\bottomrule
\end{tabular}
\end{table}

\subsection{RouterBench dataset}
\label{subsec:routerbench-dataset}

RouterBench evaluates LLM routers by running multiple models on a common collection of prompts and recording their responses, task-specific correctness scores, and inference costs~\citep{Hu.2024RouterBench}. We use its latest release~\citep{Li.2026llmrouterbench} and retain the $K=10$ flagship models in Table~\ref{tab:rb-models} with complete coverage across ten subsets (see Table~\ref{tab:rb-subsets}). The resulting benchmark contains $8{,}742$ prompts. The context $X$ includes the prompt, source subset, task domain, mathematics and code indicators, language group, and token length, following the LMArena context schema in Section~\ref{subsec:lm-arena-dataset}.

RouterBench provides absolute scores rather than relative feedback. A score $s_j(x)$ induces $Z_{jk}(x)=\operatorname{sign}\{s_j(x)-s_k(x)\}$ and is therefore transitive by construction. We use this verdict when the native scores distinguish a pair and elicit a preference only when both models succeed (if both fail, then we simply treat it as a both bad verdict). For each prompt, we sample $Q=20$ of the $\binom{10}{2}=45$ pairs uniformly and independently, giving the known propensity $e^m_0(x)=20/45$ and satisfying positivity. 

A sampled pair $(j,k)$ at prompt $x$ is then resolved as
\begin{equation}
Z_{jk}(x)=
\begin{cases}
\operatorname{sign}(s_j-s_k), & s_j\ne s_k \quad\text{(the benchmark orders the pair)},\\
0, & s_j=s_k=0 \quad\text{(both agents fail)},\\
\text{judge verdict}, & s_j=s_k>0 \quad\text{(both agents succeed)},
\end{cases}
\label{eq:rb-label-first}
\end{equation}
so a judge is used only when both models succeed. We draw judges from three latest open-source frontier models supplied by different providers and assign them distinct emphases---rigour, depth, or efficiency---within a common six-metric rubric. This emulates the real-world rater diversity and heterogeneity, which also might introduce intransitive feedback. Each pair is judged in both presentation orders and the verdicts are averaged to mitigate position bias.

We select $500$ prompts from the benchmark as a separate test set, unseen by the training. All $45$ pairs on the $500$ test prompts are evaluated by all three judges, yielding a complete reference payoff matrix as the ground truth. Each of the $8{,}242$ training prompts receives the $20$ sampled pairs and one randomly assigned judge. Overall, the processing yields $232{,}340$ comparisons, including $67{,}500$ test comparisons. Table~\ref{tab:rb-subsets} shows the constitution. Knowledge tasks account for $67.0\%$ of prompts, programming $17.8\%$, chat and instruction following $8.6\%$, logic $4.6\%$, and mathematics $2.1\%$; $20.7\%$ involve code and $4.9\%$ mathematics. Different to LM Arena, the routerbench is $96.0\%$ English, with some CJK. 

\begin{table}[htbp]
\centering
\caption{Composition of the RouterBench panel. ``Feedback modality'' is the benchmark's native
outcome format; the last two columns give the share of sampled pairs that the label
already orders and the share on which both agents succeed and a judge is therefore
consulted.}
\label{tab:rb-subsets}
\small
\setlength{\tabcolsep}{3pt}
\begin{tabularx}{\textwidth}{@{}p{0.18\textwidth}rrp{0.14\textwidth}Xrr@{}}
\toprule
\textbf{Subset} & \textbf{Prompts} & \textbf{Test} & \textbf{Domain} &
\textbf{Feedback modality} & \textbf{Ord.} & \textbf{Judged} \\
\midrule
MMLU-Pro & 3,000 & 50 & knowledge & Binary (correct/incorrect) & 13.1\% & 75.5\% \\
HLE & 2,158 & 50 & knowledge & Binary (correct/incorrect) & 17.1\% & 3.0\% \\
LiveCodeBench & 1,055 & 50 & programming & Binary (pass/fail) & 19.5\% & 60.4\% \\
ArenaHard & 750 & 100 & chat / IF & Ternary (win/tie/loss) & 48.7\% & 42.6\% \\
SimpleQA & 500 & 50 & knowledge & Ternary (correct/incorrect/not attempted) & 30.1\% & 19.5\% \\
SWE-bench & 500 & 50 & programming & Binary (pass/fail) & 24.9\% & 10.8\% \\
ARC-AGI & 400 & 50 & logic & Binary (exact/non-match) & 23.2\% & 9.8\% \\
GPQA & 198 & 50 & knowledge & Binary (correct/incorrect) & 22.3\% & 63.5\% \\
LiveMathBench & 121 & 30 & math & Binary (correct/incorrect) & 21.9\% & 54.1\% \\
AIME & 60 & 20 & math & Binary (correct/incorrect) & 28.4\% & 56.5\% \\
\midrule
\textbf{Total} & \textbf{8,742} & \textbf{500} & & & 22.1\% & 41.3\% \\
\bottomrule
\end{tabularx}
\end{table}

\begin{table}[htbp]
\centering
\caption{Selected agents from the RouterBench dataset, with list prices in USD per $10^6$ tokens.
\textsuperscript{$\dagger$}\texttt{gpt-5-chat} is no longer served; its price is taken
from the benchmark's own records.}
\label{tab:rb-models}
\small
\begin{tabular}{@{}llrr@{}}
\toprule
\textbf{Provider} & \textbf{Agent} & \textbf{Input} & \textbf{Output} \\
\midrule
\multirow{2}{*}{OpenAI}
& \texttt{gpt-5} (medium) & 1.25 & 10.00 \\
& \texttt{gpt-5-chat}\textsuperscript{$\dagger$} & 1.25 & 10.00 \\
\midrule
Anthropic & \texttt{claude-sonnet-4} & 3.00 & 15.00 \\
\midrule
\multirow{2}{*}{Google}
& \texttt{gemini-2.5-pro} & 1.25 & 10.00 \\
& \texttt{gemini-2.5-flash} & 0.30 & 2.50 \\
\midrule
Alibaba & \texttt{qwen3-235b-a22b-thinking-2507} & 0.23 & 2.30 \\
\midrule
\multirow{2}{*}{DeepSeek}
& \texttt{deepseek-v3-0324} & 0.25 & 1.00 \\
& \texttt{deepseek-r1-0528} & 0.50 & 2.15 \\
\midrule
Z-AI & \texttt{glm-4.6} & 0.43 & 1.75 \\
\midrule
Moonshot AI & \texttt{kimi-k2-0905} & 0.60 & 2.50 \\
\bottomrule
\end{tabular}
\end{table}

\subsection{Reinforcement-learning agent dataset}
\label{subsec:rl-dataset}

We construct a contextual agent-evaluation dataset using \texttt{mo-highway-fast-v0}, the multi-objective highway-driving environment in MO-Gymnasium~\citep{Felten.2023Toolkit}, built on HighwayEnv~\citep{Leurent.2018Highway}. An agent controls a vehicle among simulated traffic and receives separate rewards for speed, right-lane occupancy, and collision avoidance; the collision component is a nonpositive penalty. The tasks are contextual, containing five configuration features: lane count, vehicle count, traffic density, initial lane, and initial ego spacing. We sample lane and vehicle counts uniformly from $\{2,3,4,5\}$ and $\{20,\ldots,60\}$, density from $U[0.7,1.5]$, the initial lane uniformly among valid lanes, and ego spacing from $U[0.5,3]$, then fix the configuration for the episode.

The driving task is inherently multi-objecitve: for example, agents can trade speed with safety. We simulate external judges through linear utilities over an episode-performance vector $z_j=(\overline r_{\mathrm{speed}},\overline r_{\mathrm{lane}},1-\mathbf{1}\{\mathrm{collision}\})\in[0,1]^3$. Three equally weighted judges use $U_h(j)=v_h^\top z_j$, with weights approximately $(0.67,0.31,0.02)$, $(0.044, 0.0, 0.95)$, and $(0.26,0.01,0.72)$. Each offline comparison samples a pair menu $m=\{j,k\}$ uniformly from the ten unordered pairs and a judge uniformly from the three judges. Both agents are then evaluated on the same contextual task. The signed verdict $Z^m_{jk}$ is $+1$ if the judge prefers $j$ and $-1$ otherwise, with exact utility ties broken by a fair coin; hence $e^m_0(x)=1/10$ per comparison. Our target $A_{jk}(x)$ further averages over the initialization scenarios.

We train agents by the scalar reward $w_j^\top r$, where $w_j$ specifies agent $j$'s training weights on the reward vector \(r=(r_{\mathrm{speed}},r_{\mathrm{lane}},r_{\mathrm{collision}})\). We apply deep Q-learning~\citep{Mnih.2015DQN} through Stable-Baselines3~\citep{Raffin.2021SB3}. Each agent receives $20{,}000$ interaction steps with two hidden layers of 256 units. We adapt the official HighwayEnv DQN example,\footnote{\url{https://highway-env.farama.org/quickstart/}} using learning rate $5\times10^{-4}$, replay capacity $15{,}000$, batch size 32, and target updates every 50 steps. The final five agents use weights $(0.1,0.8,0.1)$, $(0.8,0.1,0.1)$, $(0.45,0.1,0.45)$, $(0.02,0,0.98)$, and $(0.2,0,0.8)$, respectively. The first three use discount factor $0.8$; the last two use $0.99$ to emphasize longer-term safety. These design choices make sure that the agents trained have a range of varying capabilities.

The resulting dataset contains 256 configurations (i.e. context) with eight initializations and has eventually $20{,}480$ offline comparisons. Validation and test set contain 32 configurations, with 32 and 64 initializations per configuration, respectively. For each validation or test configuration, we evaluate all five agents and average all three judges' verdicts to estimate the complete payoff matrix, without assigning ties. Primary data exploration shows that the contextual PMs contain cycle roughly $1/3$ of the times. 

\newpage
\section{Implementation}
\label{app:implementation}

 In Section~\ref{subsec:eval-metrics}, we define the evaluation metrics used in our experiments. In Section~\ref{subsec:training-details}, we specify the training details of our \method. In Section~\ref{subsec:baselines}, we describe the baselines.

 We summarize our method (i.e. \textbf{NashEval-debiased}) in Algorithm~\ref{alg:nasheval-debiased}.

\begin{algorithm}[p]
\caption{NashEval-debiased: contextual equilibrium learning from selective feedback}
\label{alg:nasheval-debiased}
\small
\begin{algorithmic}[1]
\Require Training log $\mathcal D_n=\{O_i\}_{i=1}^n$, menu aggregators $\{f_{jk}\}_{j<k}$,
context-independent strongly convex regularizer $\Omega$, number of folds $K_{\mathrm{CF}}$,
and policy class $\{\pi_\theta:\theta\in\Theta\}$.
\State Encode each observed $Y_{i,m}$ as verdicts $Z^m_{i,jk}$ using Eq.~(\ref{eq:verdict}).
\State Partition the training indices into folds $I_1,\ldots,I_{K_{\mathrm{CF}}}$, keeping all records from each context together.
\For{$b=1,\ldots,K_{\mathrm{CF}}$}
    \For{each menu $m\in\mathcal Q$}
        \State Fit $\widehat e^{m,(-b)}(x)$ to $S_i^m$ on all $i\notin I_b$ using binary log loss.
        \State Fit $\widehat\mu^{m,(-b)}_{jk}(x)$ to $Z^m_{i,jk}$ for each $j<k$ in $m$ using squared error on $i\notin I_b$ with $S_i^m=1$.
    \EndFor
    \For{each $i\in I_b$ and pair $j<k$}
        \State Set $\widehat{\boldsymbol\mu}^{(-b)}_{i,jk}
        =\big(\widehat\mu^{m,(-b)}_{jk}(X_i)\big)_{m\in\mathcal Q_{jk}}$.
        \State Construct the out-of-fold target from Eq.~(\ref{eq:payoff-pseudo-outcome}):
        \Statex \hspace{\algorithmicindent}$\displaystyle
        \widehat\Gamma_{i,jk}
        =f_{jk}\!\left(\widehat{\boldsymbol\mu}^{(-b)}_{i,jk}\right)
        +\sum_{\substack{m\in\mathcal Q_{jk}\\S_i^m=1}}
        \frac{\partial f_{jk}}{\partial u_m}
        \!\left(\widehat{\boldsymbol\mu}^{(-b)}_{i,jk}\right)
        \frac{Z^m_{i,jk}-\widehat\mu^{m,(-b)}_{jk}(X_i)}
             {\widehat e^{m,(-b)}(X_i)}.$
    \EndFor
\EndFor
\State Pool the targets from all folds and fit a shared payoff regressor:
\Statex \hspace{\algorithmicindent}$\displaystyle
    \widehat a\in\argmin_a\sum_{i=1}^n\sum_{j<k}
    \big\{a(X_i,j,k)-\widehat\Gamma_{i,jk}\big\}^2.$
\State Define $\widehat A^{\mathrm{DB}}_{jk}(x)=\operatorname{clip}_{[-1,1]}\{\widehat a(x,j,k)\}$ for $j<k$,
and set $\widehat A^{\mathrm{DB}}_{kj}(x)=-\widehat A^{\mathrm{DB}}_{jk}(x)$ and $\widehat A^{\mathrm{DB}}_{jj}(x)=0$.
\State Minimize the empirical gap loss to learn $\pi_{\widehat\theta}$:
\Statex \hspace{\algorithmicindent}$\displaystyle
    \widehat{\mathcal L}_n(\theta)
    =\frac{1}{n}\sum_{i=1}^n
    G_\Omega\!\left(\pi_\theta(X_i);\widehat A^{\mathrm{DB}}(X_i),X_i\right).$
\State \Return $\widehat A^{\mathrm{DB}}$ and $\pi_{\widehat\theta}$.
\Statex \textbf{Inference:} For a new context $x_\star$, return $\pi_{\widehat\theta}(x_\star)\in\Delta_K$.
\end{algorithmic}
\end{algorithm}

We evaluate the residual correction only for observed menus, so the algorithm never requires an unobserved verdict. We retain the full pseudo-outcome values during payoff regression and clip only the fitted payoff predictions. We use LightGBM for the nuisance and payoff models and the multilayer perceptron described in Section~\ref{subsec:training-details} for the policy. For experiments that compute equilibria directly, we replace policy training and prediction by a numerical equilibrium solver applied to $\widehat A^{\mathrm{DB}}(x)$ with the experiment's regularizer.

\subsection{Evaluation metrics}
\label{subsec:eval-metrics}
We evaluate a policy $\pi$ on $n_{\mathrm{test}}$ contexts $x_i$ using reference payoff matrices $A_i^{\mathrm{ref}}$, with positive entries indicating preference for the row agent over the column agent. We use the true matrices in synthetic experiments and approximated matrices in RouterBench and MO-Highway.

\paragraph{Average exploitability.}
Exploitability measures the largest advantage an opponent can obtain against the evaluated policy.
We average the unregularized best-response payoff over test contexts:
\begin{equation}
 \widehat{\mathrm{Expl}}_0(\pi)
 =\frac{1}{n_{\mathrm{test}}}\sum_{i=1}^{n_{\mathrm{test}}}
 \max_{q\in\Delta_K}q^\top A_i^{\mathrm{ref}}\pi(x_i).
 \label{eq:eval-exploitability}
\end{equation}
Lower is better, and we allow an unrestricted opponent even when we impose a cost penalty on the evaluated policy.

\paragraph{Support $F_1$.}
Support $F_1$ measures how accurately a policy recovers the reference winner set, without comparing the assigned probabilities.
We define $S_\tau(p)=\{j:p_j>\tau\}$ and average the overlap with the reference equilibrium $\pi_i^{\mathrm{ref}}$:
\begin{equation}
 \widehat F_1(\pi)
 =\frac{1}{n_{\mathrm{test}}}\sum_{i=1}^{n_{\mathrm{test}}}
 \frac{2|S_\tau(\pi(x_i))\cap S_\tau(\pi_i^{\mathrm{ref}})|}
 {|S_\tau(\pi(x_i))|+|S_\tau(\pi_i^{\mathrm{ref}})|}.
 \label{eq:eval-support-f1}
\end{equation}
Higher is better, and we use $\tau=10^{-3}$ in the synthetic experiments and the reference equilibrium under the experiment's specified regularization.

\paragraph{Head-to-head game value.}
Game value measures the evaluated policy's net preference advantage against a specified opponent policy.
We report the average payoff of $\pi$ against $q$:
\begin{equation}
 \widehat V(\pi,q)
 =\frac{1}{n_{\mathrm{test}}}\sum_{i=1}^{n_{\mathrm{test}}}
 \pi(x_i)^\top A_i^{\mathrm{ref}}q(x_i).
 \label{eq:eval-game-value}
\end{equation}
In the RQ2 tables, we take $q$ to be NashEval-debiased, so positive values favor the row method and negative values favor NashEval-debiased.

\paragraph{Average strict win rate.}
Strict win rate measures how often a sampled agent wins against an equally weighted pool of single-agent opponents.
We let $W_{i,jk}$ denote the reference probability that agent $j$ strictly defeats agent $k$ at $x_i$ and compute
\begin{equation}
 \widehat{\mathrm{WR}}(\pi)
 =\frac{1}{n_{\mathrm{test}}K}\sum_{i=1}^{n_{\mathrm{test}}}
 \sum_{j=1}^K\sum_{k=1}^K\pi_j(x_i)W_{i,jk}.
 \label{eq:eval-win-rate}
\end{equation}
We assign ties and self-comparisons zero credit and retain them in the denominator, so payoff margins alone do not determine this metric when ties occur.


\subsection{training details}
\label{subsec:training-details}
We use LightGBM to estimate the outcome regressions $\mu^m_{0,jk}$ and selection propensities $e^m_0$, and to fit the debiased payoff matrix $\widehat A^{\mathrm{DB}}$. We split training contexts into $K_{\mathrm{CF}}$ folds and keep all records from the same context in one fold. For each fold, we fit the nuisance models on the remaining folds and predict the nuisances on the held-out fold to construct its pseudo-outcomes. We then pool these out-of-fold pseudo-outcomes to train the payoff regressor. 

\paragraph{Sentence embedding.}
For contexts that include text, we use the fixed Qwen3-Embedding-4B encoder~\citep{Zhang.2025} to map each prompt to $f_{\mathrm{emb}}(x)\in\mathbb R^d$. We compare Qwen3-Embedding-0.6B, Qwen3-Embedding-4B, and EmbeddingGemma-300M~\citep{SchechterVera.2025} with embedding dimensions from 32 to 1024. In these comparisons, we observe little variation in nuisance prediction accuracy across encoders and only small gains from larger embeddings. We therefore use Qwen3-Embedding-4B with $d=64$ to limit encoding and storage costs.

\paragraph{Nuisance estimation.}
For each menu $m$ and pair $j<k$ in $m$, we fit a separate LightGBM regressor to estimate $\mu^m_{0,jk}(x)$ by minimizing squared error on records with $S^m=1$. We impose $\widehat\mu^m_{kj}(x)=-\widehat\mu^m_{jk}(x)$ and $\widehat\mu^m_{jj}(x)=0$. For each menu, we fit a LightGBM binary classifier to estimate $e^m_0(x)$ from the inclusion indicator $S^m$ using binary log loss on all training contexts. We cross-fit both nuisance estimators as described above.

\paragraph{Payoff regression.}
We fit a shared LightGBM regressor to the cross-fitted pseudo-outcomes $\Gamma_{jk}(O;\widehat\eta)$ using $(x,j,k)$ as inputs and treating $j$ and $k$ as categorical features. We minimize squared error to estimate the conditional mean of these targets and obtain $\widehat A^{\mathrm{DB}}(x)$. This regression targets the debiased payoff surface rather than directly observed entries of the ground-truth matrix $A(x)$.

\paragraph{Hyperparameter tuning.}
We tune the nuisance estimators and the shared payoff regressor by grid search with three-fold cross-validation. Tables~\ref{tab:nuisance-ht-grid} and~\ref{tab:payoff-ht-grid} give the respective search spaces. We select the number of boosting iterations by early stopping, with a maximum of 1,000 iterations. We distinguish these tuning folds from the $K_{\mathrm{CF}}$ folds used to construct out-of-fold nuisance predictions. Note that we use a slightly more complex search space for the payoff regressor than the nuisance estimators.

\begin{table}[htbp]
\centering
\caption{Hyperparameter search space for the nuisance estimators. The index $\nu\in\{\mu,e\}$ denotes the outcome or propensity model.}
\label{tab:nuisance-ht-grid}
\small
\begin{tabularx}{0.6\linewidth}{@{}Xll@{}}
\toprule
Hyperparameter & Symbol & Search space \\
\midrule
Learning rate & $\mathrm{lr}_{\nu}$ & $\{0.01,0.1\}$ \\
Maximum leaves per tree & $n_{\mathrm{leaf},\nu}$ & $\{7,15,31\}$ \\
Maximum tree depth & $d_{\max,\nu}$ & \{7, 10\} \\
Minimum observations per leaf & $n_{\min,\nu}$ & $\{30,50,100,200\}$ \\
$L_2$ penalty on leaf values & $\lambda_{2,\nu}$ & $\{0,0.1,1\}$ \\
Feature fraction per tree & $f_{\nu}$ & $\{0.6,0.8,1.0\}$ \\
Number of trees & $T_{\nu}$ & \{50, 100, 200\} \\
\bottomrule
\end{tabularx}
\end{table}

\begin{table}[htbp]
\centering
\caption{Hyperparameter search space for the shared payoff regressor.}
\label{tab:payoff-ht-grid}
\small
\begin{tabularx}{0.6\linewidth}{@{}Xll@{}}
\toprule
Hyperparameter & Symbol & Search space \\
\midrule
Learning rate & $\mathrm{lr}_{A}$ & $\{0.02,0.05,0.1\}$ \\
Maximum leaves per tree & $n_{\mathrm{leaf},A}$ & $\{15,31,62\}$ \\
Maximum tree depth & $d_{\max,A}$ & \{10, 15\} \\
Minimum observations per leaf & $n_{\min,A}$ & $\{50,100,200\}$ \\
$L_2$ penalty on leaf values & $\lambda_{2,A}$ & $\{0,0.1,1\}$ \\
Feature fraction per tree & $f_A$ & $\{0.6,0.8,1.0\}$ \\
Number of trees & $T_A$ & \{50, 100, 200, 400\} \\
\bottomrule
\end{tabularx}
\end{table}

\textbf{Equilibrium learning.} We parameterize the equilibrium predictor $\pi_\theta(x)$ as a multilayer perceptron with two hidden layers, ReLU activations, and a $K$-dimensional softmax output that maps each context to a distribution over agents (During inference or testing, the threshold for a winner is $0.001$.). We train the network to minimize the empirical orthogonal gap loss and approximate the target equilibrium $\pi^\star(x)$. We tune the hyperparameters of the predictor by grid search specified in Table~\ref{tab:gap-loss-ht-grid}. 

\begin{table}[htbp]
\centering
\caption{Hyperparameter search space for the equilibrium predictor. We fix two hidden layers and vary their common width.}
\label{tab:gap-loss-ht-grid}
\small
\begin{tabularx}{0.6\linewidth}{@{}Xl@{}}
\toprule
Hyperparameter & Search space or fixed setting \\
\midrule
Hidden-layer widths & $\{(64,64),(128,128),(256,256)\}$ \\
Learning rate & $\{10^{-4},3\times10^{-4},10^{-3}\}$ \\
Weight decay & $\{0,10^{-5},10^{-4}\}$ \\
Batch size & $\{128,256\}$ \\
Maximum training epochs & $500$ \\
Early-stopping patience & $30$ epochs \\
\bottomrule
\end{tabularx}
\end{table}

\subsection{Baselines}
\label{subsec:baselines}

\paragraph{Oracle nuisance baseline.}
For NashEval-oracle, we substitute the exact outcome functions $\mu^m_{0,jk}(x)$
and inclusion propensities $e^m_0(x)$ into the orthogonal training target:
\begin{equation}
 \Gamma^{\mathrm{ON}}_{i,jk}=
 \frac{1}{|\mathcal Q_{jk}|}\sum_{m\in\mathcal Q_{jk}}
 \left[
 \mu^m_{0,jk}(X_i)
 +\frac{S_i^m}{e^m_0(X_i)}
   \left\{Z^m_{i,jk}-\mu^m_{0,jk}(X_i)\right\}
 \right].
 \label{eq:syn-oracle-nuisance-target}
\end{equation}
We use the same training log and payoff regression as NashEval-debiased: 300 trees, seven leaves, learning rate $0.05$, and minimum leaf size 100. We clip predictions to $[-1,1]$ and impose antisymmetry and a zero diagonal to obtain $\widehat A^{\mathrm{ON}}(x)$.

\paragraph{NashEval-related policies.}
We compute NashEval-debiased, NashEval-plug-in, and NashEval-oracle from $\widehat A^{\mathrm{DB}}$, $\widehat A^{\mathrm{PI}}$, and $\widehat A^{\mathrm{ON}}$, respectively. In the synthetic experiment, we set all agent costs to zero and use the following regularizer:
\begin{equation}
 \Omega(\pi)=\frac{\rho}{2}
             \left\|\pi-\frac{1}{K}\mathbf 1_K\right\|_2^2,
 \qquad \rho=0.001.
 \label{eq:syn-regularizer}
\end{equation}

\paragraph{BT and Borda policies.}
For an estimated matrix $B(x)$, we fit BT scores separately at each context:
\begin{equation}
\begin{aligned}
 \widehat s^B(x)=\argmin_{s\in\mathbb R^K}\;
 &\frac{1}{\binom K2}\sum_{j<k}
 \left[
  \log\{1+\exp(s_j-s_k)\}
  -\frac{1+B_{jk}(x)}{2}(s_j-s_k)
 \right]
 +\frac{\kappa_{\mathrm{BT}}}{2}\|s\|_2^2,\\
 &\hspace{4cm}\kappa_{\mathrm{BT}}=10^{-4}.
\end{aligned}
\label{eq:syn-bt-projection}
\end{equation}
We use Newton iterations with line search and stop when the maximum absolute score gradient is at most $10^{-9}$. We set $B=\widehat A^{\mathrm{DB}}$ for BT-debiased and $B=\widehat A^{\mathrm{PI}}$ for BT-Plug-in, and select $\argmax_j\widehat s^B_j(x)$. For winner feedback, we use the encoded margins in Eq.~(\ref{eq:syn-winner-verdict}).

BT-Reg uses the same scores as BT-Plug-in to form
\begin{equation}
 \widehat A^{\mathrm{BT}}_{jk}(x)=
 \tanh\!\left(\frac{\widehat s^{\mathrm{PI}}_j(x)
                         -\widehat s^{\mathrm{PI}}_k(x)}{2}\right),
 \label{eq:syn-bt-game}
\end{equation}
and computes its equilibrium with the same policy regularizer and strength $\rho$ as NashEval in Eq.~(\ref{eq:syn-regularizer}).

Borda selects the largest average row margin of
$\widehat A^{\mathrm{DB}}(x)$:
$\argmax_j (K-1)^{-1}\sum_{k\ne j}\widehat A^{\mathrm{DB}}_{jk}(x)$.
With complete, equally weighted pairs and a common $L_2$ score penalty, BT scores and row sums have the same ordering. With identical tie-breaking by agent index, BT-debiased and Borda therefore yield the same policy, which we label \textbf{BT-debiased / Borda}.

\paragraph{BT-cost on RouterBench.}
\label{para:bt-cost-constrained}
We train a cost-constrained adaptation of the contextual BT model of Theorem~1 in \citet{Frick.2025}. For each prompt, we form $\widehat P^{\mathrm{BT}}_{jk}(x)=\sigma(\widehat s_j(x)-\widehat s_k(x))$ and fix the opponent to $q=\mathbf1/K$. Given a budget $C\geq\min_jc_j$, we solve
\begin{equation}
 \widehat\pi_{\mathrm{BT\text{-}cost}}(x;C)
 \in\argmax_{\pi\in\Delta_K:\ c^\top\pi\leq C}
       \pi^\top\widehat P^{\mathrm{BT}}(x)q.
 \label{eq:routerbench-bt-cost}
\end{equation}
We evaluate 31 equally spaced budgets between the minimum and maximum costs, together with every distinct model cost, for 37 budgets in total.

\paragraph{Maximal lottery.}
We compute a context-independent equilibrium from pooled training payoffs $\widehat M_{jk}$, using the regularization and cost treatment specified for each experiment. 
We average the observed training verdicts for each pair across contexts. Let $\mathcal I_{jk}$ index the training comparisons of agents $j$ and $k$, and let $Z_{i,jk}^{m_i}$ denote comparison $i$'s signed verdict, oriented in favor of $j$. We define
\begin{equation}
 \widehat M_{jk}
 =\frac{1}{|\mathcal I_{jk}|}\sum_{i\in\mathcal I_{jk}}Z_{i,jk}^{m_i},
 \quad j<k,\qquad
 \widehat M_{kj}=-\widehat M_{jk},\qquad
 \widehat M_{jj}=0.
 \label{eq:pooled-payoff}
\end{equation}

\paragraph{Robust maximal lottery.} Following \citet{Khalaf.2026}, we define
\begin{equation}
  \mathcal W(\widehat w_0,\tau)
  =\left\{w\in\Delta_{|\mathcal G|}:
  \frac{1}{2}\lVert w-\widehat w_0\rVert_1\le\tau\right\},
  \qquad
  \mathcal M_\tau
  =\left\{\sum_{g\in\mathcal G}w_g\widehat M^{(g)}:
  w\in\mathcal W(\widehat w_0,\tau)\right\}
  \label{eq:rq3b-robust-ambiguity}
\end{equation}
as the ambiguity set and choose the primary radius

\begin{equation}
  \tau_n=\min\!\left\{1,
  \sqrt{\frac{|\mathcal G|}{n}}+
  \sqrt{\frac{2}{n}\log\frac{2}{\delta}}\right\},
  \qquad \delta=0.05,
  \label{eq:rq3b-robust-radius}
\end{equation}
where $n$ is the number of training prompts. We also evaluate $\tau\in\{0,.025,.05,.1,.2,.4,.6,.8,1\}$, together with $\tau_n$, without test-based radius selection.

The robust policy is the solution of
\begin{equation}
  \widehat\pi_{\mathrm{rob}}
  \in \argmax_{\pi\in\Delta_K}\;
  \min_{w\in\mathcal W(\widehat w_0,\tau)}\min_{j\in[K]}
  \pi^\top\!\left(\sum_{g\in\mathcal G}w_g\widehat M^{(g)}\right)e_j.
  \label{eq:rq3b-robust-lottery}
\end{equation}
We then solve Eq.~(\ref{eq:rq3b-robust-lottery}) with the exact dual linear program from Theorem~17 of \citet{Khalaf.2026}.

\paragraph{Pluralistic leaderboard.}
We adapt the ranking method of \citet{Haghtalab.2026} to incomplete comparisons by treating each training prompt as one voter. We infer its scores from the 20 observed pairs $\mathcal O_x$ through
\begin{equation}
  \widehat u_x
  =\argmin_{u\in\mathbb R^K}
  \sum_{(j,k)\in\mathcal O_x}
  \bigl(u_j-u_k-Z_{jk}(x)\bigr)^2+\lVert u\rVert_2^2.
  \label{eq:rq3b-rank-completion}
\end{equation}
We sort $\widehat u_x$ in descending order and break ties with a permutation fixed by the data-split seed. This completion step is our adapter, so the ranking guarantees concern the inferred rankings rather than the original cyclic preferences.

We implement the iterated-rounding committee procedure and the geometric-checkpoint ranking construction of Algorithms~1 and~2 in \citet{Haghtalab.2026}. We use the paper's parameters $\epsilon=0.01$, $\alpha=1/2+4\epsilon$, and $\beta=1/4+2\epsilon$, and use growth factor two for the checkpoint sizes. Since our candidate set has only $K=10$ models, we enumerate the committees at each checkpoint and replace the paper's approximate multiplicative-weights subroutine with an exact empirical linear program. For candidate committees $C_1,\ldots,C_L$ of size $k$, let
\begin{equation}
  b_{a\ell}=\frac{1}{n}\sum_{x=1}^n
  \1\!\left\{a\succ_x c\ \text{for every }c\in C_\ell\right\}.
\end{equation}
We compute the stable lottery over committees by minimizing
\begin{equation}
  \min_{\delta\in\Delta_L,\,v}\;v
  \quad\text{subject to}\quad
  \sum_{\ell=1}^L\delta_\ell b_{a\ell}\le v
  \quad\forall a\in[K].
  \label{eq:rq3b-stable-committee-lp}
\end{equation}
We use $\delta$ to select committees in order to produce the ranking. At each rounding step, we select a committee with positive probability under $\delta$ that minimizes the fraction of unsatisfied prompts at threshold $\beta-\epsilon$, and repeat on the remaining unsatisfied prompts. We construct checkpoint committees of sizes $1,2,4,\ldots$ and append each model to the ranking when it first appears, then append any remaining models. Let $W_3$ be the top-three models in the ranking. We convert the ranking to a uniform lottery over its top three models for policy evaluation.

\newpage
\section{Proofs}
\label{app:proofs}
We collect the assumptions in Section~\ref{app:proof-assumptions}. We then prove identification, derive the EIF and the conditional bias bound, establish Neyman-orthogonality, and prove the empirical exploitability bound and quasi-oracle rate.

\paragraph{Notation and conventions.}
\label{para:assumption-notation}
\label{ass:avg-matrix-norm}
We write $P_0$ for the observed-data law and $P_X$ for its context marginal. We fix the finite agent and menu collections, the aggregators $f_{jk}$, and the context-independent regularizer $\Om$. All functions, fitted quantities, policy selections, and suprema below are measurable, and pointwise statements hold outside a common $P_X$-null set. We suppress pair and menu indices when they do not affect the argument, write $\eta=(\mu,e)$ and $\eta_0=(\mu_0,e_0)$, and set $\Delta\mu=\widehat\mu-\mu_0$ and $\Delta e=\widehat e-e_0$ for fitted nuisance errors.

We use $\|\cdot\|_1$ and $\|\cdot\|_2$ for the vector $\ell_1$ and Euclidean norms. For the policy norm $\|\cdot\|$, we define the dual norm and induced matrix norm by
\begin{align*}
    \|v\|_\star:=\sup_{\|u\|\le1}\langle v,u\rangle,
    \qquad
    \|B\|_\star:=\sup_{\|u\|\le1}\|Bu\|_\star.
\end{align*}
Finite-dimensional norm equivalence gives a constant $C_{\mathrm{mat}}$ such that $\|B\|_\star^2\le C_{\mathrm{mat}}\sum_{j,k}|B_{jk}|^2$. For a scalar function $h$, we write $\|h\|_{L_2(P_X)}=(\E_X[h(X)^2])^{1/2}$, so $\|\|\Delta\mu\|_2^2\|_{L_2(P_X)}=(\E_X\|\Delta\mu(X)\|_2^4)^{1/2}$. We write $L_2^0(P_0)$ for the square-integrable, mean-zero functions of $O$. All regularity constants are uniform in $n$, and $\kappa$ is fixed for the quasi-oracle rate.

\subsection{Assumptions}
\label{app:proof-assumptions}

We state the full list of assumptions used in the analysis.

\begin{assumption}[Identifiability]
\label{ass:identifiability}
\label{ass:neyman-identification}
 In order to identify the contextual payoff $A_{jk}(x)$ from the observed data, we make the following assumptions: (i)~\emph{pairwise menu coverage}: $\Qcal_{jk}\ne\emptyset$ for every pair $\{j,k\}\subseteq\M$; (ii)~\emph{missing at random (MAR)}: $S^m\perp Z^m_{jk} \mid X$ for every $m\in\Qcal_{jk}$; (iii)~\emph{positivity}: $e^m_0(x)=P(S^m=1\mid X=x)\ge\elb>0$ almost everywhere for every $m\in\Qcal_{jk}$. 
\end{assumption}
The potential-verdict formulation follows \citet{Rubin.1974}, and MAR~\citep{Rubin.1976} and positivity~\citep{Kennedy.2024review} are standard identifying conditions in missing-data and causal analysis.

\begin{assumption}[Boundedness]
\label{ass:boundedness}
\label{ass:neyman-measurability}
\label{ass:neyman-positivity}
\label{ass:avg-measurability}
\label{ass:avg-positivity}
\label{ass:avg-boundedness}
We fix $\alpha_\mu>0$ and write $I_\alpha:=(-1-\alpha_\mu,1+\alpha_\mu)$. We assume (i) generic and fitted nuisances used in the bias, differentiability, and rate bounds satisfy $\mu^m_{jk}(x)\in I_\alpha$ and $e^m(x)\ge\elb$, with $\elb$ from Assumption~\ref{ass:identifiability}, and (ii) the fitted payoff has finite entrywise $L_2(P_X)$ error. For the quasi-oracle result, we additionally require (iii) a uniform envelope $\sup_{h\in\mathcal H_n}|h(x)|\le H$, where $\mathcal H_n$ is the loss class in Assumption~\ref{ass:complexity} and $H$ is independent of $n$.
\end{assumption}

\begin{assumption}[Smoothness]
\label{ass:smoothness}
\label{ass:gap-differentiation}
\label{ass:neyman-smoothness}
\label{ass:conditional-bias}
\label{ass:avg-smoothness}
We assume (i) each $f_{jk}$ extends to a continuously differentiable function on $I_\alpha^{\Qcal_{jk}}$, with $I_\alpha$ from Assumption~\ref{ass:boundedness}, (ii) $\Om$ is continuously differentiable on a neighborhood of $\DeltaK$, and (iii) $\theta\mapsto\pi_\theta(x)$ is continuously differentiable when we differentiate the gap loss. For population gap-loss differentiation, we further require a neighborhood $U_\theta$ of each $\theta$ such that
\begin{align}
    \E_X[M_\theta(X)]<\infty,
    \qquad M_\theta(x):=\sup_{\vartheta\in U_\theta}
    \|D_\theta\pi_\vartheta(x)\|_{\mathrm{op}},
    \label{eq:policy-jacobian-envelope}
\end{align}
where $\|J\|_{\mathrm{op}}:=\sup_{\|v\|_2\le1}\|Jv\|$. For gap-loss differentiability, orthogonality, and the quantitative remainder and rate bounds, we further assume $\|\nabla f_{jk}(u)-\nabla f_{jk}(v)\|_2\le L_f\|u-v\|_2$ for all $u,v\in I_\alpha^{\Qcal_{jk}}$.
\end{assumption}
Such differentiability conditions are standard in influence-function and orthogonal-learning analyses~\citep{Kennedy.2024review,Foster.2023}. The Jacobian envelope justifies differentiation of the population gap loss under expectation, as we show in Section~\ref{app:proof-gap-loss-orthogonal}. 

\begin{assumption}[Strong convexity]
\label{ass:strong-convexity}
\label{ass:game-geometry}
\label{ass:avg-regularizer}
We assume that $\Om$ is $\kappa$-strongly convex on $\DeltaK$ with respect to $\|\cdot\|$, where $\kappa>0$.
\end{assumption}
Strongly convex regularization is common in regularized zero-sum games~\citep{Sokota.2022}, and it makes the equilibrium unique~\citep{Facchinei.2003}.

\begin{assumption}[Cross-fitting]
\label{ass:cross-fitting}
We observe i.i.d. data $O_1,\ldots,O_n\sim P_0$ and use a data-independent partition into a fixed number $K_{CF}$ of folds. Each nuisance fit uses only observations outside its evaluation fold and satisfies the stated propensity and rate conditions. We use $\ADB(\cdot;\widehat\eta)$ for the average of the fold-specific regression targets.
\end{assumption}
Cross-fitting is standard in DML~\citep{Chernozhukov.2018}.

\begin{assumption}[Nuisance and payoff-estimation rates]
\label{ass:estimation-rates}
\label{ass:empirical-bound}
\label{ass:avg-rates}
We assume $\|\|\Delta\mu\|_2\|_{L_2(P_X)}=O_p(r_\mu)$ and $\|\Delta e\|_{L_2(P_X)}=O_p(r_e)$ for deterministic rates $r_\mu,r_e\to0$, together with the moment bounds
\begin{align}
    \left\|\left\|\Delta\mu\right\|_2^2\right\|_{L_2(P_X)}
    =O_p(r_\mu^2), \qquad
    \left\|\Delta e\,\Delta\mu\right\|_{L_2(P_X)}
    =O_p(r_\mu r_e).
    \label{eq:remainder-component-rates}
\end{align}
The outcome bounds hold uniformly over agent pairs, the propensity bound over contributing menus, and the product bound over pairs and contributing menus, with the product evaluated menu-wise. We define
\begin{align}
    e_A:=\max_{j,k}\big\|\AhatDB_{jk}(\cdot;\widehat\eta)
          -\ADB_{jk}(\cdot;\widehat\eta)\big\|_{L_2(P_X)}.
    \label{eq:payoff-regression-rate}
\end{align}
For the quasi-oracle rate, we additionally require
\begin{align}
    e_A=O_p(n^{-1/4}),\qquad
    r_\mu\lesssim n^{-1/8},\qquad
    \elb^{-1}r_\mu r_e\lesssim n^{-1/4},
    \label{eq:quasi-oracle-conditions}
\end{align}
and omit the separate condition on $r_\mu$ for affine aggregators.
\end{assumption}
Product-rate conditions of this type are standard in DML~\citep{Chernozhukov.2018} and orthogonal statistical learning~\citep{Foster.2023}. For fixed $\elb>0$, we can satisfy Eq.~(\ref{eq:quasi-oracle-conditions}) with matching nuisance rates $r_\mu,r_e\lesssim n^{-1/8}$ and payoff-regression error $e_A=O_p(n^{-1/4})$, using flexible nuisance learners such as gradient-boosted trees (e.g., LightGBM) or neural networks~\citep{Chernozhukov.2018} when they attain these rates in the stated norms.

\begin{assumption}[Realizability]
\label{ass:realizability}
\label{ass:oracle-realizability}
For oracle recovery in Proposition~\ref{prop:oracle-recovery}, we assume that some $\theta_0\in\Theta$ satisfies $\pi_{\theta_0}(x)=\pi^\star(x)$ for $P_X$-almost every $x$.
\end{assumption}

\begin{assumption}[Loss-class complexity]
\label{ass:complexity}
\label{ass:quasi-oracle}
\label{ass:oracle-generalization}
For the quasi-oracle result, we assume a deterministic class $\mathcal B_n$ containing $A$ and, with probability tending to one, $\AhatDB(\cdot;\widehat\eta)$. We define
\begin{align*}
    \mathcal H_n:=\{x\mapsto G_\Om(\pi_\theta(x);B(x),x):
                     \theta\in\Theta,\ B\in\mathcal B_n\}
\end{align*}
and assume the expected absolute Rademacher-complexity bound
\begin{align}
    \mathfrak R_n(\mathcal H_n)
    :=\E_{X,\sigma}\!\left[\sup_{h\in\mathcal H_n}
       \left|\frac1n\sum_{i=1}^n\sigma_i h(X_i)\right|\right]
    \le C n^{-1/2},
    \label{eq:oracle-complexity}
\end{align}
where $C$ is independent of $n$, and $\sigma_1,\ldots,\sigma_n$ are independent uniform signs in $\{-1,1\}$, independent of the contexts.
\end{assumption}
Rademacher-complexity bounds are commonly used to control uniform deviations~\citep{Bartlett.2002}. Because the bound covers the entire class $\mathcal H_n$, it permits reuse of observations for payoff and policy fitting.

\subsection{Identification of the contextual payoff matrix}
\label{app:proof-identification}
\begin{proof}
We fix a pair $j\ne k$, a menu $m\in\Qcal_{jk}$, and a context $x$ outside the null set on which the identification assumptions may fail. Menu-specific positivity gives $P(S^m=1\mid X=x)=e^m_0(x)\geq\elb>0$, so the observed-data regression $\mu^m_{0,jk}(x)$ is well-defined. Menu-specific MAR implies conditional mean independence between $S^m$ and $Z^m_{jk}$ given $X$. Therefore,
\begin{align*}
    \mu^m_{0,jk}(x)
    &=\E[Z^m_{jk}\mid X=x,S^m=1]\\
    &=\E[Z^m_{jk}\mid X=x]\\
    &=g^m_{jk}(x).
\end{align*}
This argument only requires the conditional mean equality $\E[Z^m_{jk}\mid X,S^m]=\E[Z^m_{jk}\mid X]$ almost surely. It does not require joint independence across menus. Thus, our method handles the case where only one menu is observed per context (i.e., $\sum_{m\in\Qcal}S^m=1$).

The equality holds for every $m\in\Qcal_{jk}$. Hence, we identify each off-diagonal payoff entry by substituting the observed-data regressions into Eq.~(\ref{eq:payoff}):
\begin{equation*}
    A_{jk}(x)
    =f_{jk}\!\left(
        \big\{\mu^m_{0,jk}(x):m\in\Qcal_{jk}\big\}
      \right).
\end{equation*}
Pairwise menu coverage ensures that $\Qcal_{jk}$ is nonempty, and menu-specific positivity applies to every menu in this set. Definition~\ref{def:payoff} fixes $A_{jj}(x)=0$, so the observed-data law identifies the entire contextual payoff matrix $A(x)$. Under the regularizer conditions in Assumptions~\ref{ass:smoothness}(ii) and~\ref{ass:strong-convexity}, $\Om$ is fixed and the equilibrium is unique, so Eq.~(\ref{eq:contextual-equilibrium}) also identifies $\pi^\star(x)$.
\end{proof}

\subsection{Efficient influence function of the average payoff}
\label{app:proof-eif}

\begin{lemma}[Efficient influence function of the average payoff]
\label{lem:payoff-eif}
Under Assumptions~\ref{ass:identifiability}--\ref{ass:smoothness}, fix $j\ne k$ and define the average payoff $\psi_{jk}:=\E_0[A_{jk}(X)]$. In the nonparametric observed-data model, we obtain the efficient influence function of $\psi_{jk}$ at $P_0$ by centering the pseudo-outcome in Eq.~(\ref{eq:payoff-pseudo-outcome}) at the true nuisance $\eta_0=(\mu_0,e_0)$:
\begin{equation}
    \varphi_{jk}(O;P_0)
    =\Gamma_{jk}(O;\eta_0)-\psi_{jk}.
    \label{eq:payoff-eif-centered}
\end{equation}
In particular, $\E_0[\Gamma_{jk}(O;\eta_0)]=\psi_{jk}$ and $\E_0[\varphi_{jk}(O;P_0)]=0$.
\end{lemma}

\begin{proof}[Proof of Lemma~\ref{lem:payoff-eif}]
We consider a regular one-dimensional parametric submodel $\{P_t:|t|<\delta\}$ of the observed-data law, with $P_{t=0}=P_0$. Let the densities be $p_t$. We write $O=(X,\mathbf S,\widetilde{\mathbf Y})$, where $\mathbf S=(S^m)_{m\in\Qcal}$ and $\widetilde{\mathbf Y}=(S^mY_m)_{m\in\Qcal}$. We factor the joint density as
\begin{align*}
    p_t(o)=p_t(x)\,p_t(\mathbf s\mid x)\,
    p_t(\widetilde{\mathbf y}\mid x,\mathbf s).
\end{align*}
We define the score as the derivative of the log density at $t=0$:
\begin{align*}
    s(O)&:=\left.\frac{\mathrm d}{\mathrm dt}\log p_t(O)\right|_{t=0}\\
    &=s_X(X)+s_{\mathbf S\mid X}(\mathbf S\mid X)
      +s_{\widetilde{\mathbf Y}\mid X,\mathbf S}
        (\widetilde{\mathbf Y}\mid X,\mathbf S),
\end{align*}
where each component is the log derivative of the corresponding density factor. The score has finite second moment, and its components satisfy
\begin{align*}
    \E_0&[s_X(X)]=0,\\
    \E_0&[s_{\mathbf S\mid X}(\mathbf S\mid X)\mid X]=0,\\
    \E_0&[s_{\widetilde{\mathbf Y}\mid X,\mathbf S}
        (\widetilde{\mathbf Y}\mid X,\mathbf S)\mid X,\mathbf S]=0.
\end{align*}
Thus, $\E_0[s(O)]=0$, $s_X(X)=\E_0[s(O)\mid X]$, and $s-s_X$ is the conditional score of $O$ given $X$.

We fix $j\ne k$ and abbreviate $f=f_{jk}$. Under $P_t$, we define $\mu^m_{t,jk}(x)=\E_{P_t}[Z^m_{jk}\mid X=x,S^m=1]$ and collect these regressions in $\mu_t(x)=(\mu^m_{t,jk}(x))_{m\in\Qcal_{jk}}$. We write $\psi_t=\E_{P_t}[f(\mu_t(X))]$ and $\psi_0=\psi_{jk}$. The product and chain rules give
\begin{equation}
\begin{aligned}
    \dot\psi_0
    &:=\left.\frac{\mathrm d}{\mathrm dt}\psi_t\right|_{t=0}\\
    &=\E_0\!\left[\{f(\mu_0(X))-\psi_0\}s_X(X)\right]
      +\sum_{m\in\Qcal_{jk}}\E_0\!\left[
        \partial_m f(\mu_0(X))\dot\mu^m_0(X)\right],
\end{aligned}
\label{eq:eif-payoff-pathwise}
\end{equation}
where $\partial_m f=\partial f/\partial u_m$ and $\dot\mu^m_0=\left.\mathrm d\mu^m_{t,jk}/\mathrm dt\right|_{t=0}$.

In order to get $\dot\mu^m_0$, we compute the regression derivative using the observed-data ratio
\begin{align*}
    \mu^m_{t,jk}(x)
    =\frac{\E_{P_t}[S^mZ^m_{jk}\mid X=x]}
           {\E_{P_t}[S^m\mid X=x]}.
\end{align*}
The quotient rule and the conditional score identity yield
\begin{equation}
\begin{aligned}
    \dot\mu^m_0(x)
    &=\frac{1}{e^m_0(x)}\E_0\!\left[
       S^m\{Z^m_{jk}-\mu^m_{0,jk}(x)\}
       \{s(O)-s_X(x)\}\mid X=x\right]\\
    &=\E_0\!\left[
       \frac{S^m}{e^m_0(x)}\{Z^m_{jk}-\mu^m_{0,jk}(x)\}
       s(O)\mid X=x\right].
\end{aligned}
\label{eq:eif-regression-derivative}
\end{equation}
For the second equality, we use $\E_0[S^m\{Z^m_{jk}-\mu^m_{0,jk}(x)\}\mid X=x]=0$. Then, we substitute Eq.~(\ref{eq:eif-regression-derivative}) into Eq.~(\ref{eq:eif-payoff-pathwise}) and apply iterated expectations. Since $f(\mu_0(X))$ is a function of $X$, we obtain
\begin{equation}
\begin{aligned}
    \dot\psi_0
    &=\E_0\!\left[\{f(\mu_0(X))-\psi_0\}s(O)\right]\\
    &\quad+\sum_{m\in\Qcal_{jk}}\E_0\!\left[\partial_m f(\mu_0(X))
       \frac{S^m}{e^m_0(X)}
       \big\{Z^m_{jk}-\mu^m_{0,jk}(X)\big\}s(O)\right]\\
    &=\E_0\!\left[\{\Gamma_{jk}(O;\eta_0)-\psi_{jk}\}s(O)\right].
\end{aligned}
\label{eq:eif-payoff-score-representation}
\end{equation}
We now verify the EIF definition in Section~\ref{subsec:efficient-influence-function} for $\varphi_{jk}(O;P_0):=\Gamma_{jk}(O;\eta_0)-\psi_{jk}$. Each residual correction has conditional mean zero given $X$, and $\E_0[f(\mu_0(X))]=\psi_{jk}$, so $\E_0[\varphi_{jk}]=0$. Bounded verdicts, bounded derivatives, and positivity give $\E_0[\varphi_{jk}^2]<\infty$. Thus, $\varphi_{jk}$ belongs to $L_2^0(P_0)$, the space of square-integrable mean-zero functions, which is the closed tangent space of the nonparametric observed-data model. Since the submodel was arbitrary, Eq.~(\ref{eq:eif-payoff-score-representation}) represents every pathwise derivative exactly as required by Eq.~(\ref{eq:eif-pathwise-definition}), with $\psi=\psi_{jk}$ and $\varphi=\varphi_{jk}$. These properties establish that $\varphi_{jk}$ is the unique canonical gradient and hence the EIF.

At a generic law $P$, we obtain $\varphi_{jk}(O;P)=\Gamma_{jk}(O;\eta_P)-\psi_{jk}(P)$, as stated in Eq.~(\ref{eq:payoff-eif-background}). We add back $\psi_{jk}(P)$ and replace $\eta_P$ by a generic nuisance $\eta$ to obtain the uncentered pseudo-outcome in Eq.~(\ref{eq:payoff-pseudo-outcome}).
\end{proof}

\subsection{Neyman-orthogonality of $\Gamma_{jk}$}
\label{app:proof-neyman}

We prove conditional unbiasedness and Neyman-orthogonality under Assumptions~\ref{ass:identifiability}--\ref{ass:smoothness}.

\begin{proof}
We fix $j\ne k$, a direction $h$, and a context $x$ outside the null sets in the assumptions and the conditional expectation identities below. We write $\eta_t=\eta_0+th$ for $|t|\le t_0$. Since $\mu^m_{0,jk}(x)\in[-1,1]$ and $e^m_0(x)\ge\elb$, we can choose $t_0=t_0(x,h)>0$ small enough that $\mu^m_t(x)\in I_\alpha$ and $e^m_t(x)>0$ for every $m\in\Qcal_{jk}$. We write $f=f_{jk}$ and collect the outcome components into $u_t(x)=\{\mu_t^m(x)\}_{m\in\Qcal_{jk}}$, so $u_t(x)=u_0(x)+t h_\mu(x)$. Assumption~\ref{ass:identifiability} gives $g^m_{jk}(x)=\mu^m_{0,jk}(x)$ and hence $A_{jk}(x)=f(u_0(x))$. Define
\begin{equation}
G(t)\ :=\ \E_0\big[\Gamma_{jk}(O;\eta_t)\ \big|\ X=x\big],
\qquad |t|\le t_0 ,
\label{eq:Gdef}
\end{equation}
where the expectation remains under the fixed law $P_0$ as the nuisance argument varies. We will show $G(0)=A_{jk}(x)$ and $G'(0)=0$.

Bounded verdicts, continuity of the aggregator gradient, and positive propensities give conditional integrability along the path. Conditional on $X=x$, the nuisance factors in Eq.~(\ref{eq:payoff-pseudo-outcome}) are deterministic, so
\begin{equation}
G(t)= f\big(u_t(x)\big)
+\sum_{m\in\Qcal_{jk}}
\frac{\partial f}{\partial u_m}\big(u_t(x)\big)\,
\frac{1}{e^m_t(x)}\,
\underbrace{\E_0\Big[S^m\big\{Z^m_{jk}-\mu^m_t(x)\big\}\ \Big|\ X=x\Big]}_{=:\,T_m(t)} .
\label{eq:Gstep1}
\end{equation}
The selected residual has an explicit conditional mean. By the definition of the observed-data regression $\mu^m_{0,jk}$,
\begin{equation}
\E_0[S^m Z^m_{jk}\mid X=x]
=e^m_0(x)\E_0[Z^m_{jk}\mid X=x,S^m=1]
=e^m_0(x)\mu^m_{0,jk}(x).
\label{eq:ignorability-use}
\end{equation}
This identity does not require MAR. We use Assumption~\ref{ass:identifiability} only to identify the regression with the full-data margin. Since $\E_0[S^m\mu^m_t(x)\mid X=x]=e^m_0(x)\mu^m_t(x)$, subtracting gives
\begin{equation}
T_m(t)=e^m_0(x)\big\{\mu^m_{0,jk}(x)-\mu^m_t(x)\big\}
= -\,t\,e^m_0(x)\,h_{\mu}^m(x),
\label{eq:Tq}
\end{equation}
where the last equality uses $\mu_t=\mu_0+t h_\mu$. Substituting this expression into Eq.~(\ref{eq:Gstep1}), we obtain
\begin{equation}
G(t)= f\big(u_t(x)\big)\ -\ t\,B(t),
\qquad
B(t):=\sum_{m\in\Qcal_{jk}}
\frac{\partial f}{\partial u_m}\big(u_t(x)\big)\,
\frac{e^m_0(x)}{e^m_t(x)}\,h_{\mu}^m(x).
\label{eq:Gclosed}
\end{equation}
We use this expression as a version of the conditional expectation for all $|t|\le t_0$. At $t=0$ it gives $G(0)=f(u_0(x))=A_{jk}(x)$, proving conditional unbiasedness. Continuity of the gradient and positivity of the denominators imply that $B(t)$ is continuous at zero.

We can now differentiate the explicit expression for $G$, without interchanging a derivative and an expectation. For $t\ne0$, \eqref{eq:Gclosed} gives
\begin{equation}
\frac{G(t)-G(0)}{t}
=\frac{f\big(u_t(x)\big)-f\big(u_0(x)\big)}{t}-B(t).
\label{eq:Gprime}
\end{equation}
By Assumption~\ref{ass:smoothness}, the chain rule along the affine path $u_t(x)=u_0(x)+t h_\mu(x)$ gives
\begin{equation}
\left.\frac{\mathrm d}{\mathrm dt}f\big(u_t(x)\big)\right|_{t=0}
=\sum_{m\in\Qcal_{jk}}\frac{\partial f}{\partial u_m}\big(u_0(x)\big)\,h_{\mu}^m(x)
=\big\langle \nabla f\big(u_0(x)\big),\ h_{\mu}(x)\big\rangle .
\label{eq:chain}
\end{equation}
Since $e^m_t(x)=e^m_0(x)$ at $t=0$, the limit of the correction term is
\begin{equation}
B(0)=\sum_{m\in\Qcal_{jk}}\frac{\partial f}{\partial u_m}\big(u_0(x)\big)\,h_{\mu}^m(x)
=\big\langle \nabla f\big(u_0(x)\big),\ h_{\mu}(x)\big\rangle .
\label{eq:B0}
\end{equation}
The limits in \eqref{eq:chain} and \eqref{eq:B0} cancel in \eqref{eq:Gprime}, giving
\begin{align*}
G'(0)
=\big\langle \nabla f(u_0(x)),h_{\mu}(x)\big\rangle
-\big\langle \nabla f(u_0(x)),h_{\mu}(x)\big\rangle
=0 .
\end{align*}
Thus, for every direction $h$, Neyman-orthogonality holds for $P_X$-almost every $x$.
\end{proof}

\subsection{Exact conditional bias and second-order remainder}
\label{app:proof-bias}

We derive the exact conditional bias and prove Eq.~(\ref{eq:reminder}) under the identification, boundedness, and smoothness conditions in Assumptions~\ref{ass:identifiability}--\ref{ass:smoothness}.

\begin{lemma}[Exact conditional bias]
\label{lem:exact-bias}
Under these conditions, we fix $j\ne k$, write $f=f_{jk}$ and $\mu(x)=(\mu^m_{jk}(x))_{m\in\Qcal_{jk}}$, and consider a generic nuisance $\eta=(\mu,e)$ with $\mu^m_{jk}(x)\in I_\alpha$ and $e^m(x)>0$. We define $\Delta\mu=\mu-\mu_0$ and $\Delta e^m=e^m-e^m_0$. Then, for almost every $x$,
\begin{equation}
\E_0\big[\Gamma_{jk}(O;\eta)\mid X=x\big]-A_{jk}(x)
= -R_f(x)
+\sum_{m\in\Qcal_{jk}}\frac{\partial f}{\partial u_m}\big(\mu(x)\big)\,
\frac{\Delta e^m(x)}{e^m(x)}\,\Delta \mu^m(x),
\label{eq:exact-bias}
\end{equation}
where $R_f(x)=f(\mu_0(x))-f(\mu(x))-\big\langle\nabla f(\mu(x)),\,\mu_0(x)-\mu(x)\big\rangle$ is the first-order Taylor remainder of $f$ at $\mu(x)$. Consequently, Eq.~(\ref{eq:reminder}) holds.
\end{lemma}

\begin{proof}
We fix $x$ and suppress it below. The selected residual has conditional mean $\E_0[S^m(Z^m_{jk}-\mu^m_{jk})\mid X=x]=-e^m_0\Delta\mu^m$. We substitute this identity into Eq.~(\ref{eq:payoff-pseudo-outcome}) and use $A_{jk}(x)=f(\mu_0)$ to obtain
\begin{align*}
    \ADB_{jk}(x;\eta)-A_{jk}(x)
    &=f(\mu)-f(\mu_0)
      -\sum_{m\in\Qcal_{jk}}\partial_m f(\mu)
       \frac{e^m_0}{e^m}\Delta\mu^m\\
    &=\underbrace{f(\mu)-f(\mu_0)-\langle\nabla f(\mu),\Delta\mu\rangle}_{-R_f(x)}
      +\sum_{m\in\Qcal_{jk}}\partial_m f(\mu)
       \left(1-\frac{e^m_0}{e^m}\right)\Delta\mu^m.
\end{align*}
The identity $1-e^m_0/e^m=\Delta e^m/e^m$ gives Eq.~(\ref{eq:exact-bias}). Since $\mu_0(x)\in[-1,1]^{\Qcal_{jk}}$ and $I_\alpha^{\Qcal_{jk}}$ is convex, the segment between $\mu(x)$ and $\mu_0(x)$ lies in $I_\alpha^{\Qcal_{jk}}$. The Lipschitz gradient bound therefore gives $|R_f(x)|\le (L_f/2)\|\Delta\mu(x)\|_2^2$, and the triangle inequality proves Eq.~(\ref{eq:reminder}). The Lipschitz gradient on the bounded cube $I_\alpha^{\Qcal_{jk}}$ also gives $G_f:=\max_{j\ne k}\sup_{u\in I_\alpha^{\Qcal_{jk}}}\|\nabla f_{jk}(u)\|_1<\infty$. If $e^m(x)\ge\elb$, as in Assumption~\ref{ass:boundedness}, then
\begin{equation}
\Big|\E_0\big[\Gamma_{jk}(O;\eta)\mid X=x\big]-A_{jk}(x)\Big|
\ \le\ \tfrac{L_f}{2}\big\|\Delta \mu(x)\big\|_2^2
+\frac{G_f}{\elb}\max_{m\in\Qcal_{jk}}
\big|\Delta e^m(x)\Delta \mu^m(x)\big| .
\label{eq:second-order}
\end{equation}
\end{proof}

Under the remainder-component conditions in Eq.~(\ref{eq:remainder-component-rates}), this bound yields conditional bias of order $O_p(r_\mu^2+r_\mu r_e)$ in $L_2(P_X)$. For an affine aggregator, $R_f\equiv0$, and Eq.~(\ref{eq:exact-bias}) vanishes if either $\mu=\mu_0$ or $e=e_0$. For a nonlinear aggregator, correct propensities alone leave the outcome-curvature term $-R_f(x)$.

\subsection{Proof of Proposition~\ref{prop:oracle-recovery}}
\label{app:proof-oracle-recovery}

\begin{proof}
Theorem~\ref{thm:DPME} gives $\ADB(x;\eta_0)=A(x)$, and Proposition~\ref{prop:gap-loss-properties} gives $G_\Om(\pi;A(x),x)\ge0$, with equality exactly when $\pi=\pi^\star(x)$. Under Assumption~\ref{ass:realizability}, we choose $\theta_0$ with $\pi_{\theta_0}(X)=\pi^\star(X)$ almost surely, so $\Lcal(\theta_0;\eta_0)=0$ is the minimum risk. Every minimizer therefore has zero expected nonnegative gap, which implies $\pi_\theta(X)=\pi^\star(X)$ almost surely. Conversely, every such policy has zero risk and attains the minimum.
\end{proof}

\subsection{Proof of Theorem~\ref{thm:gap-loss-orthogonal}}
\label{app:proof-gap-loss-orthogonal}

We first prove differentiability under Assumptions~\ref{ass:boundedness}--\ref{ass:strong-convexity}. We then use identification and the payoff remainder to prove universal Neyman-orthogonality.

\begin{proof}

\noindent\textbf{Differentiability.}
We fix a nuisance $\eta$ satisfying Assumption~\ref{ass:boundedness} and write $B(x)=\ADB(x;\eta)$. The Lipschitz aggregator gradients on the bounded domain $I_\alpha^{\Qcal_{jk}}$ and the propensity lower bound give a uniform bound on $B(x)$ through Eq.~(\ref{eq:payoff-pseudo-outcome}). For any matrix $B$, strong convexity gives the unique maximizer $q_B(p):=\argmin_{q\in\DeltaK}\{\Om(q)-\langle Bp,q\rangle\}$ in the variational gap. Danskin's theorem gives
\begin{align}
    g_B(p):=\nabla_p G_\Om(p;B,x)
    =\nabla\Om(p)-(B+B^\top)p+B^\top q_B(p).
    \label{eq:gap-general-policy-gradient}
\end{align}
This formula holds without skew-symmetry of $B$ and reduces to Eq.~(\ref{eq:gap-gradient}) when $B$ is skew-symmetric. We fix a direction $h_\theta$ and define $\ell_t(x):=G_\Om(\pi_{\theta+t h_\theta}(x);B(x),x)$. The chain rule gives
\begin{align*}
    \frac{\mathrm d}{\mathrm dt}\ell_t(x)
    =\left\langle g_{B(x)}(\pi_{\theta+t h_\theta}(x)),
      D_\theta\pi_{\theta+t h_\theta}(x)[h_\theta]\right\rangle.
\end{align*}
The uniform bound on $B(x)$ and continuity of $\nabla\Om$ on the compact simplex give $\|g_{B(x)}(p)\|_*\le C_g$ uniformly in $x$ and $p\in\DeltaK$. For sufficiently small $|t|$, the segment from $\theta$ to $\theta+t h_\theta$ lies in $U_\theta$. The mean value theorem and duality therefore give
\begin{align*}
    \left|\frac{\ell_t(X)-\ell_0(X)}{t}\right|
    \le C_g M_\theta(X)\|h_\theta\|_2.
\end{align*}
Eq.~(\ref{eq:policy-jacobian-envelope}) makes the bound integrable, so dominated convergence justifies differentiation under expectation. The chain rule at $t=0$ then gives
\begin{align}
    D_\theta\Lcal(\theta;\eta)[h_\theta]
    &=\E_X\!\left[
      \left\langle g_{B(X)}(\pi_\theta(X)),
      D_\theta\pi_\theta(X)[h_\theta]\right\rangle\right].
    \label{eq:gap-population-gradient}
\end{align}
Therefore, the loss is differentiable with respect to $\theta$.

\noindent\textbf{Neyman-orthogonality.}
We fix $\theta$, a policy direction $h_\theta$, and a bounded nuisance direction $h_\eta$. We write $B_t(x)=\ADB(x;\eta_0+t h_\eta)$, so $B_0(x)=A(x)$. Positivity and boundedness of $h_\eta$ give $e_t^m(x)\ge\elb/2$ for sufficiently small $|t|$. Since $\mu^m_{0,jk}(x)\in[-1,1]$, boundedness of $h_\eta$ also gives $\mu^m_{0,jk}(x)+th^m_\mu(x)\in I_\alpha$ when $|t|\|h_\mu\|_\infty<\alpha_\mu$. Eq.~(\ref{eq:reminder}), the Lipschitz aggregator gradients, and the finite pair and menu collections therefore give
\begin{align}
    \mathop{\mathrm{ess\,sup}}_x\|B_t(x)-A(x)\|_\star
    \le C_h t^2,
    \label{eq:gap-path-remainder}
\end{align}
where $C_h$ is independent of $x$ and $t$.

We next bound the change in the policy gradient induced by a payoff perturbation. We fix $x$, suppress it below, and use the maximizer $q_B(p)$ defined above. We add the first-order optimality inequalities for $q_B(p)$ and $q_A(p)$ and use strong convexity to obtain
\begin{align*}
    \kappa\|q_B(p)-q_A(p)\|^2
    &\le\langle(B-A)p,q_B(p)-q_A(p)\rangle,\\
    \|q_B(p)-q_A(p)\|
    &\le\kappa^{-1}\|B-A\|_\star.
\end{align*}
The second inequality uses duality and the norm normalization in Section~\ref{sec:problem}. We apply Eq.~(\ref{eq:gap-general-policy-gradient}) and $\|B^\top\|_\star=\|B\|_\star$ to obtain
\begin{align}
    \|g_B(p)-g_A(p)\|_*
    &\le 3\|B-A\|_\star
      +\|A\|_\star\|q_B(p)-q_A(p)\|\nonumber\\
    &\le (3+\kappa^{-1}\|A\|_\star)\|B-A\|_\star.
    \label{eq:gap-policy-gradient-stability}
\end{align}
The true payoff is uniformly bounded, so $C_A:=3+\kappa^{-1}\mathop{\mathrm{ess\,sup}}_x\|A(x)\|_\star<\infty$ bounds the factor on the right uniformly in $x$ and $p$.

We now combine the payoff remainder with the policy-gradient bound. Along the path $\eta_0+t h_\eta$, the same domination argument applies with propensity lower bound $\elb/2$. We apply Eq.~(\ref{eq:gap-population-gradient}) and the bounds in Eqs.~(\ref{eq:gap-path-remainder}) and~(\ref{eq:gap-policy-gradient-stability}) to obtain
\begin{align}
    &\left|D_\theta\Lcal(\theta;\eta_0+t h_\eta)[h_\theta]
       -D_\theta\Lcal(\theta;\eta_0)[h_\theta]\right|\nonumber\\
    &\qquad\le C_A C_h t^2\|h_\theta\|_2\E_X[M_\theta(X)].
    \label{eq:gap-policy-derivative-remainder}
\end{align}
We divide by $|t|$ and let $t\to0$. Hence,
\begin{align*}
    D_\eta D_\theta\Lcal(\theta;\eta_0)[h_\theta,h_\eta]
    =\lim_{t\to0}
    \frac{D_\theta\Lcal(\theta;\eta_0+t h_\eta)[h_\theta]
          -D_\theta\Lcal(\theta;\eta_0)[h_\theta]}{t}
    =0.
\end{align*}
Since $\theta$ was arbitrary, this proves the universal Neyman-orthogonality in Eq.~(\ref{eq:gap-orthogonality}).
\end{proof}

\subsection{Empirical error bound}
\label{app:proof-avg-exploit}
\label{app:empirical-error-bound}

We first establish the deterministic oracle comparison underlying Theorem~\ref{thm:avg-exploit} using boundedness, smoothness, and strong convexity in Assumptions~\ref{ass:boundedness}--\ref{ass:strong-convexity}. We then use identification, cross-fitting, and the estimation rates in Assumptions~\ref{ass:identifiability}, \ref{ass:cross-fitting}, and~\ref{ass:estimation-rates} to obtain the rate form. The comparison does not require realizability or a specified generalization rate.

\begin{proof}
Let $\widehat\theta$ be an empirical minimizer of $\widehat{\Lcal}_n$ as defined in Eq.~(\ref{eq:empirical-gap-loss}). We define the population loss and abbreviate the deviation and payoff error as
\begin{align*}
\widehat{\Lcal}(\theta)&:=\E_X[G_\Om(\pi_\theta(X);\AhatDB(X;\widehat\eta),X)],\\
\varepsilon_{\mathrm{gen}}&:=\sup_{\theta\in\Theta}|\widehat{\Lcal}(\theta)-\widehat{\Lcal}_n(\theta)|,
\qquad
\mathcal E_A:=\kappa^{-1}\E_X\|\AhatDB(X;\widehat\eta)-A(X)\|_\star^2.
\end{align*}

Fix a context $x$ at which the assumptions hold, write $A=A(x)$ and $B=\AhatDB(x;\widehat\eta)$, and let $p\in\DeltaK$ be arbitrary. For either payoff $C\in\{A,B\}$, define
\begin{align*}
\varphi_C(p,q):=\langle-Cp,p-q\rangle+\Om(p)-\Om(q),
\qquad q\in\DeltaK.
\end{align*}
By the variational expression in the gap definition \eqref{eq:gap}, $G_\Om(p;C,x)=\max_{q\in\DeltaK}\varphi_C(p,q)$. By Assumption~\ref{ass:smoothness}, $\Om$ is continuous, so $q\mapsto\varphi_C(p,q)$ is continuous on the compact simplex and attains a finite maximum. Taking $q=p$ gives $\varphi_C(p,p)=0$, so both gaps are nonnegative. The candidate $p$ need not be an equilibrium of either game.

To compare the gaps, we use strong concavity in the opponent strategy. For $q_0,q_1\in\DeltaK$ and $\alpha\in[0,1]$, put $q_\alpha=(1-\alpha)q_0+\alpha q_1$. Assumption~\ref{ass:strong-convexity} gives
\begin{align*}
\Om(q_\alpha)\le(1-\alpha)\Om(q_0)+\alpha\Om(q_1)
-\frac{\kappa}{2}\alpha(1-\alpha)\|q_1-q_0\|^2.
\end{align*}
The term $\langle-Bp,p-q\rangle$ is affine in $q$, and $\Om(p)$ is constant in $q$. Negating the displayed inequality and adding these terms therefore yields
\begin{align}
\varphi_B(p,q_\alpha)
&\ge(1-\alpha)\varphi_B(p,q_0)+\alpha\varphi_B(p,q_1)\nonumber\\
&\quad+\frac{\kappa}{2}\alpha(1-\alpha)\|q_1-q_0\|^2.
\label{eq:strong-gap-concavity}
\end{align}
Thus $q\mapsto\varphi_B(p,q)$ is $\kappa$-strongly concave with respect to the same norm used in the theorem.

Fix $\alpha\in(0,1)$. For any $q\in\DeltaK$, apply \eqref{eq:strong-gap-concavity} with $q_0=p$ and $q_1=q$. Since $(1-\alpha)p+\alpha q\in\DeltaK$ and $\varphi_B(p,p)=0$,
\begin{align*}
G_\Om(p;B,x)
&\ge\varphi_B\big(p,(1-\alpha)p+\alpha q\big)\\
&\ge\alpha\varphi_B(p,q)+\frac{\kappa}{2}\alpha(1-\alpha)\|p-q\|^2.
\end{align*}
Dividing by $\alpha$ and rearranging gives
\begin{equation}
\varphi_B(p,q)\le\frac1\alpha G_\Om(p;B,x)-\frac{\kappa(1-\alpha)}{2}\|p-q\|^2
\qquad(q\in\DeltaK).
\label{eq:strong-gap-midpoint}
\end{equation}
The midpoint choice $\alpha=1/2$ gives the coefficient $2$ on the gap; retaining $\alpha$ allows that coefficient to approach $1$.

We now relate the two payoff matrices. Write $E=B-A$. Subtracting the two variational gap objectives gives the exact identity
\begin{align*}
\varphi_A(p,q)=\varphi_B(p,q)+\langle Ep,p-q\rangle.
\end{align*}
This identity holds for arbitrary real matrices $A$ and $B$. Let $a=\|Ep\|_*\ge0$. The dual-norm inequality gives $\langle Ep,p-q\rangle\le a\|p-q\|$. Combining this with \eqref{eq:strong-gap-midpoint} yields
\begin{align*}
G_\Om(p;A,x)
&\le\frac1\alpha G_\Om(p;B,x)
+\max_{q\in\DeltaK}\left\{a\|p-q\|-\frac{\kappa(1-\alpha)}{2}\|p-q\|^2\right\}\\
&\le\frac1\alpha G_\Om(p;B,x)
+\sup_{t\ge0}\left\{at-\frac{\kappa(1-\alpha)}{2}t^2\right\}.
\end{align*}
The second inequality enlarges the set of feasible distances to all nonnegative real numbers. Completing the square,
\begin{align*}
at-\frac{\kappa(1-\alpha)}{2}t^2
=\frac{a^2}{2\kappa(1-\alpha)}
-\frac{\kappa(1-\alpha)}{2}\left(t-\frac{a}{\kappa(1-\alpha)}\right)^2.
\end{align*}
The supremum is $a^2/[2\kappa(1-\alpha)]$, attained at $t=a/[\kappa(1-\alpha)]$. Taking $\alpha=1/2$ gives
\begin{equation}
G_\Om(p;A,x)
\le2G_\Om(p;B,x)
+\frac1\kappa\|(B-A)p\|_*^2.
\label{eq:strong-gap-comparison}
\end{equation}
This argument applies to every $p\in\DeltaK$, including boundary policies. Interchanging $A$ and $B$ gives the reverse comparison with the same squared error, since $\|(A-B)p\|_*=\|(B-A)p\|_*$.

Apply \eqref{eq:strong-gap-comparison} with $p=\pi_{\widehat\theta}(X)$ and $B=\AhatDB(X;\widehat\eta)$. Under our measurability conventions, the expectation over contexts is well defined, and
\begin{equation*}
\Exploit(\pi_{\widehat\theta}) \le2\widehat{\Lcal}(\widehat\theta) +\frac1\kappa\E_X\!\left[ \big\|(\AhatDB(X;\widehat\eta)-A(X))\pi_{\widehat\theta}(X)\big\|_*^2 \right].
\end{equation*}
All integrands are nonnegative, so the inequality also holds with an infinite right-hand side if the payoff error has no finite second moment. The $L_2(P_X)$ payoff error in Assumption~\ref{ass:boundedness} and compactness of the strategy domain ensure a finite population gap. By the definition of the uniform deviation,
\begin{align*}
\widehat{\Lcal}(\widehat\theta)
\le\widehat{\Lcal}_n(\widehat\theta)+\varepsilon_{\mathrm{gen}}.
\end{align*}
This inequality applies to the data-dependent $\widehat\theta$ because the supremum defining $\varepsilon_{\mathrm{gen}}$ is over all $\theta\in\Theta$. Substituting it gives the intermediate bound \eqref{eq:empirical-gap-transfer}. Empirical minimization further allows $\widehat{\Lcal}_n(\widehat\theta)$ to be written as $\inf_{\theta\in\Theta}\widehat{\Lcal}_n(\theta)$, although the bound itself holds for any fitted policy in the class.

Thus the resulting bound is
\begin{equation}
\Exploit(\pi_{\widehat\theta}) \le2\big\{\widehat{\Lcal}_n(\widehat\theta)+\varepsilon_{\mathrm{gen}}\big\} +\frac1\kappa\E_X\!\left[ \big\|(\AhatDB(X;\widehat\eta)-A(X))\pi_{\widehat\theta}(X)\big\|_*^2 \right]. \label{eq:empirical-gap-transfer}
\end{equation}

Write $\Lcal^A(\theta):=\E_X[G_\Om(\pi_\theta(X);A(X),X)]=\Exploit(\pi_\theta)$, and let
\begin{align*}
\delta(\theta):=\frac1{2\kappa}\E_X\!\left[\big\|(\AhatDB(X;\widehat\eta)-A(X))\pi_\theta(X)\big\|_*^2\right].
\end{align*}
Applying \eqref{eq:strong-gap-comparison} in both directions and taking expectations over $X$ gives
\begin{align*}
\Lcal^A(\theta)\le2\widehat{\Lcal}(\theta)+2\delta(\theta),
\qquad
\widehat{\Lcal}(\theta)\le2\Lcal^A(\theta)+2\delta(\theta).
\end{align*}
For any comparator $\theta\in\Theta$, the definition of $\varepsilon_{\mathrm{gen}}$ and empirical optimality yield
\begin{align}
\widehat{\Lcal}(\widehat\theta)
&\le\widehat{\Lcal}_n(\widehat\theta)+\varepsilon_{\mathrm{gen}}\nonumber\\
&\le\widehat{\Lcal}_n(\theta)+\varepsilon_{\mathrm{gen}}
\le\widehat{\Lcal}(\theta)+2\varepsilon_{\mathrm{gen}}.
\label{eq:quasi-oracle-transfer}
\end{align}
Using the forward comparison at $\widehat\theta$ and the reverse comparison at $\theta$ therefore gives
\begin{equation}
\Exploit(\pi_{\widehat\theta})\le4\Lcal^A(\theta)+4\varepsilon_{\mathrm{gen}}+4\delta(\theta)+2\delta(\widehat\theta).
\label{eq:oracle-comparator}
\end{equation}
The norm normalization in Section~\ref{sec:problem} gives $\delta(\theta)\le\mathcal E_A/2$ uniformly over $\theta$. Taking the infimum gives the deterministic oracle inequality
\begin{equation}
\Exploit(\pi_{\widehat\theta})\le4\inf_{\theta\in\Theta}\Exploit(\pi_\theta)+4\varepsilon_{\mathrm{gen}}+3\mathcal E_A.
\label{eq:oracle-inequality}
\end{equation}
No population minimizer is required. If $\mathcal E_A$ or $\varepsilon_{\mathrm{gen}}$ is infinite, the inequality holds with an infinite right-hand side.

To obtain the rate form in the theorem, we use the remaining assumptions. The norm normalization in Section~\ref{sec:problem} gives $\|(\AhatDB-A)\pi_{\widehat\theta}\|_*\le\|\AhatDB-A\|_\star$. We separate the error of the second-stage regression from the bias induced by the fitted nuisances:
\begin{align*}
\AhatDB(\cdot;\widehat\eta)-A
=\big\{\AhatDB(\cdot;\widehat\eta)-\ADB(\cdot;\widehat\eta)\big\}
+\big\{\ADB(\cdot;\widehat\eta)-A\big\}.
\end{align*}
The first term has maximum entrywise $L_2(P_X)$ error $e_A$ by Eq.~(\ref{eq:payoff-regression-rate}). Let
\begin{align*}
b_n:=\max_{j,k}\|\ADB_{jk}(\cdot;\widehat\eta)-A_{jk}\|_{L_2(P_X)}.
\end{align*}
The positivity and smoothness conditions in Assumptions~\ref{ass:boundedness} and~\ref{ass:smoothness} allow us to apply \eqref{eq:second-order}. Since $\|\max_m|v_m|\|_{L_2}^2\le\sum_m\|v_m\|_{L_2}^2$ and the menu collections are fixed, the remainder-component rates in Assumption~\ref{ass:estimation-rates} give $b_n=O_p(r_\mu^2+\elb^{-1}r_\mu r_e)$. Under Assumption~\ref{ass:cross-fitting}, the triangle inequality gives the same bound for the average of the fold-specific targets. The triangle inequality also gives $\max_{j,k}\|\AhatDB_{jk}-A_{jk}\|_{L_2(P_X)}\le e_A+b_n$. Thus, writing $C_{\mathrm{mat}}$ for the norm-comparison constant defined in the notation paragraph,
\begin{align*}
\mathcal E_A\le\frac{C_{\mathrm{mat}}}{\kappa}\sum_{j,k}\|\AhatDB_{jk}-A_{jk}\|_{L_2(P_X)}^2
\le\frac{2C_{\mathrm{mat}}K^2}{\kappa}(e_A^2+b_n^2).
\end{align*}
Squaring the rate for $b_n$ yields
\begin{equation}
\mathcal E_A=\frac{1}{\kappa}\E_X\|\AhatDB-A\|_\star^2
=O_p\!\left(\frac{K^2}{\kappa}\big(e_A^2+r_\mu^4+\elb^{-2}r_\mu^2r_e^2\big)\right).
\label{eq:avg-exploit-rate}
\end{equation}
The $O_p$ constant depends on $C_{\mathrm{mat}}$, $L_f$, and $G_f$. For non-Euclidean policy norms, $C_{\mathrm{mat}}$ can itself depend on $K$, whereas the Euclidean norm gives $C_{\mathrm{mat}}=1$. For the affine aggregators of Section~\ref{sec:problem}, $L_f=0$ and $R_f\equiv0$, so the $r_\mu^4$ term is absent and the nuisance contribution is $\elb^{-2}r_\mu^2r_e^2$.

Substituting \eqref{eq:avg-exploit-rate} into \eqref{eq:oracle-inequality}, expanding $\varepsilon_{\mathrm{gen}}$ as its defining supremum, and absorbing the factor $3$ into $O_p(\cdot)$ proves exactly \eqref{eq:avg-exploit}. The stochastic remainder is nonnegative and is bounded by a stochastically bounded multiplier of the displayed scale, which includes the random regression error $e_A^2$. Neither realizability nor a rate for the generalization deviation is used in this argument.
\end{proof}
\subsection{Proof of Corollary~\ref{cor:quasi-oracle}}
\label{app:proof-quasi-oracle}

We specialize Theorem~\ref{thm:avg-exploit} under realizability and loss-class complexity in Assumptions~\ref{ass:realizability} and~\ref{ass:complexity}, the bounded loss envelope in Assumption~\ref{ass:boundedness}, and the additional rates in Eq.~(\ref{eq:quasi-oracle-conditions}).

\begin{proof}
Write $\Lcal^A(\theta):=\Exploit(\pi_\theta)$. Assumption~\ref{ass:realizability} implies $\inf_{\theta\in\Theta}\Lcal^A(\theta)=0$. We first establish the oracle rate, which does not use the nuisance or payoff convergence rates in \eqref{eq:quasi-oracle-conditions}.

We use the i.i.d. sampling in Assumption~\ref{ass:cross-fitting} and the complexity bound in Assumption~\ref{ass:complexity} to control generalization. For a measurable function $h$, write $Ph:=\E_X h(X)$ and $P_nh:=n^{-1}\sum_i h(X_i)$, and set $Z_n:=\sup_{h\in\mathcal H_n}|Ph-P_nh|$. Let $X'_1,\ldots,X'_n$ be an independent copy of the contexts, with empirical average $P'_n$. Jensen's inequality gives the first bound below. Independently swapping each pair $(X_i,X'_i)$ leaves its law unchanged, which gives the equality after introducing the signs. The triangle inequality then gives
\begin{align*}
\E Z_n
&\le\E_{X,X'}\sup_{h\in\mathcal H_n}|P'_nh-P_nh|\\
&=\E_{X,X',\sigma}\sup_{h\in\mathcal H_n}
\left|\frac1n\sum_{i=1}^n\sigma_i\{h(X'_i)-h(X_i)\}\right|\\
&\le2\mathfrak R_n(\mathcal H_n)\le2C n^{-1/2}.
\end{align*}
Markov's inequality implies $Z_n=O_p(n^{-1/2})$.

To define the oracle benchmark, let
\begin{align*}
\Lcal_{n,\mathrm{or}}(\theta):=\frac1n\sum_{i=1}^nG_\Om(\pi_\theta(X_i);A(X_i),X_i),
\qquad
\widehat\theta_{\mathrm{or}}\in\argmin_{\theta\in\Theta}\Lcal_{n,\mathrm{or}}(\theta).
\end{align*}
Because $A\in\mathcal B_n$, its deviation $\varepsilon_{\mathrm{gen,or}}:=\sup_{\theta\in\Theta}|\Lcal^A(\theta)-\Lcal_{n,\mathrm{or}}(\theta)|$ also satisfies $\varepsilon_{\mathrm{gen,or}}\le Z_n=O_p(n^{-1/2})$. For every $\theta\in\Theta$,
\begin{align*}
\Lcal^A(\widehat\theta_{\mathrm{or}})
\le\Lcal_{n,\mathrm{or}}(\widehat\theta_{\mathrm{or}})+\varepsilon_{\mathrm{gen,or}}
\le\Lcal_{n,\mathrm{or}}(\theta)+\varepsilon_{\mathrm{gen,or}}
\le\Lcal^A(\theta)+2\varepsilon_{\mathrm{gen,or}}.
\end{align*}
Taking the infimum and using realizability proves $\Exploit(\pi_{\widehat\theta_{\mathrm{or}}})=O_p(n^{-1/2})$.

For the learner using the estimated payoff, retain $\varepsilon_{\mathrm{gen}}$ and $\mathcal E_A$ as defined in the theorem proof. On the event that the fitted payoff belongs to $\mathcal B_n$, we have $\varepsilon_{\mathrm{gen}}\le Z_n$. This event has probability tending to one, so $\varepsilon_{\mathrm{gen}}=O_p(n^{-1/2})$ even when the fitted payoff uses the policy-training observations. The deterministic comparison \eqref{eq:oracle-inequality} and realizability give
\begin{align*}
\Exploit(\pi_{\widehat\theta})\le4\varepsilon_{\mathrm{gen}}+3\mathcal E_A.
\end{align*}
The matrix-norm comparison in the notation paragraph and the payoff-error decomposition used in \eqref{eq:avg-exploit-rate} give
\begin{equation}
\mathcal E_A\lesssim\frac1\kappa\sum_{j,k}\|\AhatDB_{jk}-A_{jk}\|_{L_2(P_X)}^2
=O_p\!\left(\frac{K^2}{\kappa}\{e_A^2+r_\mu^4+\elb^{-2}r_\mu^2r_e^2\}\right).
\label{eq:oracle-payoff-rate}
\end{equation}
The conditions in \eqref{eq:quasi-oracle-conditions} imply $e_A^2=O_p(n^{-1/2})$, $r_\mu^4\lesssim n^{-1/2}$, and $\elb^{-2}r_\mu^2r_e^2\lesssim n^{-1/2}$. Since $K$ and the regularity constants are fixed, $\mathcal E_A=O_p(n^{-1/2})$. Substitution into \eqref{eq:oracle-inequality} proves the same exploitability rate for $\widehat\theta$, which minimizes the estimated-payoff loss. For affine aggregators, the $r_\mu^4$ term vanishes, so the separate condition on $r_\mu$ is unnecessary. Replacing the payoff and nuisance rate conditions by the corresponding little-$o$ conditions makes $\mathcal E_A=o_p(n^{-1/2})$.
\end{proof}

\end{document}